\documentclass[twoside]{article}

\usepackage[accepted]{aistats2022}

\usepackage[utf8]{inputenc} 
\usepackage[T1]{fontenc}    
\usepackage{url}            
\usepackage{booktabs}       
\usepackage{amsfonts}       
\usepackage{nicefrac}       
\usepackage{microtype}      
\usepackage{amsmath}
\usepackage{amsthm}
\usepackage{amssymb}
\usepackage{graphicx}
\usepackage{caption}
\usepackage{subcaption}
\usepackage[ruled]{algorithm2e}
\usepackage{float} 
\usepackage{xr}
\usepackage{hyperref}
\usepackage[parfill]{parskip} 
\usepackage{bbm}
\usepackage[T1]{fontenc}
\usepackage{lmodern}
\usepackage[dvipsnames]{xcolor}

\usepackage[hide]{notes} 
\newcommand{\highlight}{\color{black}}

\usepackage[round]{natbib}

\newenvironment{proofof}[1]{\paragraph{Proof of #1:}}{\hfill$\square$}
\newenvironment{pf}{\paragraph{Proof}}{\hfill$\square$}
\newtheorem{theorem}{Theorem}
\newtheorem{lemma}{Lemma}
\newtheorem{proposition}{Proposition}
\newtheorem{corollary}{Corollary}
\newtheorem{definition}{Definition} 
\newtheorem{example}{Example} 

\newtheorem{remark}{Remark}
\newtheorem{assumption}{Assumption}
\DeclareMathOperator*{\argmax}{argmax} 

\newcommand{\reals}{\ensuremath{\mathbb{R}}\xspace}
\newcommand{\complex}{\ensuremath{\mathbb{C}}\xspace}
\newcommand{\ints}{\ensuremath{\mathbb{N}}\xspace}
\newcommand{\sphere}{\ensuremath{\mathbb{S}}\xspace}
\newcommand{\bigO}{\ensuremath{\mathcal{O}}\xspace}
\newcommand{\bigOmega}{\ensuremath{\Omega}\xspace}
\newcommand{\bigTheta}{\ensuremath{\Theta}\xspace}

\newcommand{\expectation}[1]{\ensuremath{\mathbb{E}\left[#1\right]}\xspace}

\newcommand{\prob}[1]{\ensuremath{\mathbb{P}\left[#1\right]}\xspace}

\newcommand{\vect}[1]{\ensuremath{\mathbf{#1}}\xspace}
\newcommand{\Matrix}[1]{\ensuremath{\mathbf{#1}}\xspace}
\newcommand{\eigval}[2]{\ensuremath{\mathbf{\lambda}_{#1}({#2})}\xspace}
\newcommand{\sigval}[2]{\ensuremath{\mathbf{\sigma}_{#1}({#2})}\xspace}
\newcommand{\eigvec}[2]{\ensuremath{\mathbf{u}_{#1}(#2)}\xspace}
\newcommand{\eigvalA}[1]{\ensuremath{\mathbf{\lambda}_{#1}}\xspace}
\newcommand{\sigvalA}[1]{\ensuremath{\mathbf{\sigma}_{#1}}\xspace}
\newcommand{\neigval}[1]{\ensuremath{\mathbf{\alpha}_{#1}}\xspace}
\newcommand{\eigvecA}[1]{\ensuremath{\mathbf{u}_{#1}}\xspace}
\newcommand{\hu}{\ensuremath{\hat{\vect{u}}}\xspace}

\newcommand{\norm}[1]{\ensuremath{\left\|{#1}\right\|}\xspace}

\newcommand{\abs}[1]{\ensuremath{\lvert{#1}\rvert}\xspace}
\newcommand{\range}[1]{\ensuremath{\text{range}(#1)}\xspace}

\newcommand{\sign}[1]{\ensuremath{\text{sign}\left({#1}\right)}\xspace}

\newcommand{\vecones}[1]{\ensuremath{\mathbf{1}_{#1}}\xspace}
\newcommand{\veczeros}[1]{\ensuremath{\mathbf{0}_{#1}}\xspace}

\newcommand{\diag}[1]{\ensuremath{\text{diag}({#1})}\xspace}
\newcommand{\angl}[2]{\ensuremath{\theta({#1},{#2})}\xspace}

\newcommand{\RC}[1]{\ensuremath{(\text{#1})}\xspace}

\newcommand{\bernoulli}[1]{\ensuremath{\text{Bernoulli}(#1)}\xspace}
\newcommand{\gaussian}[2]{\ensuremath{\mathcal{N}(#1, #2)}\xspace}
\newcommand{\uniform}[1]{\ensuremath{\text{Uniform}(#1)}\xspace}
\newcommand{\grassmannian}[2]{\ensuremath{\mathcal{G}_{#1, #2}}\xspace}

\newcommand{\multirsvd}{\ensuremath{\textup{\sf RSVD}}\xspace}
\newcommand{\onersvd}{\ensuremath{\textup{\sf SinglePass-RSVD}}\xspace}
\newcommand{\rsum}{\ensuremath{\textup{\sf RandSum}}\xspace}
\newcommand{\onersum}{\ensuremath{\textup{\sf SinglePass-RandSum}}\xspace}
\newcommand{\blockkrylov}{\ensuremath{\textup{\sf BlockKrylov}}\xspace}
\newcommand{\lancoz}{\ensuremath{\textup{\sf Lanczos\-Method}}\xspace}

\newcommand{\randround}{\ensuremath{\textup{Random\-EigenSign}}\xspace}

\newcommand{\wikiv}{\ensuremath{\textsf{Wiki\-Vot}}\xspace}
\newcommand{\referendum}{\ensuremath{\textsf{Referendum}}\xspace}
\newcommand{\slashdot}{\ensuremath{\textsf{Slash\-dot}}\xspace}
\newcommand{\wikic}{\ensuremath{\textsf{Wiki\-Con}}\xspace}

\newcommand{\ptop}{\ensuremath{\textsf{Gnutella\-31}}\xspace}

\newcommand{\youtube}{\ensuremath{\textsf{You\-Tube}}\xspace}
\newcommand{\roadca}{\ensuremath{\textsf{Road\-CA}}\xspace}
\newcommand{\fbart}{\ensuremath{\textsf{FB\-Artist}}\xspace}
\newcommand{\berkstan}{\ensuremath{\textsf{BerkStan}}\xspace}
\newcommand{\google}{\ensuremath{\textsf{Google}}\xspace}
\newcommand{\notredame}{\ensuremath{\textsf{Notre\-Dame}}\xspace}
\newcommand{\stanford}{\ensuremath{\textsf{Stanford}}\xspace}

\newcommand{\para}[1]{\noindent{{\bf{#1}}}}

\makeatletter
\newtheorem*{rep@theorem}{\rep@title}
\newcommand{\newreptheorem}[2]{%
	\newenvironment{rep#1}[1]{%
		\def\rep@title{#2 \ref{##1}}%
		\begin{rep@theorem}}%
		{\end{rep@theorem}}}
\makeatother

\newreptheorem{theorem}{Theorem}
\newreptheorem{lemma}{Lemma}
\newreptheorem{proposition}{Proposition}
\newreptheorem{corollary}{Corollary}
\newreptheorem{assumption}{Assumption}

\begin{document}
\twocolumn[
\aistatstitle{Improved analysis of randomized SVD for top-eigenvector approximation}

\runningauthor{Ruo-Chun Tzeng, Po-An Wang, Florian Adriaens, Aristides Gionis, Chi-Jen Lu}
\aistatsauthor{ 
		Ruo-Chun Tzeng$^1$ \And 
		Po-An Wang$^2$  \And 
		Florian Adriaens$^1$  \AND 
		Aristides Gionis$^1$ \And  
		Chi-Jen Lu$^3$ 
}
\aistatsaddress{ 
	$^1$Division of Theoretical Computer Science\\
	$^2$Division of Decision and Control Systems\\KTH Royal Institute of Technology, Sweden
	\And 
	$^3$Institute of Information Science\\Academia Sinica, Taiwan
}
]

\begin{abstract}
Computing the top eigenvectors of a matrix is a problem of fundamental interest to various fields.
While the majority of the literature has focused on analyzing the reconstruction error
of low-rank matrices associated with the retrieved eigenvectors, 
in many applications one is interested in finding one vector with high Rayleigh quotient.
In this paper we study the problem of approximating the top-eigenvector.
Given a symmetric matrix $\Matrix{A}$ with largest eigenvalue $\eigvalA{1}$, 
our goal is to find a vector \hu that approximates the leading eigenvector $\eigvecA{1}$ 
with high accuracy, as measured by the ratio
$R(\hu)=\eigvalA{1}^{-1}{\hu^T\Matrix{A}\hu}/{\hu^T\hu}$.
We present a novel analysis of the randomized SVD algorithm of \citet{halko2011finding} and 
derive tight bounds in many cases of interest.
Notably, this is the first work that provides non-trivial bounds for approximating the ratio $R(\hu)$ 
using randomized SVD with any number of iterations.
Our theoretical analysis is complemented with a thorough experimental study that confirms the efficiency and accuracy of the method.

\end{abstract}

\section{\uppercase{Introduction}}\label{sec:intro}

Spectral methods, which typically rely on computing the leading eigenvectors of an appropriately-designed matrix, have been shown to provide high-quality solutions to a variety of problems in the fields of data analysis, optimization, clustering, and learning~\citep{kannan2009spectral}.
From a computational perspective, randomized approaches for spectral methods, often give good estimates of leading eigenvectors and low-rank structures, opening up the possibility of dealing with truly large datasets \citep{halko2011algorithm}.

In this paper, we study the problem of approximating the leading eigenvector of a matrix 
while using a small amount of memory and making a limited number of passes over the input matrix.
More concretely, 
given a symmetric matrix $\Matrix{A}$ with largest eigenvalue~$\eigvalA{1}$, 
our goal is to find a vector $\hu$ that maximizes the ratio
\begin{equation}
\label{eq:obj-multiplicative-gap}
R(\hu)=\eigvalA{1}^{-1}\frac{\hu^T\Matrix{A}\hu}{\hu^T\hu}.
\end{equation}
Note that since $\eigvalA{1}$ is fixed given \Matrix{A}, 
it can be omitted from the definition of~$R$; 
it is used only for convenience, to ensure that $R\le 1$.
Often, in different applications, 
in addition to having to select which matrix $\Matrix{A}$ to use,
it is also required that $\hu \in \mathcal{T}\subseteq\reals^n$, 
where $\mathcal{T}$ is typically a discrete subspace of~$\reals^n$.
A common strategy in this case, 
is to first compute an approximation of the leading eigenvector in $\reals^n$
and then ``round'' the solution in $\mathcal{T}$.
Below we outline some prominent examples of this scheme.

{\bf (1)}
The most direct example is PCA, 
where $\Matrix{A}$ is the covariance matrix~\citep{jolliffe1986principal};
in this case $\mathcal{T}=\reals^n$,
and no rounding is required;
{\bf (2)}
In the community-detection problem 
we can partition a network into two communities
(and then recursively find more communities) 
by maximizing modularity~\citep{newman2006modularity}, 
which can be mapped to our setting by
taking~$\Matrix{A}$ to be the modularity matrix 
and~$\mathcal{T}=\{\pm1\}^n$;
{\bf (3)}
The problem of finding $k$ conflicting groups in signed networks
can be mapped to our setting by
taking~$\Matrix{A}$ to be the adjacency matrix of the signed network
and $\mathcal{T}=\{0,-1,\ell\}^n$, 
for $\ell\in[k-1]$~\citep{bonchi2019discovering,tzeng2020discovering};
{\bf (4)}
For the fair densest subgraph,
\citet{anagnostopoulos2020spectral} 
consider $\mathcal{T}=\{0,1\}^n$ and 
obtain $\Matrix{A}$ after projecting the adjacency matrix 
onto the subspace orthogonal to a given fairness labeling $z\in\{\pm1\}^n$;
{\bf (5)}
In few other applications, 
a solution to our problem 
is used as an intermediate result in the proposed method~\citep{abdullah2014spectral, hopkins2016fast,allen2016lazysvd,silva2018spectral}.

Despite numerous pass-efficient algorithms proposed in the literature for computing top eigenvectors, 
prior attempts to analyze $R(\hu)$ 
have strong limitations when applied in practice.
The main shortcoming is that most works provide \emph{additive bounds}
and require $\bigOmega(\ln n)$ passes to be meaningful \citep{simchowitz2018tight}, 
whereas a smaller number of passes (constant or even a single pass)
is critical in practical settings. 
It is unclear in the state-of-the-art
whether $\bigOmega(\ln n)$ passes is necessary for previous methods, 
or whether such a bound is an artifact of the analysis.

In this paper we demonstrate that the requirement of $\bigOmega(\ln n)$ passes in the analysis of prior works is artificial.
We show this by giving a \emph{multiplicative bound} for $R(\hu)$ 
achieved by the randomized SVD method (\multirsvd) of \citet{halko2011finding}, which is one of the most prominent and widely-implemented pass-efficient algorithms
\citep{pedregosa2011scikit, rehurek_lrec, libskylark2021rsvd, erichson2019randomized, terry2021statpack, antoine2021rsvd}

Our analysis shows that, for any positive semi\-definite matrix, 
the \multirsvd\ method,
{\highlight using $\bigO(dn)$ space for $d\in\ints$ typically $d\ll n$ (e.g., $d=\bigO(\ln n)$),}
returns with high probability a vector~$\hu$ with 
$R(\hu)=\bigOmega\left((d/n)^{1/(2q+1)}\right)$
after $q\in\ints$ iterations (Theorem~\ref{thm:rsvd-main}),
and our analysis is tight (Theorem~\ref{thm:tightness-rsvd-main}).
Notably, our analysis subsumes the guarantee by prior works in the regime of $\bigOmega(\ln n)$ passes (Remark~\ref{re:recover-prior-analysis}), 
and to the best of our knowledge, 
provides the first non-trivial guarantee of $R(\hu)$ 
in the literature of pass-efficient algorithms for $o(\ln n)$ passes.
Moreover, we show that under some natural conditions satisfied by real-world datasets, it is even possible to achieve $R(\hu)=\bigOmega(1)$ with a single pass (Remark~\ref{re:rsvd-power-law}).


Our core technical argument is a reduction from the optimization problem of maximizing $R$ over a random subspace 
to the problem of estimating the projection length of a vector onto a random subspace.
By using our technique, 
we derive the first non-trivial guarantee 
of $R(\hu)$ 
for any number of passes for indefinite matrices (Theorem~\ref{thm:rsvd-indefinite-power-law}), 
under mild conditions (Assumption~\ref{assumption:rsvd-indefinite}).

In addition, we propose an extension of the \multirsvd\ method, called \rsum,
by using a random matrix sampled from $\bernoulli{p}$ with mean $p\in(0,1)$.
While such a random matrix is rarely used in the literature of random projections, 
we show that there exist applications \citep{bonchi2019discovering, tzeng2020discovering} 
especially suitable for this technique, 
and we show several properties of such a random matrix, which may be of independent interest.

\section{\uppercase{Related work}}
\label{sec:rw}

For lack of space, we only provide a brief overview of the related work, 
focusing on the most relevant works for our paper.
For a general introduction on pass-efficient algorithms for matrix approximations, 
we refer the reader to \citet{mahoney2011randomized, woodruff2014sketching, martinsson2020randomized}.

The study of $R(\hu)$ 
for pass-efficient algorithms can be dated back to \cite{kuczynski1992estimating} who analyzed two classical methods: 
the power method and the Lanczos method with random start.
For any positive semi\-def\-inite matrix, they showed that the power method 
(respectively, Lanczos method) with random start, 
after $q\geq 2$ iterations returns an approximated top-eigenvector $\hu$ 
with $\expectation{R(\hu)} \geq 1-0.871 \frac{\ln n}{q-1}$ 
(respectively, $\expectation{R(\hu)} \geq 1-2.575(\frac{\ln n}{q-1})^2$).

The aforementioned methods are generalized to randomized SVD \citep{halko2011finding} 
and block-Krylov methods \citep{musco2015randomized}, 
and a similar additive analysis of $R(\hu)$ 
by \cite{musco2015randomized} 
showed that for any positive semi\-def\-inite matrix, \multirsvd (respectively, randomized block-Krylov method) 
using $\bigO(nd)$ space (respectively, $\bigO(ndq)$ space) and after~$q$ iterations, 
returns an approximate top-eigenvector $\hu$ with $R(\hu) \geq 1-\bigO(\frac{\ln n}{q})$ 
(respectively, $R(\hu) \geq 1-\bigO((\frac{\ln n}{q})^2)$), with probability at least $1-e^{-\bigOmega(d)}$.%
\footnote{\cite{musco2015randomized} showed that the aforementioned results hold with constant probability, 
which could be improved to hold with probability $1-e^{-\bigOmega(d)}$ 
by using stronger concentration results \citep{rudelson2010non} in their proofs of Lemma 4 and Lemma 9.}

The analysis of the previous works \citep{kuczynski1992estimating, musco2015randomized}  is tight, 
as shown by \citet{simchowitz2018tight}
for a class of methods (which include \multirsvd and block Krylov), 
which with high probability fail to find a vector $\hu$ with $R(\hu)\geq 23/24$ within $q=\bigO(\ln n)$ passes.

In the aforementioned works there are two limitations.
First, the bounds of \citet{kuczynski1992estimating} and \citet{musco2015randomized} are additive, 
and unfortunately require $\Omega(\ln n)$ passes to be meaningful.
In contrast, our analysis provides a multiplicative bound for $R(\hu)$ 
and offers non-trivial guarantees for any number of passes.
Second, the applicability of the methods of \citet{kuczynski1992estimating} and \citet{musco2015randomized} 
is limited to positive semi\-def\-inite matrices.
Instead, we provide sharp analysis of randomized SVD for positive semi\-def\-inite matrices 
and show that our proof techniques generalize to indefinite matrices under mild conditions.

To complement our study, 
we briefly compare the measure $R(\hu)$ 
with other classical metrics.
Note here that, even though it is possible to covert an error guarantee for classical metrics 
\citep{xu2018accelerated, drineas2018structural, ghashami2016frequent, chen2017frosh, musco2017sublinear, huang2018near}
into a lower bound for $R(\hu)$ 
by matrix perturbation theory \citep{Stewart90,yu2015useful}, the resulting bound is additive and depends on the eigen\-gap.
We also note that classical metrics typically compare the approximation $\hu$ to the top-eigenvector 
$\eigvecA{1}$ of $\Matrix{A}$, 
however, such a comparison is not meaningful in our setting
as small distance between $\hu$ and $\eigvecA{1}$\footnote{More precisely, the distance between $\hu$ and the eigen\-space associated with the largest eigenvalue $\eigvalA{1}$ of $\Matrix{A}$.} is a sufficient but not necessary condition for having large $R(\hu)$.

\section{\uppercase{Preliminaries}}
\label{sec:preliminary}

Let $\ints$ be the set of natural numbers excluding $0$.
Let $\reals$ be the set of real numbers,
$\sphere^{m-1}=\{\vect{x}\in\reals^m:\vect{x}^T\vect{x}=1\}$, and $[m]=\{1,\ldots,m\}$.
Let $\range{\Matrix{M}}$ denote the column space of matrix~$\Matrix{M}$, 
and $\lVert\cdot\rVert_F$ and $\lVert\cdot\rVert_2$ denote the Frobenius norm and the spectral norm, respectively.
For a square matrix $\Matrix{M}$, 
let $\eigval{i}{\Matrix{M}}$ be its $i$-th largest eigenvalue and $\eigvec{i}{\Matrix{M}}$ the corresponding eigenvector, and let $\sigval{i}{\Matrix{M}}$ be the $i$-th largest singular value.
In all subsequent sections, we use boldface \Matrix{A} to denote the input matrix, 
and abbreviate $\eigvalA{i}=\eigval{i}{\Matrix{A}}$, 
$\eigvecA{i}=\eigvec{i}{\Matrix{A}}$, and 
$\sigvalA{i}=\sigval{i}{\Matrix{A}}$. 
We use $\langle\cdot,\cdot\rangle$ to denote the vector inner product.
Finally, we use $\vecones{n}=[1,\ldots,1]^T$ to denote the $n$-dimensional vector of all $1$'s and 
$\veczeros{n}=[0,\ldots,0]^T$ to denote the $n$-dimensional vector of all $0$'s.

For simplicity, we assume that the input matrix $\Matrix{A}$ is real-valued and symmetric, 
with $\eigvalA{1}>0$.

\medskip
\begin{definition}[Vector projection onto subspace]
\label{def:smallest-principal-angle}
Let $\vect{v}\in \reals^n$ be a nonzero vector and $\mathcal{X}\subseteq \reals^{n}$ be a non-empty subspace.
The projection length of $\vect{v}$ onto $\mathcal{X}$ is given by $\cos\theta(\vect{v},\mathcal{X})$, 
where 
\[
\theta(\vect{v},\mathcal{X}) = 
\cos^{-1}\left(\max_{\vect{x}\in\mathcal{X}}
\frac{\langle \vect{v},\vect{x}\rangle}{\norm{\vect{v}}_2\norm{\vect{x}}_2}\right)
\]
is the projection angle.
For a matrix $\Matrix{X}$, we use $\theta(\vect{v},\Matrix{X})$ 
to denote the projection angle of $\vect{v}$ onto the range of~$\Matrix{X}$.
\end{definition}

It is well-known that projecting any vector $\vect{v}\in\reals^n$ onto
the $\range{\Matrix{S}}$ of a random matrix $\Matrix{S}\sim\gaussian{0}{1}^{n\times d}$ results in 
$\cos^2\theta(\vect{v},\Matrix{S})\approx d/n$ with high probability.

\medskip
\begin{lemma}\citep{hardt2014noisy}
\label{lem:cos-angle-normal}
Let $\vect{v}\in\reals^n$ be a nonzero vector and 
$\Matrix{S}\sim\gaussian{0}{1}^{n\times d}$, 
where $n,d\in\ints$ and $n\geq d$.
Then,
\[
	\cos^2\theta(\vect{v},\Matrix{S}) = \bigTheta\left(\frac{d}{n}\right),
\]
with probability at least $1-e^{-\bigOmega(d)}$.
\end{lemma}
For completeness, we provide the proof of Lemma~\ref{lem:cos-angle-normal} in Appendix~\ref{proof:cos-angle-normal}. 
The proof idea is to observe that 
$\frac{\norm{\Matrix{S}^T\vect{v}}_2}{\sigval{1}{\Matrix{S}}}\leq\cos\theta(\vect{v},\Matrix{S})\leq\frac{\norm{\Matrix{S}^T\vect{v}}_2}{\sigval{d}{\Matrix{S}}}$ and use the concentration of the extreme singular values of a
Gaussian random matrix.

More generally, Lemma~\ref{lem:cos-angle-normal} holds for any random matrix $\Matrix{S}$ 
whose range is uniformly distributed with respect to the Haar measure on Grassmannian $\grassmannian{n}{d}$ 
of all the $d$-dimensional subspaces of $\reals^n$, 
written as $\range{\Matrix{S}}\sim\uniform{\grassmannian{n}{d}}$.
The reader may refer to \citet{achlioptas2001database} and \citet{halko2011finding} 
for other choices of $\Matrix{S}$ and \citet[Section 5]{vershynin2018high} 
for a general introduction to this phenomenon.

\section{\uppercase{Randomized SVD}}
\label{sec:rsvd}

\begin{algorithm}[t]
	$\Matrix{Y}\leftarrow\Matrix{A}^q\Matrix{S}$ where $\Matrix{S}\sim\mathcal{D}$\;
	$\Matrix{Y}=\Matrix{Q}\Matrix{R}$\;
	$\Matrix{B}\leftarrow \Matrix{Q}^T\Matrix{A}\Matrix{Q}$\;
	$\hu= \Matrix{Q}\,\eigvec{1}{\Matrix{B}}$\;
	return \hu\;
	\caption{$\multirsvd(\Matrix{A},\mathcal{D},q,d)$}
	\label{alg:rsvd}
\end{algorithm}

We briefly review the following variant of the randomized SVD (\multirsvd) algorithm, 
as proposed by \citet{halko2011finding}, 
and shown in Algorithm~\ref{alg:rsvd}.
The algorithm returns an estimate \hu of the leading eigenvector~$\eigvecA{1}$
of the input matrix \Matrix{A}.
It uses $\bigO(dn)$ space and requires $q+1$ passes over the matrix \Matrix{A}, 
where $q\in\ints$.\footnote{More precisely, \multirsvd requires $q$ passes when $d=1$ and $q+1$ passes when $d>1$ as there is no need to compute $\eigvec{1}{B}$ when $d=1$.}
The distribution $\mathcal{D}$ is over $\reals^{n \times d}$, 
and one particular instance of the algorithm sets 
$\mathcal{D} = \gaussian{0}{1}^{n\times d}$.
The algorithm begins with a random projection $\Matrix{Y}=\Matrix{A}^q\Matrix{S}$.
The eigenvectors of $\Matrix{A}^q$ are the same as $\Matrix{A}$, but the eigenvalues of $\Matrix{A}^q$ have much stronger decay. Thus intuitively, by taking powers of the input matrix, the relative weight of the eigenvectors associated with the small eigenvalues is reduced, which is helpful in the basis identification for input matrices whose eigenvalues decay slowly.
After projecting, the algorithm efficiently approximates the top-eigenvector of $\Matrix{A}$ by
\begin{equation}\label{eq:rsvd-returned-vec}
	 \hu\in\argmax\{\vect{v}^T\Matrix{A}\vect{v}: \vect{v}\in\range{\Matrix{Y}}\cap\sphere^{n-1}\}.
\end{equation}
Indeed, any $\vect{v}\in\range{\Matrix{Y}}$ of unit length can be written as $\vect{v}=\Matrix{Q}\vect{a}$ for some $\vect{a}\in\sphere^{d-1}$, where $\Matrix{Q}$ is an $n\times d$ ortho\-normal basis given by a QR decomposition of $\Matrix{Y}$. So it follows that
\[\max_{\vect{v}\in\range{\Matrix{Y}}\cap\sphere^{n-1}} \vect{v}^T\Matrix{A}\vect{v} = \max_{\vect{a}\in\sphere^{d-1}}\vect{a}^T\Matrix{B}\vect{a} = \eigval{1}{\Matrix{B}}.\]
Thus, the vector $\hu=\Matrix{Q}\eigvec{1}{\Matrix{B}}$ maximizes expression~\eqref{eq:rsvd-returned-vec}, 
and $\eigvec{1}{\Matrix{B}}$ can be efficiently computed as the matrix $\Matrix{B}$ is of dimension $d\times d$.

\subsection{Analysis of \multirsvd}\label{sec:4.1}

We now derive lower and upper bounds for $R(\hu)$, 
where $\hu$ is the output of Algorithm~\ref{alg:rsvd}, and 
$R(\vect{v})=\eigvalA{1}^{-1}\frac{\vect{v}^T\Matrix{A}\vect{v}}{\vect{v}^T\vect{v}}$ 
is defined for any nonzero vector $\vect{v} \in \reals^n$.
Note that due to expression \eqref{eq:rsvd-returned-vec}, 
\hu maximizes $R$ over the column space $\range{\Matrix{Y}}$ of $\Matrix{Y}$.
Since $\range{\Matrix{Y}}=\{\Matrix{Y}\vect{a}: \vect{a} \in \reals^{d}\}$, 
we can rewrite $R(\hu)$ as 
\begin{align*}
R(\hu) 
= \max_{\vect{v}\in \range{\Matrix{Y}}\backslash\{\veczeros{n}\}}R(\vect{v})
= \max_{\vect{a}\in\sphere^{d-1}}R(\Matrix{Y}\vect{a}),
\end{align*}
where the latter equality follows from the scale invariance of $R$. 
For notational convenience, we denote $R_\vect{a} = R(\Matrix{Y}\vect{a})$. 
After substituting $\Matrix{Y}=\Matrix{A}^q\Matrix{S}$ in the definition of~$R$, 
we can evaluate $R_\vect{a}$ as
\begin{equation}\label{eq:rsvd-Ra}
	R_\vect{a} 
	= \frac{1}{\eigvalA{1}}\frac{(\Matrix{S}\vect{a})^T\Matrix{A}^{2q+1}(\Matrix{S}\vect{a})}{(\Matrix{S}\vect{a})^T\Matrix{A}^{2q}(\Matrix{S}\vect{a})}.
\end{equation}
Since $\Matrix{A}$ is real and symmetric, it has a real-valued eigen-decomposition $\Matrix{A} = \sum_{i=1}^n\eigvalA{i}\eigvecA{i}\eigvecA{i}^T$, with $\{\eigvecA{i}\}_{i=1}^n$ being orthonormal. 
Hence $\Matrix{A}^{k} = \sum_{i=1}^n\eigvalA{i}^k\eigvecA{i}\eigvecA{i}^T$, for any $k\in\ints$, 
and we further expand
Equation~\eqref{eq:rsvd-Ra} as
\begin{equation}\label{eq:rsvd-Ra2}
	R_\vect{a}
	= \frac{1}{\eigvalA{1}}\frac{\sum_{i}\lambda_i^{2q+1}\langle \Matrix{S}^T\eigvecA{i}, \vect{a}\rangle^2}{\sum_{i}\lambda_i^{2q}\langle \Matrix{S}^T\eigvecA{i}, \vect{a}\rangle^2}
	= \frac{\sum_{i}\neigval{i}^{2q+1}\langle \Matrix{S}^T\eigvecA{i}, \vect{a}\rangle^2}{\sum_{i}\neigval{i}^{2q}\langle \Matrix{S}^T\eigvecA{i}, \vect{a}\rangle^2},
\end{equation}
where $\neigval{i}=\eigvalA{i}/\eigvalA{1}$, for all $i\in[n]$.
This is well-defined since $\eigvalA{1}>0$.
For our analysis of $R(\hu)=\max_{\vect{a}\in\sphere^{d-1}}R_\vect{a}$, 
we first consider the case when $\Matrix{A}$ is positive semi\-definite (p.s.d.).
The proof strategy and arguments serve as a building block for the indefinite case,
discussed in Section~\ref{sec:rsvd-indefinite}.

\subsection{Positive semi\-definite matrices}\label{sec:rsvd-psd}

Our first result, is a guarantee on the performance of \multirsvd, 
asserted by the following.

\medskip
\begin{theorem}\label{thm:rsvd-main}
	Let \Matrix{A} be a positive semi\-definite matrix with $\eigvalA{1}>0$ and $\hu=\multirsvd(\Matrix{A},\gaussian{0}{1}^{n\times d},q,d)$ for any $q\in\ints$.
	Then
	\[R(\hu) = \left(\bigOmega\left(\frac{d}{n}\right)\right)^{\frac{1}{2q+1}}\]
	holds with probability at least $1-e^{-\bigOmega(d)}$. 
\end{theorem}

\begin{proof}
If \Matrix{A} is p.s.d.\ we have $\neigval{i} \geq 0$, and thus (assuming $q\in\ints$) 
we can repeatedly apply the Cauchy-Schwarz inequality to Equation~\eqref{eq:rsvd-Ra2} and get
\begin{equation}\label{eq:main-intuition1}
	R_\vect{a} 
	\geq \frac{\sum_{i}\neigval{i}^{2q}\langle \Matrix{S}^T\eigvecA{i}, \vect{a}\rangle^2}{\sum_{i}\neigval{i}^{2q-1}\langle \Matrix{S}^T\eigvecA{i}, \vect{a}\rangle^2}
	\geq \cdots 
	\geq \frac{\sum_{i}\neigval{i}\langle \Matrix{S}^T\eigvecA{i}, \vect{a}\rangle^2}{\sum_{i}\langle \Matrix{S}^T\eigvecA{i}, \vect{a}\rangle^2}.
\end{equation}
The key observation is that by repeatedly using Equation~\eqref{eq:main-intuition1} results in
\begin{align*}
	\sum_{i=1}^n\langle \Matrix{S}^T\eigvecA{i}, \vect{a}\rangle^2 & \geq R_\vect{a}^{-1}\sum_{i=1}^n\neigval{i}\langle \Matrix{S}^T\eigvecA{i}, \vect{a}\rangle^2 \geq \cdots\nonumber\\
	& \geq R_\vect{a}^{-(2q+1)}\sum_{i=1}^n\neigval{i}^{2q+1}\langle \Matrix{S}^T\eigvecA{i}, \vect{a}\rangle^2
\end{align*}
which implies
\begin{equation}\label{eq:main-intuition3}
	R_\vect{a}^{2q+1} 
	\geq \frac{\sum_{i=1}^n\neigval{i}^{2q+1}\langle \Matrix{S}^T\eigvecA{i}, \vect{a}\rangle^2}{\sum_{i=1}^n\langle \Matrix{S}^T\eigvecA{i}, \vect{a}\rangle^2}
	\geq \frac{\langle \Matrix{S}^T\eigvecA{1}, \vect{a}\rangle^2}{\sum_{i=1}^n\langle \Matrix{S}^T\eigvecA{i}, \vect{a}\rangle^2}.
\end{equation}
Finally, by $R(\hu)=\max_{\vect{a}\in\sphere^{d-1}}R_\vect{a}$ and Definition~\ref{def:smallest-principal-angle} 
we have
\begin{equation}\label{eq:rsvd-key}
	R(\hu)^{2q+1}
	\geq \max_{\vect{a}\in\sphere^{d-1}}\frac{\langle \Matrix{S}^T\eigvecA{1}, \vect{a}\rangle^2}{\sum_{i=1}^n\langle \Matrix{S}^T\eigvecA{i}, \vect{a}\rangle^2} 
	= \cos^2\angl{\eigvecA{1}}{\Matrix{S}},
\end{equation}
and invoking Lemma~\ref{lem:cos-angle-normal} proves the claim.
\end{proof}


We offer a few remarks.
First note that the fact that Equation~\eqref{eq:rsvd-key} implies Theorem~\ref{thm:rsvd-main} 
can be proven by estimating $R_\vect{a}$ only on 
$\vect{a}=\frac{\Matrix{S}^T\eigvecA{1}}{\norm{\Matrix{S}^T\eigvecA{1}}_2}$, 
since we essentially prove Lemma~\ref{lem:cos-angle-normal} on such a vector $\vect{a}$ --- 
see our discussion in Section~\ref{sec:preliminary} or Appendix~\ref{proof:cos-angle-normal}.
Second, Equation~\eqref{eq:main-intuition3} can also be shown by Hölder's inequality ---
see a simplified proof of Theorem~\ref{thm:rsvd-main} in Appendix~\ref{proof:rsvd-main}.
Third, from Theorem~\ref{thm:rsvd-main}, 
we see that increasing the number of passes $q$ makes $R(\hu)$ approaching to $1$ exponentially fast, 
while increasing the dimension $d$ leads to stronger concentration of $R(\hu)$ 
around the slowly increased mean $\bigOmega((d/n)^{1/(2q+1)})$.
Finally, we have:

\medskip
\begin{remark}\label{re:recover-prior-analysis}
	The guarantee by Theorem~\ref{thm:rsvd-main} can be written as $R(\hu)=e^{-\bigO\left(\ln n/(2q+1)\right)}\geq 1-\bigO(\ln n/q)$, and hence, subsumes the result of \citet{musco2015randomized}.
\end{remark}

One may wonder if our analysis is tight.
The next theorem confirms the tightness of Theorem~\ref{thm:rsvd-main} up to a constant factor.

\medskip
\begin{theorem}\label{thm:tightness-rsvd-main}
	For any $q\in \ints$, there exists a positive semi\-definite matrix \Matrix{A} with $\eigvalA{1}>0$, 
	so that for $\hu=\multirsvd(\Matrix{A},\gaussian{0}{1}^{n\times d},q,d)$, it holds
	\[
	R(\hu) = \bigO\left(\left(\frac{d}{n}\right)^{\frac{1}{2q+1}}\right),
	\] 
	with probability at least $1-e^{-\bigOmega(d)}$.
\end{theorem}

We prove Theorem~\ref{thm:tightness-rsvd-main} in Appendix~\ref{proof:tightness-rsvd-main} by considering the following eigenvalue distribution $\{\neigval{i}\}$:
\begin{equation}\label{eq:bad-psd-example}
	1=\neigval{1}>\neigval{2}=\cdots=\neigval{n}=\left(\frac{d}{n}\right)^{\frac{1}{2q+1}}.
\end{equation}

While our worst-case analysis is tight, Equation~\eqref{eq:bad-psd-example} rarely happens in practice.
Instead, real-world matrices are often observed to have rapidly decaying singular values \citep{chakrabarti2006graph, eikmeier2017revisiting}.
To take this consideration into account, 
we introduce the following definition to capture whether\Matrix{A} has 
at least power-law decay of its singular values $\{\sigvalA{i}\}_{i\geq i_0}^n$.

\medskip
\begin{definition}\label{def:power-law}
	{\highlight
	Let 
	\[
	i_0=\begin{cases}
		\min_{j\in\mathcal{J}}j & \text{ if }\mathcal{J}\neq\emptyset,\\
		n & \text{ otherwise},
	\end{cases}
	\]
	 where $\mathcal{J}\subseteq[n]$ consists of all the integers $j\in[n]$ such that 
	 there exists $\gamma>1/q$ and $C>0$ satisfying $\sigvalA{i}/\sigvalA{1} \leq C\cdot i^{-\gamma}$,
	 for all $i\geq j$.}
\end{definition}

\medskip
\begin{theorem}\label{thm:rsvd-power-law}
	Let $\Matrix{A}$ be a positive semi\-definite matrix, $\hu=\multirsvd(\Matrix{A},\gaussian{0}{1}^{n\times d},q,d)$ for any $q\in\ints$, and $i_0$ be defined as in Definition~\ref{def:power-law}.
	Then
	\[R(\hu) = \bigOmega\left(\left(\frac{d}{d+i_0}\right)^{\frac{1}{2q+1}}\right)\] 
	holds with probability at least $1-e^{-\bigOmega(d)}$.
\end{theorem}

The proof of Theorem~\ref{thm:rsvd-power-law} can be found in Appendix~\ref{proof:rsvd-power-law}.
The idea is to estimate $R_\vect{a}$ on $\vect{a}=\frac{\Matrix{S}^T\eigvecA{1}}{\norm{\Matrix{S}^T\eigvecA{1}}_2}$ and check two possible cases.
If $i_0$ is large, the analysis reduces to Theorem~\ref{thm:rsvd-main}, while if $i_0$ is small, 
we invoke Bernstein-type inequalities and show that $R_\vect{a}=\bigOmega(1)$ with high probability.
So, the overall guarantee of $R_\vect{a}$ is determined by the former case, 
and recalling $R(\hu)\geq\max_{\vect{a}\in\sphere^{d-1}}R_\vect{a}$ yields Theorem~\ref{thm:rsvd-power-law}.

\medskip
\begin{remark}\label{re:rsvd-power-law}
	Theorem~\ref{thm:rsvd-power-law} subsumes Theorem~\ref{thm:rsvd-main} 
	up to a constant factor as ${d+i_0=\bigO(n)}$, 
	and provides a much better guarantee if $\Matrix{A}$ has singular values having at least power-law decay.
	In particular, if {\highlight $i_0=\bigO(d)$} then $R(\hu)=\bigOmega(1)$ with high probability, 
	even with a single pass when $q=1$ and $d=1$.
\end{remark}

\subsection{Indefinite matrices}
\label{sec:rsvd-indefinite}

If \Matrix{A} has negative eigenvalues, the Inequality \eqref{eq:main-intuition1} in the proof of Theorem \ref{thm:rsvd-main} is not valid anymore.
Nevertheless, we expect to have a guarantee of $R(\hu)$ similar to that of Theorem~\ref{thm:rsvd-main} if the negative eigenvalues are not too large. 
We introduce the following technical assumption.\\


\begin{assumption}\label{assumption:rsvd-indefinite}
	Assume there exists a constant $\kappa\in(0,1]$ such that
	{\highlight
	$\sum_{i=2}^n\eigvalA{i}^{2q+1} \geq \kappa \sum_{i=2}^n\abs{\eigvalA{i}}^{2q+1}$. 
	}
\end{assumption}


An important observation is that Theorems~\ref{thm:rsvd-main} and~\ref{thm:rsvd-power-law} 
can be proved by estimating $R_\vect{a}$ only on one specific 
vector $\vect{a}=\frac{\Matrix{S}^T\eigvecA{1}}{\norm{\Matrix{S}^T\eigvecA{1}}_2}$;
see Section~\ref{sec:rsvd-psd}.
Hence, it suffices to use the following lemma 
(proved in Appendix~\ref{proof:rsvd-indefinite}) 
to generalize our results in Section~\ref{sec:rsvd-psd} 
to indefinite matrices satisfying Assumption~\ref{assumption:rsvd-indefinite}.

\medskip
\begin{lemma}\label{lem:rsvd-indefinite} 
	Assume that matrix $\Matrix{A}$ satisfies Assumption~\ref{assumption:rsvd-indefinite} and 
	$\Matrix{S}\sim\gaussian{0}{1}^{n\times d}$. 
	There exists a constant $c_{\kappa}\in(0,1]$ such that with probability at least
	{\highlight
	$1-e^{-\bigOmega(\sqrt{d}\kappa^2)}$}, it holds
	\[\sum_{i=1}^n\eigvalA{i}^{2q+1}\langle \Matrix{S}^T\eigvecA{i}, \Matrix{S}^T\eigvecA{1}\rangle^2 \geq c_{\kappa} \sum_{i=1}^n\abs{\eigvalA{i}}^{2q+1}\langle \Matrix{S}^T\eigvecA{i}, \Matrix{S}^T\eigvecA{1}\rangle^2.\]
\end{lemma}

Lemma~\ref{lem:rsvd-indefinite} essentially states that any indefinite matrix $\Matrix{A}$ 
satisfying Assumption~\ref{assumption:rsvd-indefinite} has $R_\vect{a} = \bigTheta(\bar{R}_\vect{a})$
on such a vector $\vect{a}=\frac{\Matrix{S}^T\eigvecA{1}}{\norm{\Matrix{S}^T\eigvecA{1}}_2}$, where 
\begin{equation}\label{eq:rsvd-Ra-bar}
	\bar{R}_\vect{a} = \frac{\sum_{i=1}^n\abs{\neigval{i}}^{2q+1}\langle \Matrix{S}^T\eigvecA{i}, \vect{a}\rangle^2}{\sum_{i=1}^n\neigval{i}^{2q}\langle \Matrix{S}^T\eigvecA{i}, \vect{a}\rangle^2}.
\end{equation}
The next theorem, proven in Appendix~\ref{proof:rsvd-indefinite-power-law}, 
follows from Lemma~\ref{lem:rsvd-indefinite} and the proof of Theorem~\ref{thm:rsvd-power-law}.

\medskip
\begin{theorem}\label{thm:rsvd-indefinite-power-law}
	Assume that matrix $\Matrix{A}$ satisfies Assumption~\ref{assumption:rsvd-indefinite}.
	Let $\hu=\multirsvd(\Matrix{A},\gaussian{0}{1}^{n\times d},q,d)$ for any $q\in\ints$.
	Then, {\highlight
		\[
			R(\hu) = \bigOmega\left(c_{\kappa}\left(\frac{d}{d+i_0}\right)^{\frac{1}{2q+1}}\right),
		\]  
		with probability at least $1-e^{-\bigOmega(\sqrt{d}\kappa^2)}$.
	} 
\end{theorem}

\medskip
\begin{remark}\label{re:other-rip-distributions}
	As discussed in Section~\ref{sec:preliminary}, all the theorems shown in this section, 
	i.e., Theorems~\ref{thm:rsvd-main}, ~\ref{thm:tightness-rsvd-main}, ~\ref{thm:rsvd-power-law}, and ~\ref{thm:rsvd-indefinite-power-law}, can be easily extended 
	to any random matrix $\Matrix{S}$ satisfying $\Matrix{S}\sim\uniform{\grassmannian{n}{d}}$.
\end{remark}
\section{\uppercase{Extension: combining with projection from Bernoulli}}\label{sec:rsum}

In this section, we propose an extension of Randomized SVD, 
which we name \rsum, and show as Algorithm~\ref{alg:ours}.
In \rsum, half of the columns of $\Matrix{S}$ are replaced with i.i.d.\ samples from a Bernoulli distribution with 
mean $p\in(0,1)$.\footnote{$\bernoulli{p}^{n\times d}$ does not belong to the class of distributions mentioned 
in Section~\ref{sec:preliminary} to which Lemma~\ref{lem:cos-angle-normal} applies.}
We can show that the guarantee achieved by the \rsum\ algorithm for $R(\hu)$ is no worse than that by the
\multirsvd\ algorithm, since half of the coulmns of $\Matrix{S}$ come from a normal distribution.
To study the additional benefits due to the submatrix drawn from the Bernoulli, 
we derive the following lemma as an analog of Lemma~\ref{lem:cos-angle-normal} 
for a Bernoulli random matrix.
The proof is in Appendix~\ref{proof:cos-angle-bernoulli-lb}.

\begin{algorithm}[t]
	$\Matrix{S}_1\sim\gaussian{0}{1}^{n\times\lceil\frac{d}{2}\rceil}$, $\Matrix{S}_2\sim\bernoulli{p}^{n\times\lfloor\frac{d}{2}\rfloor}$\;
	$\Matrix{S} \leftarrow \left[\begin{matrix} \Matrix{S}_1 & \Matrix{S}_2\end{matrix}\right]$\;
	return $\multirsvd(\Matrix{A}, \Matrix{S}, q, d)$;
	\caption{\rsum(\Matrix{A}, $q$, $d$, $p$)}
	\label{alg:ours}
\end{algorithm}

\medskip
\begin{lemma}\label{lem:cos-angle-bernoulli-lb}
	Let $\vect{v}\in\sphere^{n-1}$, $d\leq n/3$, and $\Matrix{S}\sim\bernoulli{p}^{n\times d}$ for a constant $p\in(0,1)$
	Then,
	\[
	\cos^2\theta(\vect{v}, \Matrix{S}) = \bigOmega\left(\frac{\max\{1,\langle \vect{v}, \vecones{n}\rangle^2\}}{n}\right)
	\]
	holds with probability at least $1-e^{-\bigOmega(d)}$.
\end{lemma}

The next theorem, which holds for any p.s.d.\ matrix~\Matrix{A}, 
is a direct consequence of Lemmas~\ref{lem:cos-angle-normal} and~\ref{lem:cos-angle-bernoulli-lb} 
and applying the techniques introduced in Theorem~\ref{thm:rsvd-main}.
The proof is in Appendix~\ref{proof:rsum-main}.

\medskip
\begin{theorem}\label{thm:rsum-main}
	Let \Matrix{A} be a positive semin\-definite matrix with $\eigvalA{1}>0$, 
	and $\hu=\rsum(\Matrix{A},q,d,p)$ for any constant $p\in(0,1)$ and integer $d\geq 2$.
	Then, 
	\[
	R(\hu) = \left(\bigOmega\left(\frac{\max\{d, \langle \eigvecA{1}, \vecones{n}\rangle^2\}}{n}\right)\right)^{\frac{1}{2q+1}}
	\]
	holds with probability at least $1-e^{-\bigOmega(d)}$.
\end{theorem}
{\highlight
Theorem~\ref{thm:rsum-main} shows that $R(\hu)=\Theta(1)$ with high probability when $\langle\eigvecA{1}, \vecones{n}\rangle^2=\Theta(n)$, which is acheviable as the maximum possible value of $\langle\eigvecA{1}, \vecones{n}\rangle^2$ is $n$.
}

\medskip
\begin{remark}\label{re:example-large-sum}
	For certain tasks such as conflicting-group detection \citep{bonchi2019discovering, tzeng2020discovering},
	one could expect to have large $\langle\eigvecA{1}, \vecones{n}\rangle^2$, since $\langle\eigvecA{1}, \vecones{n}\rangle^2$ naturally corresponds to the size of the subgraph, 
	which is located by $\eigvecA{1}$.\footnote{We say that $\eigvecA{1}$ is located around some indices $\mathcal{I}\subseteq [n]$ if the magnitude of $(\eigvecA{1})_i$ for any $i\in\mathcal{I}$ is much larger than those not in $\mathcal{I}$.}
	{\highlight
	However, for tasks such as community detection, $\langle\eigvecA{1}, \vecones{n}\rangle^2\approx 0$ is often the case.
	}
\end{remark}

{\highlight
Finally, we consider the generalization of Theorem~\ref{thm:rsum-main} to indefinite matrices.
To derive Lemma~\ref{lem:rsum-indefinite}, the analog of Lemma~\ref{lem:rsvd-indefinite} for Bernoulli random matrices, we introduce Assumption~\ref{assumption:rsum-indefinite}, where 
($i$) is merely for the ease of presentation and 
($ii$) generalizes Assumption~\ref{assumption:rsvd-indefinite} 
as $\xi_i=1$ for $\Matrix{S}\sim\bernoulli{p}^{n\times d}$.
The proof of Lemma~\ref{lem:rsum-indefinite} can be found in Appendix~\ref{proof:rsum-indefinite}.
}

\medskip
\begin{assumption}\label{assumption:rsum-indefinite}
	{\highlight
	Assume that (i) $\langle\vect{u}_1,\vecones{n}\rangle^2=\bigOmega(1)$ and
	(ii) there exists a constant $\kappa'\in(0,1]$ such that
		\[\sum_{i=2}^n\eigvalA{i}^{2q+1}\xi_i \geq \kappa' \sum_{i=2}^n\abs{\eigvalA{i}}^{2q+1}\xi_i,\]
		where $\xi_i=\expectation{\langle\Matrix{S}^T\vect{u}_i,\frac{\vecones{d}}{\sqrt{d}}\rangle^2}$, 
		for all $i\in[n]$. 
	}
\end{assumption}

\medskip

\begin{lemma}\label{lem:rsum-indefinite}
	Assume that $\Matrix{A}$ satisfies Assumption~\ref{assumption:rsum-indefinite}.
	Let $\Matrix{S}\sim\bernoulli{p}^{n\times d}$ for a constant $p\in(0,1)$.
	There exists a constant $c_{\kappa'}\in(0,1]$, such that
	\[
	\sum_{i=1}^n\eigvalA{i}^{2q+1}\langle \Matrix{S}^T\eigvecA{i}, \Matrix{S}^T\eigvecA{1}\rangle^2 \geq c_{\kappa'} \sum_{i=1}^n\abs{\eigvalA{i}}^{2q+1}\langle \Matrix{S}^T\eigvecA{i}, \Matrix{S}^T\eigvecA{1}\rangle^2,
	\]
	with probability at least $1-e^{-\bigOmega(\sqrt{d}{\kappa'}^2)}$. 
\end{lemma}

Our last result, Theorem~\ref{thm:rsum-indefinite-stronger}, immediately follows from 
Theorems~\ref{thm:rsvd-indefinite-power-law} and \ref{thm:rsum-main} and Lemma~\ref{lem:rsum-indefinite}.
{\highlight
The proof and the full version are in Appendix~\ref{proof:rsum-indefinite-stronger}.
}

\medskip
\begin{theorem}\label{thm:rsum-indefinite-stronger}
	Assume that $\Matrix{A}$ satisfies Assumptions~\ref{assumption:rsvd-indefinite} and ~\ref{assumption:rsum-indefinite}.
	Let $\hu=\rsum(\Matrix{A},q,d,p)$ for any constant $p\in(0,1)$ and any $q\in\ints$, 
	and $i_0$ be defined as in Definition~\ref{def:power-law}.
	Then,
	\[R(\hu) = \bigOmega\left(\left(\max\left\{\frac{d}{d+i_0}, \frac{\langle \eigvecA{1}, \vecones{n}\rangle^2}{n}\right\}\right)^{\frac{1}{2q+1}}\right)\]  holds with probability at least $1-e^{-\bigOmega(\sqrt{d})}$.
	{\highlight
		(For the full dependency on $\kappa$, $\kappa'$, $c_{\kappa}$, and $c_{\kappa'}$, 
		see Appendix~\ref{proof:rsum-indefinite-stronger}.)
	}
\end{theorem}

\section{\uppercase{Experiments}}

In this section we evaluate the randomized algorithms we analyze in this paper 
using synthetic and real-world datasets. 
In Section~\ref{sec:exp-synthetic}, we use synthetic datasets to benchmark the \multirsvd algorithm 
with respect to the $R$ measure, 
and study the effect of its parameters. 
In Section~\ref{sec:exp-real}, 
we employ \multirsvd and \rsum as sub\-routines of spectral approaches 
for specific knowledge-discovery tasks on real-world datasets.

\para{Settings.}
We use \lancoz, provided by the ARPACK library \citep{lehoucq1998arpack}, 
for computing~$\eigvalA{1}$, 
which is required for measuring $R$.
We fix $q=1$ while varying $d\in\{1,5,10,25,50\}$ to study the effect of~$d$, and 
fix $d=10$ while varying $q\in\{1,2,4,8,16\}$ to study the effect of~$q$.
Each setting is repeated $100$ times and the average is reported.
All experiments are performed on an Intel\,Core\,i5 machine at 1.8\,GHz with 8\,GB\,RAM. 
All methods are implemented in Python~3.7.4.\footnote{The code is available at the github repo \url{https://bit.ly/34dI4Nl}.}

\subsection{Evaluation with synthetic data}
\label{sec:exp-synthetic}

We consider different types of eigenvalue distributions, 
also illustrated in Figure~\ref{fig:exp-syn-eigval}.
The size of the input matrix is set to $n=10\,000$ and $i_0=100$ 
(see Definition~\ref{def:power-law}).
For all types of synthetic matrices
we set $\eigvalA{i}=i^{-0.01}$, 
for $i<i_0$, 
and the rest of the eigenvalues $\{\eigvalA{i}\}_{i\geq i_0}^n$ are specified as follows:
\begin{itemize}
	\item Type 1: $\eigvalA{i}=i^{-1}$ for $i\geq i_0$.
	\item Type 2: $\eigvalA{i}=i^{-\frac{1}{7}}$ for $i\geq i_0$.
	\item Type 3:
	$\eigvalA{i}=
	\begin{cases}
		i^{-\frac{1}{3}} & \text{if }i\in[i_0, \frac{2n}{3}],\\ 
		-(i-\frac{2n}{3})^{-1} & \text{if }i>\frac{2n}{3}.
	\end{cases}$
	\item Type 4:
	$\eigvalA{i}=
	\begin{cases}
		i^{-\frac{1}{2}} & \text{if }i\in[i_0, \frac{n}{2}],\\
		-\frac{9}{10}(i-\frac{n}{2})^{-\frac{1}{2}} & \text{if }i\in(\frac{n}{2}, n-i_0),\\
		-\frac{9}{10}i^{-0.01} & \text{if }i\geq n-i_0.
	\end{cases}$
\end{itemize}

For the value of $\kappa$ in Assumption~\ref{assumption:rsvd-indefinite}, 
we compute $\kappa$ with $q=1$ and get: 
$\kappa=1$ for Type 1 and Type 2, 
$\kappa=0.99$ for Type 3, and 
$\kappa=0.22$ for Type 4.
For each type of eigenvalue distribution, 
we generate a random $n\times n$ input matrix by sampling the eigenvectors uniformly from
the space of orthogonal matrices.

\begin{figure}[t]
	\begin{center}
		\includegraphics[width=0.6\linewidth]{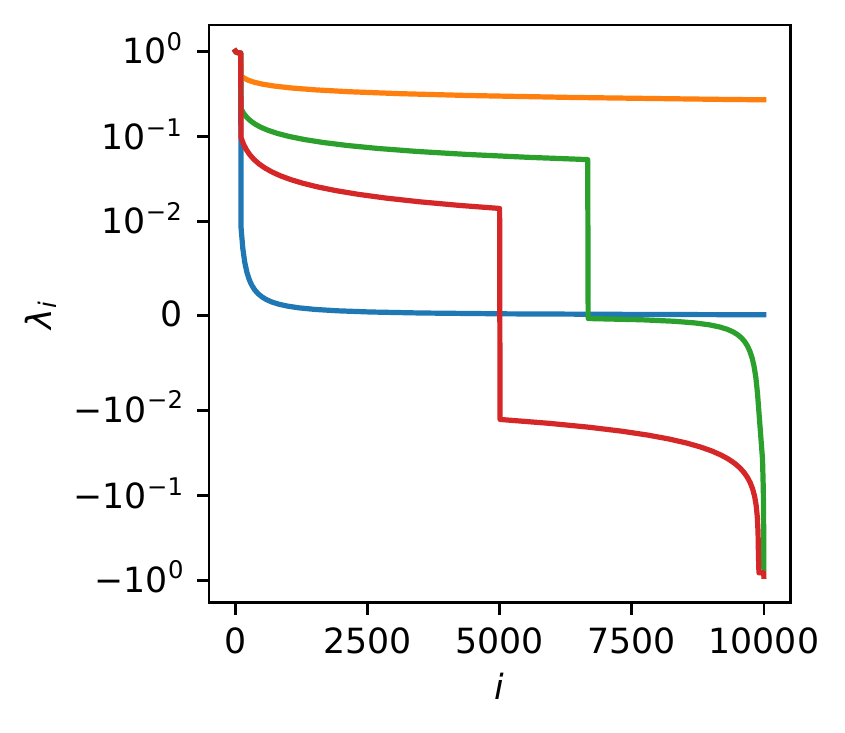}
		\includegraphics[width=0.9\linewidth]{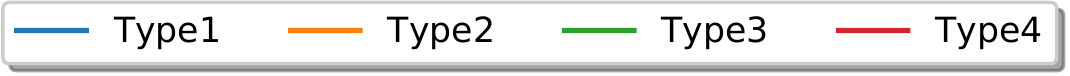}
	\end{center}
	\caption{Different types of eigenvalue distributions.}
	\label{fig:exp-syn-eigval}
\end{figure}

\begin{figure}[t]
	\begin{center}
		\includegraphics[width=0.99\linewidth]{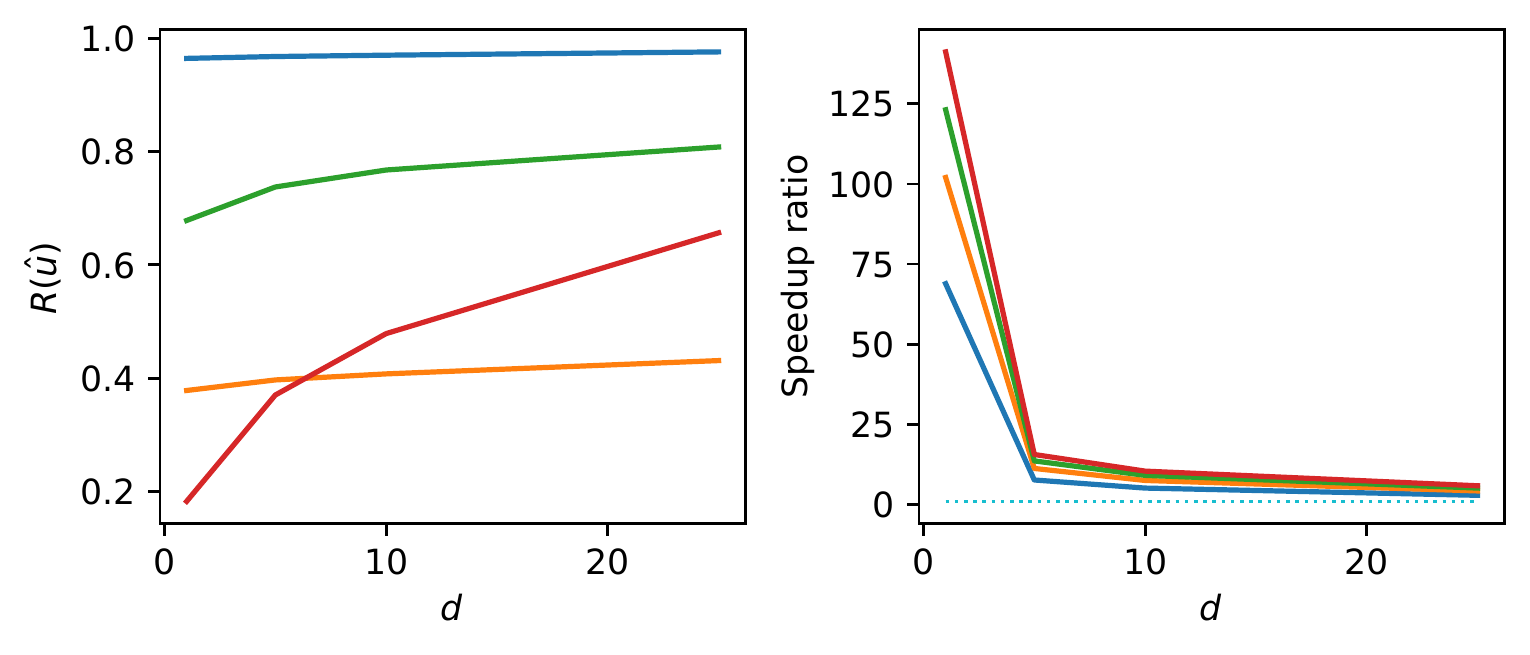}
		\includegraphics[width=0.99\linewidth]{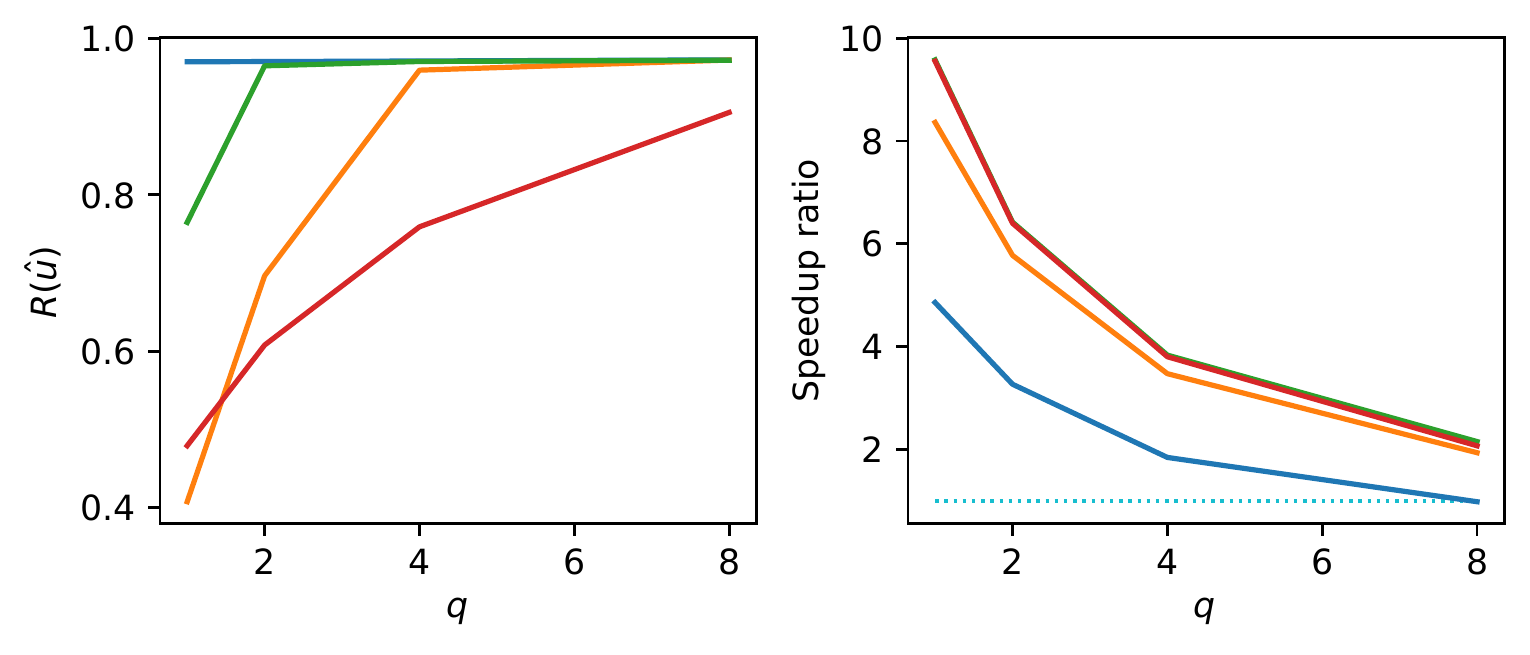}
		\includegraphics[width=0.9\linewidth]{figs/synthetic-legend.pdf}
	\end{center}
	\caption{The value of $R(\hu)$ for $\hu$ computed by 
	$\multirsvd(\Matrix{A},\gaussian{0}{1}^{n\times d},q,d)$. 
	Top row shows dependence with~$d$. Bottom row shows dependence with~$q$. 
	The speedup is measured against \lancoz.}
	\label{fig:exp-syn}
\end{figure}

Figure~\ref{fig:exp-syn} shows the value of $R$ 
for the vector \hu computed by $\multirsvd(\Matrix{A}, \gaussian{0}{1}^{n\times d}, q, d)$,
and the speedup in running time against \lancoz.

For matrices of Type 1, 
it is expected that \multirsvd performs the best as the eigenvalues of such matrices
have the fastest decay and $\kappa=1$.

For matrices of Type 2, 
we notice that $R(\hu)$ is very close to 1 when $q\geq 4$.
This result is better than what our analysis predicts
since by Theorem~\ref{thm:rsvd-power-law} it is 
$R(\hu)=\bigOmega(1)$ with high probability after $q=7$ (since the decay rate of Type 2 is $1/7$).

For matrices of Type 3, 
despite being indefinite, the magnitude of the negative eigenvalues is almost negligible ($\kappa=0.99$).
By Theorem~\ref{thm:rsvd-indefinite-power-law} and Lemma~\ref{lem:rsvd-indefinite}, 
$R(\hu)$ is nearly identical to its counterpart $\bar{R}$ (see~\eqref{eq:rsvd-Ra-bar}), 
so it is expected that \multirsvd performs better on data of Type 3 than on data of Type 2, 
as the eigen\-value-distribution decay rate is faster.

For matrices of Type 4, 
although the eigenvalues decay faster than those of Type 3 matrices, 
the magnitudes of the negative eigenvalues are much larger ($\kappa=0.22$).
By Theorem~\ref{thm:rsvd-indefinite-power-law} and Lemma~\ref{lem:rsvd-indefinite}, 
$R(\hu)$ is upper-bounded by a factor of $\kappa$ when increasing $q$, 
and the results indeed show that the performance of \multirsvd 
is worse for Type 4 matrices, compared to Type 3 ($\kappa=0.99$).

\subsection{Applications on real-world data}
\label{sec:exp-real}

We use publicly-available networks from the SNAP collection \citep{snapnets}.
Statistics of the datasets are listed in Tables~\ref{tab:exp-signed} and~\ref{tab:exp-unsigned}.

\subsubsection{Detection of $2$ conflicting groups}
\label{sec:exp-cg}

The problem of $2$-conflicting group detection aims to find two optimal groups 
that maximize the polarity objective $P(\vect{x})=\vect{x}^T\Matrix{A}\vect{x}/\vect{x}^T\vect{x}$, 
where $\Matrix{A}$ is the signed adjacency matrix and 
$x\in\mathcal{T}=\{0,\pm1\}^n\backslash\{\veczeros{n}\}$.
\citet{bonchi2019discovering} propose a tight $\bigO(n^{1/2})$-approximation algorithm 
based on the leading eigenvector~$\eigvecA{1}$. 
In Appendix~\ref{sec:conflicting-group-approx} we show that applying their approach on the approximated top-eigenvector $\hu$ yields an $\bigO(n^{1/2}/R(\hu))$-approx algorithm.

\begin{table}[t]
	\caption{\label{tab:exp-signed} Datasets for conflicting group detection.}
	\centering
	\resizebox{0.45\textwidth}{!}{%
		\begin{tabular}{lrrrr}\toprule
			& \wikiv & \referendum & \slashdot & \wikic \\\midrule
			$|V|$ & 7\,115 & 10\,884 & 82\,140 & 116\,717 \\
			$|E|$ & 100\,693 & 251\,406 & 500\,481 & 2\,026\,646 \\
			$(\gamma, i_0)$ & (4.6, 15) & (4.5, 16) & (5.3, 17) & (2.8, 22) \\
			$\kappa$ & 0.397 & 0.620 & 0.204 & 0.034 \\
			$\cos\theta(\eigvecA{1},\vecones{n})$ & 0.378 & 0.399 & 0.194 & 0.193 \\
			\bottomrule
		\end{tabular}%
	}
\end{table}

\begin{table}[t]
	\caption{\label{tab:exp-unsigned} Datasets for community detection.}
	\centering
	\resizebox{0.45\textwidth}{!}{%
		\begin{tabular}{lrrrr}\toprule
			& \fbart &  \ptop & \youtube & \roadca \\\midrule
			$|V|$ & 50\,515 & 62\,586 & 1\,134\,890 & 1\,965\,206 \\
			$|E|$ & 819\,306 & 147\,892 & 2\,987\,624 & 2\,766\,607 \\
			\bottomrule
		\end{tabular}%
	}
\end{table}

\para{Datasets.}
The statistics of datasets we use for this experiment are listed in Table~\ref{tab:exp-signed}.
We observe that all datasets have rapidly-decaying singular values.
To measure the parameters $\gamma$ and $i_0$ (see Definition~\ref{def:power-law}),
due to memory limitations, we compute the top $1\,000$ eigenvalues (in magnitude) 
of its signed adjacency matrix by \lancoz, 
and fit the parameters $(\gamma,i_0)$ by an MLE-based method \citep{clauset2009power}.
Moreover, we test the validity of 
Assumption~\ref{assumption:rsvd-indefinite} by computing $\kappa$ with $q=1$, 
and also computing $\langle\eigvecA{1},\vecones{n}\rangle$.


\para{Results.}
Figure~\ref{fig:exp-cg} illustrates the results obtained 
by applying the spectral algorithm of \citet{bonchi2019discovering} 
on the top-eigenvector $\hu$ returned by \multirsvd and \rsum.
Due to the value of $\kappa$, 
the result is that, as expected,
both algorithms perform the best on \referendum.
Due to the value of $\cos\theta(\eigvecA{1}, \vecones{n})$, 
the superiority of \rsum over \multirsvd is, as expected, 
more pronounced on \wikiv and \referendum than on \slashdot and \wikic.

\begin{figure}[t] 
	\begin{center}
		\includegraphics[width=0.99\linewidth]{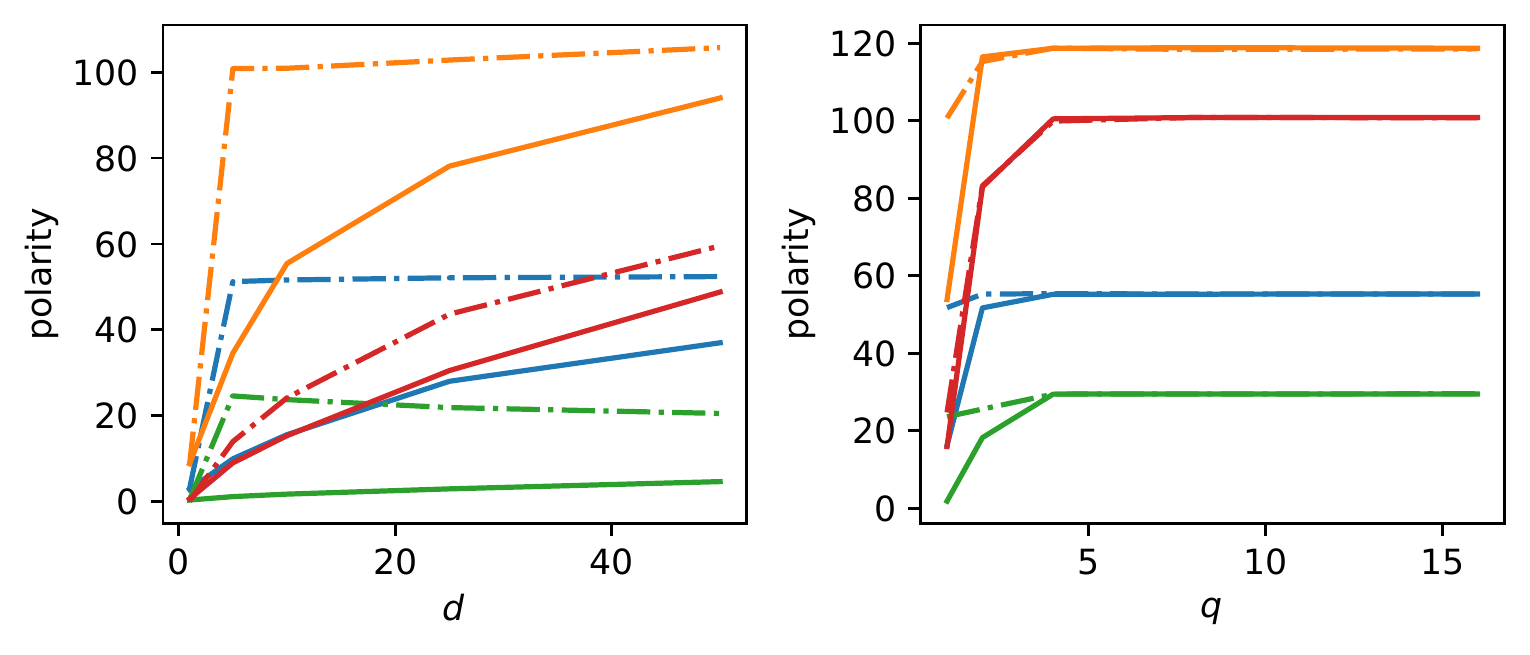}
		\includegraphics[width=0.99\linewidth]{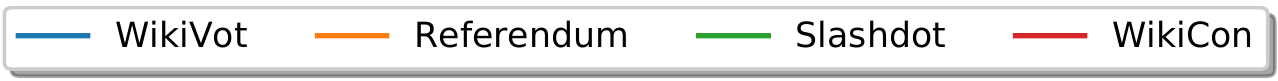}
	\end{center}
	\caption{Results on the task of detecting $2$ conflicting groups. 
	Results for \multirsvd (resp. \rsum) are plotted with a solid (resp. dashed) line.}
	\label{fig:exp-cg}
\end{figure}

\begin{figure}[t] 
	\begin{center}
		\includegraphics[width=0.99\linewidth]{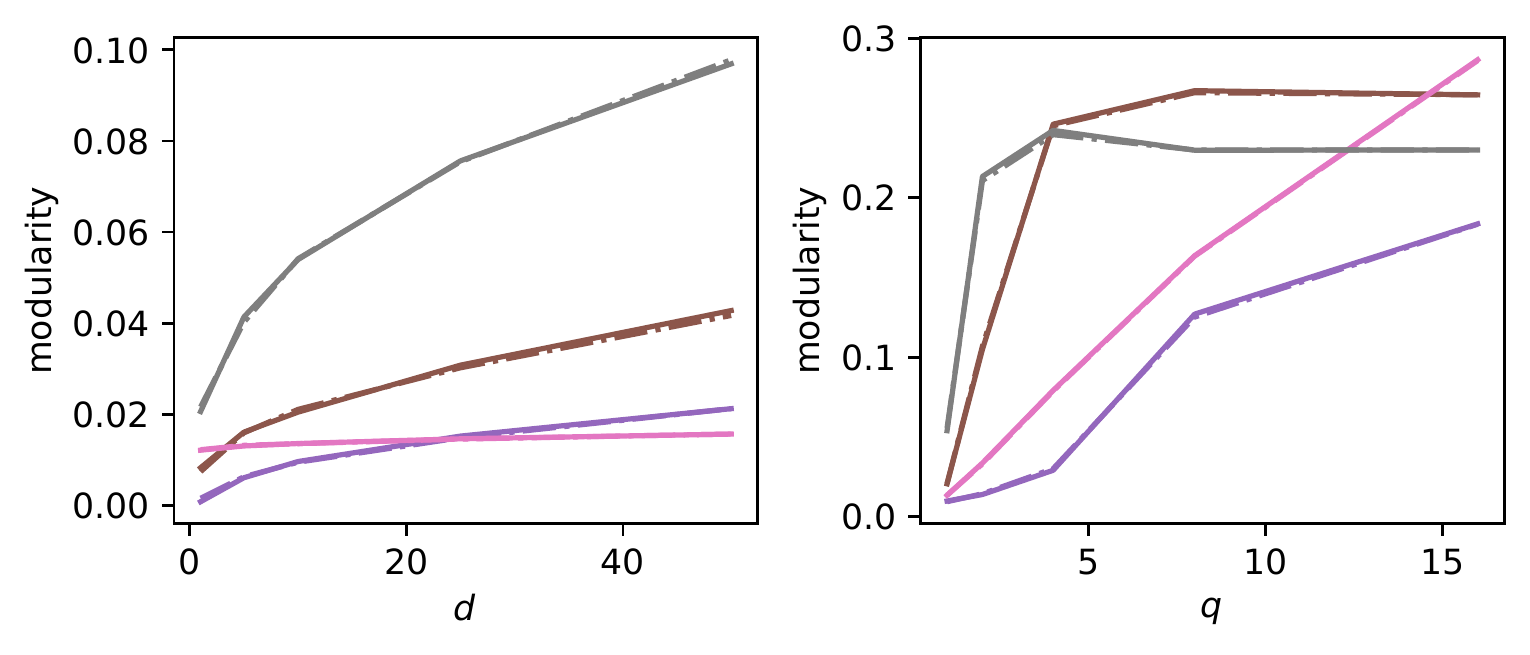}
		\includegraphics[width=0.99\linewidth]{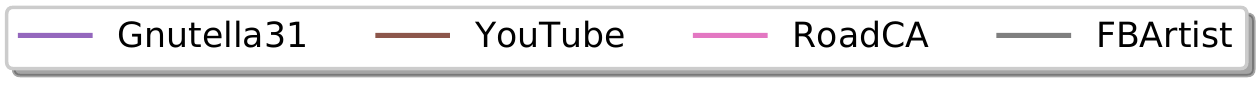}
	\end{center}
	\caption{
		Results on the task of detecting $2$ communities. 
		{\highlight
			Results for \multirsvd (in solid line) and \rsum (in dotted line) are nearly the same.}
		}
	\label{fig:exp-mod}
\end{figure}

\subsubsection{Detection of $2$ communities}
\label{sec:exp-mod}

For the task of detecting two communities in a graph,
\citet{newman2006modularity} proposed an efficient algorithm
by maximizing the modularity score 
$Q(\vect{x})=\vect{x}^T\Matrix{M}\vect{x}/4|E|$, 
where $\Matrix{M}_{i,j}=\Matrix{A}_{i,j} - \deg(i)\deg(j)/2|E|$, 
\Matrix{A} is the adjacency matrix of the input graph, 
and the two communities are determined by the sign of the top eigenvector of~$\Matrix{M}$.

\para{Datasets.}
The datasets used for evaluating this task are listed in Table~\ref{tab:exp-unsigned}.
As the modularity matrix $\Matrix{M}$ is dense and the networks are large, 
\lancoz runs out-of-memory on our machine when trying to compute the top eigenvalues,
{\highlight
and hence, unlike Table~\ref{tab:exp-signed}, the number $\kappa$ and the parameters $(\gamma, i_0)$ 
are not displayed in Table~\ref{tab:exp-unsigned}.
}
\note[Ruochun]{$\uparrow$ resolves the suggestion on wording by Reviewer\#6}

\para{Results.}
Figure~\ref{fig:exp-mod} shows the results by applying the spectral algorithm of \cite{newman2006modularity} 
on the top-eigenvector $\hu$ returned by \multirsvd and \rsum.
Notice that on this task, \rsum has no advantage over \multirsvd since $\Matrix{M}\vecones{n}=0$, 
and thus $\langle\eigvecA{1},\vecones{n}\rangle=\veczeros{n}$ if $\eigvalA{1}\geq 0$.
When fixing $d=10$ and increasing $q$, the modularity scores converge much faster on \fbart and \youtube 
than on \ptop and \roadca, suggesting that it could be hard to discover community structures in \ptop and \roadca.
This is an expected result. 
For Gnutella (\ptop) the design of the network prevents the formation of large communities so as to enable 
reliable communication
For the road network of California (\roadca) the reason is the grid-like structure of the network \citep{leskovec2009community}.

\section{\uppercase{Conclusion}}

In this paper, we study the problem of approximating the leading eigenvector 
of a matrix with limited number of passes.
The problem is of interest in many applications.
We provide a tight theoretical analysis of the popular randomized SVD method, 
with respect to the metric 
$R(\hu)=\eigvalA{1}^{-1}{\hu^T\Matrix{A}\hu}/{\hu^T\hu}$.
Our results substantially improve the analysis of randomized SVD 
in the regime of $o(\ln n)$ passes and recover the analysis of prior works 
in the regime of $\bigOmega(\ln n)$ passes.
A new technique is introduced to transform the problem of maximizing $R(\hu)$ 
into a well-studied problem in the literature of high-dimensional probability.

Our work opens several interesting directions.
First, it is an open problem to characterize the fundamental limit of maximizing $R(\hu)$ 
for any algorithm with fixed number of pass and $\bigO(n)$ space.
Second, our results may be extended in different ways.
For example, we may relax the requirement on the input matrix from symmetric to stochastic, so as to analyze approximations of Page\-Rank~\citep{page1999pagerank}.
{\highlight
Or, we may extend \rsum to use any non-centered subgaussian distribution for drawing $\Matrix{S}_2$, and we conjecture this yields similar results.} 
Another direction is to extend our analysis to top-$k$ eigenvectors;
since there are already several 
methods for computing top-$k$ eigenvectors~\citep{halko2011finding,mackey2008deflation,allen2016lazysvd},
the most challenging part is to define the proper metric to maximize,
as a generalization of~$R(\hu)$.

\subsubsection*{Acknowledgements}
We thank the anonymous reviewers for their insightful feedback. 
This research is supported by the ERC Advanced Grant REBOUND (834862), 
the EC H2020 RIA project SoBigData++ (871042), 
and the Wallenberg AI, Autonomous Systems and Software Program (WASP) funded by the Knut and Alice Wallenberg Foundation. 

\bibliography{sscg}

\appendix
\onecolumn \makesupplementtitle

\section{Proofs of \multirsvd}\label{sec:proof-rsvd}
\subsection{Large deviation of projection length for Gaussian random matrix}
This subsection is devoted to proving Lemma~\ref{lem:cos-angle-normal} restated below.\\

\begin{replemma}{lem:cos-angle-normal}\label{proof:cos-angle-normal}
	Let $\vect{v}\in\reals^n$ be a nonzero vector and $\Matrix{S}\sim\gaussian{0}{1}^{n\times d}$ where $n,d\in\ints$ and $n\geq d$.
	Then,
	\[
	\cos^2\theta(\vect{v},\Matrix{S}) = \bigTheta\left(\frac{d}{n}\right)
	\]
	with probability at least $1-e^{-\bigOmega(d)}$.\\
\end{replemma}

This lemma stems from the observations that $\frac{\sigval{1}{\Matrix{S}^T\vect{v}}}{\sigval{1}{\Matrix{S}}} \leq \cos\theta(\vect{v},\Matrix{S}) \leq \frac{\sigval{1}{\Matrix{S}^T\vect{v}}}{\sigval{d}{\Matrix{S}}}$ and the distribution of $\frac{\Matrix{S}^T \vect{v}}{\left\|\vect{v}\right\|_2}$ is exactly $\gaussian{0}{1}^{d\times 1}$.
The proof relies on the union bound of concentration inequalities on the extreme singular values of Gaussian random matrix, Lemma \ref{lem:largest-sigval-gaussian}, and Lemma \ref{lem:smallest-sigval-gaussian}. 
Similar inequalities shown in the previous works, e.g. \cite{hardt2014noisy}, also rely on this observation.\\

\medskip
\noindent

\begin{lemma}[Theorem 4.4.5~\citep{vershynin2018high}]\label{lem:largest-sigval-gaussian}
	Let $\Matrix{S}$ be a $n\times d$ random matrix whose entries are i.i.d.\ zero-mean subgaussian r.v.'s.
	\[
	\text{ For all } t>0,\quad
	\prob{\sigval{1}{\Matrix{S}} \geq c\,(\sqrt{n} + \sqrt{d} +t)} \leq 2e^{-t^2},
	\]
	where $c>0$ depends linearly only on $\norm{\Matrix{S}_{1,1}}_{\psi_2}$ (see Definition~\ref{def:subgaussian-norm} of $\psi_2$-norm in Appendix~\ref{app:sub}).
\end{lemma}

\medskip
\noindent

\begin{lemma}[Theorem 1.1 \citep{rudelson2009smallest}]\label{lem:smallest-sigval-gaussian}
	Let $\Matrix{S}$ be a $n\times d$ random matrix whose entries are i.i.d.\ zero-mean subgaussian r.v.'s and $n\geq d$.
	\[\text{ For all } \delta>0,\quad
	\prob{\sigval{d}{\Matrix{S}} \leq \delta\,(\sqrt{n} - \sqrt{d-1})} \leq (c_1\delta)^{n-d+1} + e^{-c_2n},\]
	where $c_1,c_2>0$ have polynomial dependence on $\norm{\Matrix{S}_{1,1}}_{\psi_2}$ (see Definition~\ref{def:subgaussian-norm} of $\psi_2$-norm in Appendix~\ref{app:sub}).
\end{lemma}

\medskip
\noindent

\begin{proofof}{Lemma~\ref{lem:cos-angle-normal}}
	For the simplicity of presentation, we assume $\norm{\vect{v}}_2=1$ as $\cos\theta(\cdot,\cdot)$ is scale-invariant.

	\underline{(i) $\cos\theta(\vect{v},\Matrix{S}) = \Omega (\sqrt{d/n})$:}\\

	Recall that $\cos\theta(\vect{v},\Matrix{S}) = \max_{\vect{a}\in\sphere^{d-1}}\frac{\langle \vect{v}, \Matrix{S}\vect{a}\rangle}{\norm{\Matrix{S}\vect{a}}_2} $.  
	Let $\vect{a}=\Matrix{S}^T\vect{v}/\norm{\Matrix{S}^T\vect{v}}$. 
	We get
	\begin{align*}
		\cos\theta(\vect{v},\Matrix{S})  \geq \frac{\langle \vect{v}, \Matrix{S}\Matrix{S}^T\vect{v}\rangle}{\norm{\Matrix{S}\Matrix{S}^T\vect{v}}_2} = \frac{\norm{\Matrix{S}^T\vect{v}}_2^2}{\norm{\Matrix{S}\Matrix{S}^T\vect{v}}_2} \geq \frac{\norm{\Matrix{S}^T\vect{v}}_2}{\sigval{1}{\Matrix{S}}} = \frac{\sigval{1}{\Matrix{S}^T\vect{v}}}{\sigval{1}{\Matrix{S}}},
	\end{align*}
	where the second inequality directly follows from the definitions of the largest singular value. Because $\Matrix{S}^T \vect{v}\sim \gaussian{0}{1}^{d\times 1}$, invoking Lemma \ref{lem:smallest-sigval-gaussian} with $\delta = e^{-1}$ yields that $\prob{\sigval{1}{\Matrix{S}^T\vect{v}} \ge \sqrt{d}/e} \geq 1-e^{-\Omega(d)}$. Meanwhile, Lemma~\ref{lem:largest-sigval-gaussian} with $t=\sqrt{n}-\sqrt{d}$ implies that $\prob{\sigval{1}{\Matrix{S}}\le 2c\sqrt{n}}\ge 1-e^{-\Omega(n)}$. We hence conclude (i) by applying the union bound. 

	\medskip
	\noindent

	\underline{(ii) $\cos\theta(\vect{v},\Matrix{S}) = \bigO (\sqrt{d/n})$:}\\

	Due to $\sigval{d}{\Matrix{S}}\le \left\|\Matrix{S}\right\|_2$ and $\langle \vect{v}, \Matrix{S}\vect{a}\rangle \le \left\|\Matrix{S}^T\vect{v}\right\|_2 \left\|\vect{a}\right\|_2= \sigval{1}{\Matrix{S}^T\vect{v}}$, $\text{ for all } \vect{a}\in \sphere^{d-1}$,
	\[
	\cos \theta(\vect{v},\Matrix{S})
	= \max_{\vect{a}\in\sphere^{d-1}}\frac{\langle \vect{v}, \Matrix{S}\vect{a}\rangle}{\norm{\Matrix{S}\vect{a}}_2}
	\le \frac{\sigval{1}{\Matrix{S}^T\vect{v}}}{\sigval{d}{\Matrix{S}}}.
	\]
	For the denominator, Lemma \ref{lem:smallest-sigval-gaussian} with $\delta = e^{-1}$ is applied to permit that $\prob{\sigval{d}{\Matrix{S}} \ge \frac{\sqrt{n}-\sqrt{d-1}}{e}} \geq 1-e^{-\Omega (n-d+1)} - e^{-\Omega (n)}$.
	For the numerator, as $\Matrix{S}^T\vect{v} \sim \gaussian{0}{1}^{d\times 1}$, Lemma \ref{lem:largest-sigval-gaussian} with $t=\sqrt{d}$ shows that $\prob{\sigval{1}{\Matrix{S}^T\vect{v}} \le 2\sqrt{d}} \geq 1-e^{-\Omega (d)}$.
	Thus, (ii) holds by applying the union bound.
\end{proofof}

\subsection{\multirsvd with positive semidefinite matrices}\label{sec:A1}

\begin{lemma}\label{lem:psd-rsvd-holder}
	Let $\vect{x}=(\vect{x}_1,\ldots,\vect{x}_n)$ and $\vect{y}=(\vect{y}_1,\ldots,\vect{y}_n)$ be two vectors in $\reals^n$ satisfying (i) there exists $i\in [n]$ s.t.\ $\vect{x}_i \vect{y}_i\neq 0$, and (ii) there exists $j\in [n]$ s.t.\ $\vect{y}_j\neq 0$. Then for all $q\in \mathbb{N}$,
	\[
	\frac{\sum_{i=1}^n\left|\vect{x}_i\right|^{2q+1}\vect{y}_i^2}{\sum_{i=1}^n \left|\vect{x}_i\right|^{2q}\vect{y}_i^2} \geq \left(\frac{\sum_{i=1}^n\left|\vect{x}_i\right|^{2q}\vect{y}_i^2}{\sum_{i=1}^n \vect{y}_i^2}\right)^{\frac{1}{2q}}.
	\]
\end{lemma}
\begin{pf}
	For any $n$-dimensional vectors $\vect{a}=(\vect{a}_1,\ldots,\vect{a}_n)$, $\vect{v}=(\vect{b}_1,\ldots,\vect{b}_n)$ satisfying that (i)$^\prime$ there exists $i\in [n]$ s.t.\ $\vect{a}_i \vect{b}_i\neq 0$, and (ii)$^\prime$ there exists $j\in [n]$ s.t.\ $\vect{b}_j\neq 0$,  Hölder's inequality implies that
	\begin{equation}\label{eq:holder-statement}
	\left(\frac{\sum_{i=1}^n\abs{\vect{a}_i}^{r}}{ \sum_{i=1}^n \abs{\vect{a}_i\vect{b}_i}   }\right)^{\frac{1}{r}} \ge \left( \frac{ \sum_{i=1}^n \abs{\vect{a}_i\vect{b}_i}    }{            \sum_{i=1}^n\abs{\vect{b}_i}^s  }\right)^{\frac{1}{s}},
	\end{equation}
	where $r,s \in [1,\infty]$ with $1/r+1/s=1$. Let $\vect{a}_i=\abs{\vect{x}_i}^{2q}\vect{y}_i^{2/r}$ and $\vect{b}_i=\abs{\vect{y}_i}^{2/s}, \text{ for all } i\in[n]$, then (i) and (ii) imply (i)$^\prime$ and (ii)$^\prime$ respectively. Hence, (\ref{eq:holder-statement}) with $r=(2q+1)/2q$, $s=2q+1$ gives us that
	\[
	\left(\frac{\sum_{i=1}^n\left(\left|\vect{x}_i\right|^{2q}\vect{y}_i^{\frac{4q}{2q+1}}\right)^{\frac{2q+1}{2q}}}{\sum_{i=1}^n\left|\vect{x}_i\right|^{2q}\vect{y}_i^2}\right)^{\frac{2q}{2q+1}}
	\ge \left(\frac{\sum_{i=1}^n\left|\vect{x}_i\right|^{2q}\vect{y}_i^2}{\sum_{i=1}^n\left(\vect{y}_i^{\frac{2}{2q+1}}\right)^{2q+1}}\right)^{\frac{1}{2q+1}}.
	\]
	We conclude this lemma by rearranging the above inequality.
\end{pf}

\medskip
\noindent

\begin{reptheorem}{thm:rsvd-main}\label{proof:rsvd-main}
	Let \Matrix{A} be a positive semidefinite matrix with $\eigvalA{1}>0$ and $\hu=\multirsvd(\Matrix{A},\gaussian{0}{1}^{n\times d},q,d)$ for any $q\in\ints$.
	Then,
	\[
	R(\hu) = \left(\bigOmega\left(\frac{d}{n}\right)\right)^{\frac{1}{2q+1}}
	\]
	holds with probability at least $1-e^{-\bigOmega(d)}$.
\end{reptheorem}
\begin{pf}
	Thanks to Lemma \ref{lem:cos-angle-normal}, the proof follows if the following inequality holds almost surely
	\begin{equation}\label{eq:thm1-1}
		R(\hu)^{2q+1}		\ge \max_{\vect{a}\in\sphere^{d-1}}\frac{\langle \Matrix{S}^T\eigvecA{1}, \vect{a}\rangle^2}{\sum_{i=1}^n\langle \Matrix{S}^T\eigvecA{i}, \vect{a}\rangle^2} = \cos^2\angl{\eigvecA{1}}{\Matrix{S}},
	\end{equation}
	where the equation is due to Definition~\ref{def:smallest-principal-angle}. 
	We show~(\ref{eq:thm1-1}) by Lemma~\ref{lem:psd-rsvd-holder} and the alternating form of $R(\hu)$  follows by \eqref{eq:rsvd-Ra2} in Section~\ref{sec:rsvd-psd},
	\begin{equation}\label{eq:thm1-2}
		R(\hu)=\max_{\vect{a}\in \sphere^{d-1} }	R_{\vect{a}} = \max_{\vect{a}\in \sphere^{d-1} }	\frac{\sum_{i=1}^n\neigval{i}^{2q+1}\langle \Matrix{S}^T\eigvecA{i}, \vect{a}\rangle^2}{\sum_{i=1}^n\neigval{i}^{2q}\langle \Matrix{S}^T\eigvecA{i}, \vect{a}\rangle^2}.
	\end{equation}
	Let $\vect{x}_i=\neigval{i}$ and $\vect{y}_i= \langle \Matrix{S}^T\eigvecA{i}, \vect{a}\rangle, \, \text{for all } i\in [n]$, because $\langle \Matrix{S}^T \eigvecA{1},a\rangle \neq 0$ a.e., the conditions of Lemma~\ref{lem:psd-rsvd-holder}, (i) and (ii)., hold a.e.. Therefore, it holds almost surely that
	\begin{align*}
			R_{\vect{a}} &= \frac{\sum_{i=1}^n\neigval{i}^{2q+1}\langle \Matrix{S}^T\eigvecA{i}, \vect{a}\rangle^2}{\sum_{i=1}^n\neigval{i}^{2q}\langle \Matrix{S}^T\eigvecA{i}, \vect{a}\rangle^2}\ge \left(   \frac{\sum_{i=1}^n\neigval{i}^{2q}\langle \Matrix{S}^T\eigvecA{i}, \vect{a}\rangle^2}{\sum_{i=1}^n\langle \Matrix{S}^T\eigvecA{i}, \vect{a}\rangle^2}\right)^{\frac{1}{2q}}\\
			&=\left(\frac{\sum_{i=1}^n\neigval{i}^{2q}\langle \Matrix{S}^T\eigvecA{i}, \vect{a}\rangle^2}  {\sum_{i=1}^n\neigval{i}^{2q+1}\langle \Matrix{S}^T\eigvecA{i}, \vect{a}\rangle^2} \frac{\sum_{i=1}^n\neigval{i}^{2q+1}\langle \Matrix{S}^T\eigvecA{i}, \vect{a}\rangle^2}{\sum_{i=1}^n\langle \Matrix{S}^T\eigvecA{i}, \vect{a}\rangle^2}\right)^{\frac{1}{2q}} =\left(R_{\vect{a}}^{-1} \frac{\sum_{i=1}^n\neigval{i}^{2q+1}\langle \Matrix{S}^T\eigvecA{i}, \vect{a}\rangle^2}{\sum_{i=1}^n\langle \Matrix{S}^T\eigvecA{i}, \vect{a}\rangle^2}\right)^{\frac{1}{2q}},
	\end{align*}
where the last equation follows from~\eqref{eq:rsvd-Ra2} in Section~\ref{sec:rsvd-psd} again. 
Rearranging the above inequality, we get that
\begin{equation}\label{eq:thm1-3}
	R_{\vect{a}}^{2q+1}
	\geq \frac{\sum_{i=1}^n\neigval{i}^{2q+1}\langle \Matrix{S}^T\eigvecA{i}, \vect{a}\rangle^2}{\sum_{i=1}^n\langle \Matrix{S}^T\eigvecA{i}, \vect{a}\rangle^2}
	\geq \frac{\langle \Matrix{S}^T\eigvecA{1}, \vect{a}\rangle^2}{\sum_{i=1}^n\langle \Matrix{S}^T\eigvecA{i}, \vect{a}\rangle^2},\, \hbox{a.e.,}
\end{equation}
where the second inequality is leveraged the fact that $\sum_{i\neq 1}\neigval{i}^{2q+1}\langle \Matrix{S}^T\eigvecA{i}, \vect{a}\rangle^2\ge 0$. (\ref{eq:thm1-3}) and the definition $R(\hu)=\max_{\vect{a}\in \sphere^{d-1} }	R_{\vect{a}}$ imply~(\ref{eq:thm1-1}) as desired and hence the proof completes.
\end{pf}

\medskip
\noindent

\begin{reptheorem}{thm:tightness-rsvd-main}\label{proof:tightness-rsvd-main}
	For any $q\in \ints$, there exists a positive semi-definite matrix $\Matrix{A}$ with $\eigvalA{1}>0$, 
	so that for $\hu=\multirsvd(\Matrix{A},\gaussian{0}{1}^{n\times d},q,d)$, it holds
	\[
	R(\hu) = \bigO\left(\left(\frac{d}{n}\right)^{\frac{1}{2q+1}}\right),
	\] 
	with probability at least $1-e^{-\bigOmega(d)}$.
\end{reptheorem}
\begin{pf}
	Let \Matrix{A} be a diagonal matrix with $\Matrix{A}_{1,1}=1$ and $\Matrix{A}_{i,i}=(d/n)^{\frac{1}{2q+1}},\,\text{ for all } i\neq 1$. Apparently, $\Matrix{A} = \vect{e}_1^T \vect{e}_1 + \sum_{i=2}^n \alpha\vect{e}_i^T \vect{e}_i$, where $\alpha = (d/n)^{\frac{1}{2q+1}}$ and $\{\vect{e}_1,\ldots, \vect{e}_n\}$ is the canonical basis in $\mathbb{R}^n$. As discussed in Section~\ref{sec:rsvd}, $R (\hu) =\max_{\vect{a}\in \sphere^{d-1} } R_{\vect{a}}$ and the alternating expression of $R_{\vect{a}}$,~\eqref{eq:rsvd-Ra2} in Section~\ref{sec:rsvd-psd}, can be rewritten as
	\begin{equation}\label{eq:thm2-1}
	\text{ for all } \vect{a}\in \sphere^{d-1},	\quad	R_{\vect{a}} 
	= \frac{  \langle \Matrix{S}^T\vect{e}_1,\vect{a}\rangle^2  }{  \langle \Matrix{S}^T\vect{e}_1, \vect{a}\rangle^2+ \sum_{i=2}^n \alpha^{2q}\langle \Matrix{S}^T\vect{e}_i,\vect{a}\rangle^2  }  
	+ \frac{  \sum_{i=2}^n \alpha^{2q+1}\langle \Matrix{S}^T\vect{e}_i,\vect{a}\rangle^2}{  \langle \Matrix{S}^T\vect{e}_1, \vect{a}\rangle^2+ \sum_{i=2}^n \alpha^{2q}\langle \Matrix{S}^T\vect{e}_i,\vect{a}\rangle^2 }.
	\end{equation}
	On the one hand, as $1>(d/n)^{\frac{2q}{2q+1}}=\alpha^{2q}$, the first term in \eqref{eq:thm2-1} is upper bounded as:
	\begin{equation}\label{eq:thm2-2}
		\frac{  \langle \Matrix{S}^T\vect{e}_1,\vect{a}\rangle^2  }{  \langle \Matrix{S}^T\vect{e}_1, a\rangle^2+ \sum_{i=2}^n \alpha^{2q}\langle \Matrix{S}^T\vect{e}_i,\vect{a}\rangle^2  } 
		\leq \frac{  \langle \Matrix{S}^T\vect{e}_1,\vect{a}\rangle^2  }{   \sum_{i=1}^n \alpha^{2q}\langle \Matrix{S}^T\vect{e}_i,\vect{a}\rangle^2  } \le \alpha^{-2q} \cos^2\theta (\vect{e}_1 , \Matrix{S}),
	\end{equation}
	where the second inequality follows directly from the definition of $\cos^2\theta (\vect{e}_1 , \Matrix{S})$. On the other hand, the second term in \eqref{eq:thm2-1} is upper bounded as:
	\begin{equation}\label{eq:thm2-3}
		 \frac{  \sum_{i=2}^n \alpha^{2q+1}\langle \Matrix{S}^T\vect{e}_i,\vect{a}\rangle^2}{  \langle \Matrix{S}^T\vect{e}_1 , \vect{a}\rangle^2+ \sum_{i=2}^n \alpha^{2q}\langle \Matrix{S}^T\vect{e}_i,\vect{a}\rangle^2 }\le  \frac{  \sum_{i=2}^n \alpha^{2q+1}\langle \Matrix{S}^T\vect{e}_i,\vect{a}\rangle^2}{  \sum_{i=2}^n \alpha^{2q}\langle \Matrix{S}^T\vect{e}_i,\vect{a}\rangle^2 }=\alpha .
	\end{equation}
	By substituting \eqref{eq:thm2-2} and\eqref{eq:thm2-3} into \eqref{eq:thm2-1}, we derive that $R_{\vect{a}}\le \alpha^{-2q} \cos^2\theta (\vect{e}_1 , \Matrix{S})+\alpha,\,\text{ for all } \vect{a}\in \sphere^{d-1}$, which provides an upper bound of $R(\hu)$. Finally, invoking Lemma~\ref{lem:cos-angle-normal}, which states that $\cos^2\theta(\vect{e}_1,\Matrix{S})=\bigTheta(d/n)$ with high probability, and recalling that $\alpha=( d/n)^{\frac{1}{2q+1}}$ yields the conclusion.
\end{pf}

\medskip
\noindent

\begin{reptheorem}{thm:rsvd-power-law}\label{proof:rsvd-power-law}
Let $\Matrix{A}$ be a positive semi-definite matrix, $\hu=\multirsvd(\Matrix{A},\gaussian{0}{1}^{n\times d},q,d)$ for any $q\in\ints$, and $i_0$ be defined as in Definition~\ref{def:power-law} in Section~\ref{sec:rsvd-psd}.
Then
\[
R(\hu) = \bigOmega\left(\left(\frac{d}{d+i_0}\right)^{\frac{1}{2q+1}}\right)
\] 
holds with probability at least $1-e^{-\bigOmega(d)}$.
\end{reptheorem}
\begin{pf}
If $i_0=n$, then we subsume the result by Theorem \ref{thm:rsvd-main} directly. Hence we assume that $ i_0<n$ below.
 	
By applying Corollary \ref{cor:length-gaussian} in Appendix \ref{app:C1} with $\delta = \frac{1}{3}, \vect{x} =\vect{u}_1$, we have probability $1-e^{-\bigOmega(d)}$ that 
\begin{equation}\label{eq:power0}
	\frac{2d}{3}\le \norm{\Matrix{S}^T\vect{u}_1}_2^2\le \frac{4d}{3} , 
\end{equation}
	which directly implies that $\Matrix{S}^T\vect{u}_1\neq 0$. In the following, we consider (i). $\sum_{i=1}^{i_0} \neigval{i}^{2q}\langle \Matrix{S}^T\eigvecA{i}, \Matrix{S}^T\eigvecA{1}\rangle^2 >  \sum_{i=i_0+1}^{n} \neigval{i}^{2q}\langle \Matrix{S}^T\eigvecA{i}, \Matrix{S}^T\eigvecA{1}\rangle^2     $; (ii). otherwise. Then, we show the claimed lower bound in (i). and (ii). separately by invoking Lemma \ref{lem:gaussian-sum-of-inner-product} in Appendix \ref{app:tech}, which gives the bounds for the weighted sum with high probability.\\
	
	\underline{(i). $\sum_{i=1}^{i_0} \neigval{i}^{2q}\langle \Matrix{S}^T\eigvecA{i}, \Matrix{S}^T\eigvecA{1}\rangle^2 >  \sum_{i=i_0+1}^{n} \neigval{i}^{2q}\langle \Matrix{S}^T\eigvecA{i}, \Matrix{S}^T\eigvecA{1}\rangle^2 $.}\\
	Roughly speaking in (i), the top $i_0$ terms dominate, hence one can expect the similar proof for Theorem \ref{thm:rsvd-main} without the last $n-i_0$ terms will help us reason. The alternating form of $R(\hu)$ follows by~\eqref{eq:rsvd-Ra2} in Section \ref{sec:4.1}, 
	\begin{equation}\label{eq:power1-1}
	R(\hu)= \max_{\vect{a}\in \sphere^{d-1} }	\frac{\sum_{i=1}^n\neigval{i}^{2q+1}\langle \Matrix{S}^T\eigvecA{i}, \vect{a}\rangle^2}{\sum_{i=1}^n\neigval{i}^{2q}\langle \Matrix{S}^T\eigvecA{i}, \vect{a}\rangle^2} \ge \frac{\sum_{i=1}^{n}\neigval{i}^{2q+1}\langle \Matrix{S}^T\eigvecA{i}, \Matrix{S}^T\vect{u}_1\rangle^2}{\sum_{i=1}^n\neigval{i}^{2q}\langle \Matrix{S}^T\eigvecA{i},  \Matrix{S}^T\vect{u}_1\rangle^2}  > \frac{\sum_{i=1}^{i_0} \neigval{i}^{2q+1}\langle \Matrix{S}^T\eigvecA{i}, \Matrix{S}^T\vect{u}_1\rangle^2}{2\sum_{i=1}^{i_0}\neigval{i}^{2q}\langle \Matrix{S}^T\eigvecA{i},  \Matrix{S}^T\vect{u}_1\rangle^2},
	\end{equation}
	where the first inequality comes from the fact that $\Matrix{S}^T\vect{u}_1/\left\| \Matrix{S}^T\vect{u}_1 \right\|_2 \in \sphere^{d-1}$ and the last one uses that $\sum_{i=i_0+1}^n \neigval{i}^{2q+1}\langle \Matrix{S}^T\eigvecA{i}, \Matrix{S}^T\vect{u}_1\rangle^2\ge  0$ and (i). 
	From \eqref{eq:power1-1}, we repeat the deduction of \eqref{eq:main-intuition3} in Section \ref{sec:4.1}, by viewing $R_{\vect{a}}$ as $2R(\hu)$, $\alpha_i=\alpha_i$ for $i=1,\ldots,i_0$, $\alpha_i=0$ for $i>i_0$, and $\langle \Matrix{S}^T\eigvecA{i},  \Matrix{S}^T\vect{a}\rangle^2=\langle \Matrix{S}^T\eigvecA{i},  \Matrix{S}^T\vect{u}_1\rangle^2$ to conclude that  (an alternative way is to use Lemma~\ref{lem:psd-rsvd-holder} as shown Appendix~\ref{sec:A1})
	\begin{equation}\label{eq:power1-2}
 \frac{\sum_{i=1}^{i_0}\neigval{i}^{2q+1}\langle \Matrix{S}^T\eigvecA{i} ,   \Matrix{S}^T\eigvecA{1}  \rangle^2}{\sum_{i=1}^{i_0}\neigval{i}^{2q}\langle \Matrix{S}^T\eigvecA{i},  \Matrix{S}^T\eigvecA{1} \rangle^2}>\left((2R(\hu))^{-1} \frac{\sum_{i=1}^{i_0}\neigval{i}^{2q+1}\langle \Matrix{S}^T\eigvecA{i},\Matrix{S}^T\eigvecA{1}   \rangle^2}{\sum_{i=1}^{i_0}\langle \Matrix{S}^T\eigvecA{i},  \Matrix{S}^T\eigvecA{1}\rangle^2}\right)^{\frac{1}{2q}}.
	\end{equation}
	Rearranging the inequalities (\ref{eq:power1-1}) and (\ref{eq:power1-2}), we get
	\begin{equation}\label{eq:power1-3}
		(2R(\hu))^{2q+1} >  \frac{\langle \Matrix{S}^T\eigvecA{1}, \Matrix{S}^T\eigvecA{1}\rangle^2}{\sum_{i=1}^{i_0}\langle \Matrix{S}^T\eigvecA{i}, \Matrix{S}^T\eigvecA{1} \rangle^2}\ge \frac{4d^2}{16d^2+9\sum_{1\le i\le i_0}  \langle \Matrix{S}^T\eigvecA{i}, \Matrix{S}^T\eigvecA{1} \rangle^2 },
	\end{equation}
	where the last inequality is a consequence of \eqref{eq:power0}. By applying Lemma \ref{lem:gaussian-sum-of-inner-product} in Appendix \ref{app:tech} with $\epsilon=\frac{1}{3},\delta=\frac{1}{3}$, $\beta_1=\ldots=\beta_{i_0}=1$ and $\beta_{i_0+1}=\ldots =\beta_{n}=0$, then we have probability $1-e^{-\bigOmega(d)}$ that $\sum_{1\le i\le i_0}  \langle \Matrix{S}^T\eigvecA{i}, \Matrix{S}^T\eigvecA{1} \rangle^2 \le \frac{16di_0}{9}$. Together with \eqref{eq:power1-3}, the proof is derived by the union bound.
	
	\medskip
	\noindent
	
	\underline{(ii). $\sum_{i=1}^{i_0} \neigval{i}^{2q}\langle \Matrix{S}^T\eigvecA{i}, \Matrix{S}^T\eigvecA{1}\rangle^2 \le  \sum_{i=i_0+1}^{n} \neigval{i}^{2q}\langle \Matrix{S}^T\eigvecA{i}, \Matrix{S}^T\eigvecA{1}\rangle^2 $.}\\
	
	As $\Matrix{S}^T\vect{u}_1/\left\| \Matrix{S}^T\vect{u}_1 \right\|_2 \in \sphere^{d-1}$, \eqref{eq:rsvd-Ra2} in Section~\ref{sec:rsvd-psd} yields that 
	\[
	R(\hu)	\geq \frac{\sum_{i=1}^n\neigval{i}^{2q+1}\langle \Matrix{S}^T\eigvecA{i}, \Matrix{S}^T\eigvecA{1}\rangle^2}{\sum_{i=1}^n\neigval{i}^{2q}\langle \Matrix{S}^T\eigvecA{i}, \Matrix{S}^T\eigvecA{1}\rangle^2}
	\geq \frac{\langle \Matrix{S}^T\eigvecA{1}, \Matrix{S}^T\eigvecA{1}\rangle^2}{2\sum_{i=i_0+1}^n\neigval{i}^{2q}\langle \Matrix{S}^T\eigvecA{i}, \Matrix{S}^T\eigvecA{1}\rangle^2}  \ge \frac{2d^2}{9\sum_{i=i_0+1}^n\neigval{i}^{2q}\langle \Matrix{S}^T\eigvecA{i}, \Matrix{S}^T\eigvecA{1}\rangle^2} ,
	\]
where the second inequality is due to (ii); the last is a result of \eqref{eq:power0} . By Lemma~\ref{lem:gaussian-sum-of-inner-product} with $\delta=d, \epsilon=\frac{1}{2}$, $\beta_2=\ldots=\beta_{i_0}=0$, and $\beta_{i}= \alpha_i^{2q}$ for all $i=i_0+1,\ldots, n$, we have
\[
\prob{\sum_{i=i_0+1}^n\neigval{i}^{2q}\langle \Matrix{S}^T\eigvecA{i}, \Matrix{S}^T\eigvecA{1}\rangle^2
	\le \frac{3d(d+1)}{2}\sum_{i=i_0+1}^n\neigval{i}^{2q}}
\leq 1- e^{-\bigOmega(d)}.
\]
By Definition~\ref{def:power-law}, since $\gamma>1/q$,
\[
\sum_{i=i_0+1}^n\neigval{i}^{2q} 
\leq C\int_{1}^\infty x^{-2q\gamma}dx < C\int_{1}^\infty x^{-2}dx=C.
\]
Hence, the union bound yields $R (\hu)=\bigOmega(1)$ with probability at least $1-e^{-\bigOmega(d)}$.
\end{pf}

\medskip
\noindent

\subsection{\multirsvd with indefinite matrices}\label{app:A3}
Assumption \ref{assumption:rsvd-indefinite} is restated here for convenience.

\medskip

\begin{repassumption}{assumption:rsvd-indefinite}		
	Assume there exists a constant $\kappa\in(0,1]$ such that
	$\sum_{i=2}^n\eigvalA{i}^{2q+1} \geq \kappa \sum_{i=2}^n\abs{\eigvalA{i}}^{2q+1}$.
\end{repassumption}

\medskip

\begin{replemma}{lem:rsvd-indefinite}\label{proof:rsvd-indefinite}
Assume that matrix $\Matrix{A}$ satisfies Assumption~\ref{assumption:rsvd-indefinite} and 
$\Matrix{S}\sim\gaussian{0}{1}^{n\times d}$. 
There exists a constant $c_{\kappa}\in(0,1]$ such that  
{\highlight
\[
\mathbb{P}\left[ \sum_{i=1}^n\eigvalA{i}^{2q+1}\langle \Matrix{S}^T\eigvecA{i}, \Matrix{S}^T\eigvecA{1}\rangle^2 \geq c_{\kappa} \sum_{i=1}^n\abs{\eigvalA{i}}^{2q+1}\langle \Matrix{S}^T\eigvecA{i}, \Matrix{S}^T\eigvecA{1}\rangle^2\right] \ge 1-e^{-\bigOmega(\sqrt{d}\kappa^2)}.
\]
}
\end{replemma}
\begin{pf}
	Recall $\neigval{i}=\eigvalA{i}/\eigvalA{1}$ for all $i\in[n]$ and introduce $\mathcal{I}_+=\{i\in[n]\backslash\{1\}:\neigval{i}>0\}$ and $\mathcal{I}_-=\{i\in[n]:\neigval{i}<0\}$.
	It is natural to assume $\mathcal{I}_- \neq \emptyset$, as otherwise this lemma trivially holds. Also, for simplicity $\kappa\in (0,1]$ is assumed to be the number such that $\sum_{i=2}^n\neigval{i}^{2q+1}=\kappa\sum_{i=2}^n\abs{\neigval{i}}^{2q+1}$ (this number can be found always).\\
	
	Apparently in both sums of interest, $\sum_{i=1}^n \alpha_{i}^{2q+1} \langle \Matrix{S}^T\eigvecA{i}, \Matrix{S}^T\eigvecA{1}\rangle^2$ and $\sum_{i=1}^n\abs{\alpha_i}^{2q+1}\langle \Matrix{S}^T\eigvecA{i}, \Matrix{S}^T\eigvecA{1}\rangle^2$, the largest single term is the first one, $\langle \Matrix{S}^T\eigvecA{1}, \Matrix{S}^T\eigvecA{1}\rangle^2 $, so our initial step is to derive its high probability bound.  Applying Corollary~\ref{cor:length-gaussian} in Appendix \ref{app:C1} with $\vect{x}=\vect{u}_1$ and $\delta=1-\sqrt{1-\kappa/4}$ (resp. $\delta=\sqrt{1+\kappa/4}-1$) for the lower-tail (resp. upper-tail) yields
	{\highlight
		\begin{equation}\label{eq:A3_0}
			\prob{d^2\left(1-\frac{\kappa}{4}\right)
				\leq \langle \Matrix{S}^T\eigvecA{1}, \Matrix{S}^T\eigvecA{1}\rangle^2 
				\leq d^2\left(1+\frac{\kappa}{4}\right)}
			\geq 1-e^{-\bigOmega(d\kappa^2)}.
		\end{equation}
	}
	
	However, as the other terms highly depend on the decay rate of eigenvalues, to derive a high probability bound, we need to carefully choose the parameters when applying concentration inequalities.
	In what follows, we define $s_+=\sum_{i\in\mathcal{I}_+}\neigval{i}^{2q+1}$ and $s_-=\sum_{i\in\mathcal{I}_-}\abs{\neigval{i}}^{2q+1}$, and then prove in two cases: either (i). $s_- = \bigOmega(\sqrt{d})$ or (ii). $s_- = o(\sqrt{d})$.\\
	
	\underline{(i). $s_- = \bigOmega(\sqrt{d})$.} Applying the lower-tail (resp. upper tail) of Lemma~\ref{lem:gaussian-sum-of-inner-product} in Appendix \ref{app:C1} with $\delta=\epsilon=1-\sqrt{1-\kappa/2}$ (resp. $\delta=\epsilon=\sqrt{1+\kappa/2}-1$), $\beta_i=\alpha_i$ for $i\in \mathcal{I}_+$, and $\beta_i=0$ otherwise, we get
	{\highlight
		\begin{equation}\label{eq:A3_11}
			\prob{
				d\left(1-\frac{\kappa}{2}\right)s_+
				\leq \sum_{i\in\mathcal{I}_+}\neigval{i}^{2q+1}\langle \Matrix{S}^T\eigvecA{i}, \Matrix{S}^T\eigvecA{1}\rangle^2
				\leq d\left(1+\frac{\kappa}{2}\right)s_+
			}
			\geq 1-e^{-\bigOmega(\sqrt{d}\kappa^2)}.
		\end{equation}
	}
	In addition, using Lemma~\ref{lem:gaussian-sum-of-inner-product} with $\delta=\epsilon=\sqrt{1+\kappa/2}-1$ in Appendix \ref{app:C1} , $\beta_i=\left|\alpha_i\right|$ for $i\in \mathcal{I}_-$, and $\beta_i=0$ otherwise, we derive
	{\highlight
		\begin{equation}\label{eq:A3_12}
			\prob{
				\sum_{i\in\mathcal{I}_-}\abs{\neigval{i}}^{2q+1}\langle \Matrix{S}^T\eigvecA{i}, \Matrix{S}^T\eigvecA{1}\rangle^2
				\leq d\left(1+\frac{\kappa}{2}\right)s_-
			}
			\geq 1-e^{-\bigOmega(\sqrt{d}\kappa^2)}.
		\end{equation}
	}	
	Now, we prove our assertion. 
	The lower-tails in \eqref{eq:A3_0}\eqref{eq:A3_11} and the upper-tail in \eqref{eq:A3_12} implies 
	\begin{align*}
		\sum_{i=1}^n\neigval{i}^{2q+1}\langle\Matrix{S}\vect{u}_i, \Matrix{S}\vect{u}_1\rangle^2  
		& \ge d\left(d(1-\frac{\kappa}{4}) + (1-\frac{\kappa}{2})s_+ - (1+\frac{\kappa}{2})s_-\right)\\
		& \overset{(a)}{=} d\left(\frac{d(4-\kappa)}{4} + \frac{\kappa(s_+ + s_-)}{2}\right)\\
		& \overset{(b)}{\geq} \frac{\kappa}{3}\left(d\left(d(1+\frac{\kappa}{4}) + (1+\frac{\kappa}{2})(s_+ + s_-)\right)\right)
		\overset{(c)}{\geq} \frac{\kappa}{3}\sum_{i=1}^n\abs{\neigval{i}}^{2q+1}\langle\Matrix{S}\vect{u}_i, \Matrix{S}\vect{u}_1\rangle^2,
	\end{align*}
	where 
	$(a)$ is due to $(1-\kappa)s_+ = (1+\kappa)s_-$ (rearranged from $\sum_{i=2}^n\neigval{i}^{2q+1}=\kappa\sum_{i=2}^n\abs{\neigval{i}}^{2q+1}$),
	$(b)$ is easily checked by comparing the coefficients,
	and $(c)$ follows from the upper-tails in \eqref{eq:A3_0}\eqref{eq:A3_11}\eqref{eq:A3_12}. Therefore, a union bound completes the proof with $c_\kappa=\frac{\kappa}{3}$ in this case.

	\smallskip
	\noindent
	
	\underline{(ii). $s_- = o(\sqrt{d})$.} This is equivalent to say that there exists a constant $c>0$ such that $s_- \leq c\sqrt{d}$.
	Notice that from $\sum_{i=1}^n \alpha_{i}^{2q+1} \langle \Matrix{S}^T\eigvecA{i}, \Matrix{S}^T\eigvecA{1}\rangle^2$ to $\sum_{i=1}^n\abs{\alpha_{i}}^{2q+1}\langle \Matrix{S}^T\eigvecA{i}, \Matrix{S}^T\eigvecA{1}\rangle^2$, only the term with index $i\in\mathcal{I}_-$ changes its sign and we show our assertion in the sense that the terms with indices in $\mathcal{I}_-$  do not affect too much with high probability. \\
	
	Invoking Lemma~\ref{lem:gaussian-sum-of-inner-product} in Appendix \ref{app:tech} with $\delta=\frac{\kappa\sqrt{d}}{8c}$, $\epsilon=\frac{\delta}{1+\delta}$, $\beta_i=\left|\alpha_i\right|$ for $i\in \mathcal{I}_-$, and $\beta_i=0$ otherwise , we get
	{\highlight
		\begin{equation}\label{eq:A3_22}
			\prob{
				\sum_{i\in\mathcal{I}_-}\abs{\neigval{i}}^{2q+1}\langle \Matrix{S}^T\eigvecA{i}, \Matrix{S}^T\eigvecA{1}\rangle^2
				\leq d\left(1+\frac{\kappa\sqrt{d}}{4c}\right)s_-
			}
			\geq 1-e^{-\bigOmega(\sqrt{d}\kappa^2)}.
		\end{equation}
	}
	On the one hand, the lower-tail in \eqref{eq:A3_0} and the upper-tail in \eqref{eq:A3_22} yield that with high probability
	\begin{equation}\label{eq:A3_23}
		\sum_{i=1}^n\neigval{i}^{2q+1}\langle\Matrix{S}\vect{u}_i, \Matrix{S}\vect{u}_1\rangle^2
		\geq d^2\left(1-\frac{\kappa}{4}\right) - d\left(1+\frac{\kappa\sqrt{d}}{4c}\right)s_- + \sum_{i\in\mathcal{I}_+}^n\neigval{i}^{2q+1}\langle\Matrix{S}\vect{u}_i, \Matrix{S}\vect{u}_1\rangle^2.
	\end{equation}
	On the other hand, the upper-tails in \eqref{eq:A3_0} and \eqref{eq:A3_22} imply that with high probability
	\begin{equation}\label{eq:A3_24}
		\sum_{i=1}^n\abs{\neigval{i}}^{2q+1}\langle\Matrix{S}\vect{u}_i, \Matrix{S}\vect{u}_1\rangle^2
		\leq  d^2(1+\frac{\kappa}{4}) + d(1+\frac{\kappa\sqrt{d}}{4c})s_-
		+ \sum_{i\in\mathcal{I}_+}^n\neigval{i}^{2q+1}\langle\Matrix{S}\vect{u}_i, \Matrix{S}\vect{u}_1\rangle^2.
	\end{equation}
	Finally with a union bound on \eqref{eq:A3_23} and \eqref{eq:A3_24} , we have probability at least {\highlight $1-e^{-\bigOmega(\sqrt{d}\kappa^2)}$} that for any $d\geq \left(\frac{14c}{10-7\kappa}\right)^2 = \bigTheta(1)$,
	\begin{align*}
		\sum_{i=1}^n\neigval{i}^{2q+1}\langle\Matrix{S}\vect{u}_i, \Matrix{S}\vect{u}_1\rangle^2 - \frac{\sum_{i=1}^n\abs{\neigval{i}}^{2q+1}\langle\Matrix{S}\vect{u}_i, \Matrix{S}\vect{u}_1\rangle^2}{6} 
		& \geq d^2\left(1-\frac{\kappa}{4}\right) - d\left(1+\frac{\kappa\sqrt{d}}{4c}\right)s_-
		- \frac{d^2(1+\frac{\kappa}{4}) + d(1+\frac{\kappa\sqrt{d}}{4c})s_-}{6}\\
		& \geq 		\frac{(10-7\kappa)d^2 - 14cd\sqrt{d}}{12} \geq 0,
	\end{align*}
	 where the second inequality stems from $s_-\leq c\sqrt{d}$. Hence, the proof is completed with $c_\kappa=\frac{1}{6}$ in this case.
\end{pf}

\medskip
\noindent

\begin{reptheorem}{thm:rsvd-indefinite-power-law}\label{proof:rsvd-indefinite-power-law}
	Assume $\Matrix{A}$ satisfies Assumption~\ref{assumption:rsvd-indefinite}.
	Let $\hu=\multirsvd(\Matrix{A},\gaussian{0}{1}^{n\times d},q,d)$ for any $q\in\ints$.
	Then, {\highlight
		\[R(\hu) = \bigOmega\left(c_{\kappa}\left(\frac{d}{d+i_0}\right)^{\frac{1}{2q+1}}\right)\]  
		with probability at least $1-e^{-\bigOmega(\sqrt{d}\kappa^2)}$.
	} 
\end{reptheorem}

\begin{pf}
	Evaluating $R_{\vect{a}}$ defined in \eqref{eq:rsvd-Ra2} in Section~\ref{sec:4.1} on $\vect{a}=\Matrix{S}^T\eigvecA{1}/\norm{\Matrix{S}^T\eigvecA{1}}_2$ and by Lemma~\ref{lem:rsvd-indefinite} there exists a constant $c_{\kappa}\in (0,1]$ such that
	\[
	R_{\vect{a}} =
	\frac{\sum_{i=1}^n\neigval{i}^{2q+1}\langle \Matrix{S}^T\eigvecA{i}, \Matrix{S}^T\eigvecA{1}\rangle^2}{\sum_{i=1}^n\neigval{i}^{2q}\langle \Matrix{S}^T\eigvecA{i}, \Matrix{S}^T\eigvecA{1}\rangle^2}
	\geq c_{\kappa}\, \frac{\sum_{i=1}^n\abs{\neigval{i}}^{2q+1}\langle \Matrix{S}^T\eigvecA{i}, \Matrix{S}^T\eigvecA{1}\rangle^2}{\sum_{i=1}^n\neigval{i}^{2q}\langle \Matrix{S}^T\eigvecA{i}, \Matrix{S}^T\eigvecA{1}\rangle^2}
	= 	c_{\kappa}\,\bar{R}_{\vect{a}},
	\]
	where $\bar{R}_{\vect{a}}$ is introduced in \eqref{eq:rsvd-Ra-bar} in Section~\ref{sec:rsvd-indefinite}, with probability at least {\highlight $1-e^{-\bigOmega(\sqrt{d}\kappa^2)}$}. Repeating the arguments in the proof of Theorem~\ref{thm:rsvd-power-law} in Appendix~\ref{proof:rsvd-power-law} with replacing $R(\hu)$ by $\bar{R}_{\vect{a}}$ yields:
	{\highlight
	\[
	\bar{R}_{\vect{a}} = \bigOmega\left(c_{\kappa}\left(\frac{d}{d+i_0}\right)^\frac{1}{2q+1}\right)
	\]
	with probability at least $1-e^{-\bigOmega(\sqrt{d}\kappa^2)}$}, and hence the desired result follows by the union bound.
\end{pf}

\section{Proofs of \rsum}\label{sec:proof-ours}
\subsection{Large deviation of projection length for Bernoulli random matrix}
This subsection is used to prove Lemma~\ref{lem:cos-angle-bernoulli-lb}, which serves as an intermediate step for Theorem~\ref{thm:rsum-main}, restated below. The proof relies on a simple but powerful concept, $\varepsilon$-net. As its usefulness, the definition and related theorems can be found in literature of random matrix. Here we shortly define it and state its important property below Lemma~\ref{lem:cos-angle-bernoulli-lb}. Interested reader are referred to the reference therein.

\medskip
\noindent 

\begin{definition}[$\varepsilon$-net,  Definition 4.2.1 in \citep{vershynin2018high}]
Let $(\sphere^{d-1}, \norm{\cdot}_2)$ be a metric space and $\varepsilon>0$.
A subset $\mathcal{N_{\varepsilon}}\subseteq \sphere^{d-1}$ is called $\varepsilon$-net if 
\[\forall x,y\in\mathcal{N}_{\varepsilon}, \norm{x-y}_2 \leq \varepsilon.\]
\end{definition}

\medskip
\noindent

\begin{lemma}[Corrollary 4.2.13 in \citep{vershynin2018high}]\label{lem:size-of-eps-net}
For any $\varepsilon\in (0,1)$, the size of $\mathcal{N}_{\varepsilon}$ is bounded by
\[ \abs{\mathcal{N}_{\varepsilon}} \leq 3^d\varepsilon^{-d}.\]
\end{lemma}

\medskip
\noindent 

We are ready to prove Lemma~\ref{lem:cos-angle-bernoulli-lb} restated below.\\

\begin{replemma}{lem:cos-angle-bernoulli-lb}\label{proof:cos-angle-bernoulli-lb}
	Let $\vect{v}\in\sphere^{n-1}$, $d\leq n/3$, and $\Matrix{S}\sim\bernoulli{p}^{n\times d}$ for a constant $p\in(0,1)$
	Then,
	\[
	\cos^2\theta(\vect{v}, \Matrix{S}) = \bigOmega\left(\frac{\max\{1,\langle \vect{v}, \vecones{n}\rangle^2\}}{n}\right)
	\]
	holds with probability at least $1-e^{-\bigOmega(d)}$.
\end{replemma}
\begin{pf}
		As it is easy to see that $\Matrix{S}$ is a nonzero matrix with probability $1-e^{-nd}$, the following deduction will be made under $\left\|\Matrix{S}\right\|_2>0$.\\
	By the second inequality in Corollary~\ref{cor:length-bernoulli} in Appendix~\ref{app:C1} with $\vect{x}=\vect{v}$ and $\delta=1/2$, we deduce that
	\begin{equation}\label{eq:ours1}
	\prob{\norm{\Matrix{S}^T\vect{v}}_2 \geq \sqrt{\frac{dp(1-p+p\langle \vect{v},\vecones{n}\rangle^2)}{2}} }
	\geq 1-e^{-\bigOmega\left(d\right)}.
	\end{equation}
		Recall that $\cos\theta(\vect{v},\Matrix{S}) = \max_{\vect{a}\in\sphere^{d-1}}\frac{\langle \vect{v}, \Matrix{S}\vect{a}\rangle}{\norm{\Matrix{S}\vect{a}}_2} $, \eqref{eq:ours1} allows us to substitute $\vect{a}= \Matrix{S}^T\vect{v}/\norm{\Matrix{S}^T\vect{v}}$ and have
		\[
		\cos\theta(\vect{v}, \Matrix{S})
		\geq \frac{\norm{\Matrix{S}^T\vect{v}}_2^2}{\norm{\Matrix{S}\Matrix{S}^T\vect{v}}_2}
		\geq \frac{\norm{\Matrix{S}^T\vect{v}}_2}{\norm{\Matrix{S}}_2} \ge \frac{\sqrt{dp(1-p+p\langle \vect{v},\vecones{n}\rangle^2)} }{\sqrt{2}  \norm{\Matrix{S}}_2},
		\]
	where the second inequality is due to submultiplicativity of $\left\|\cdot\right\|_2$, namely $\norm{\Matrix{S}\Matrix{S}^T\vect{v}}_2\leq \norm{\Matrix{S}}_2\norm{\Matrix{S}^T\vect{v}}_2$, and the last one is a consequence of \eqref{eq:ours1}. It remains to show that $\norm{\Matrix{S}}_2\le \bigO\left(\sqrt{nd}\right)$ w.h.p., then the proof is done. For this goal, we use the $\varepsilon$-net technique, introduced in the beginning of this subsection, and give a bound in two steps:
	\begin{enumerate}
		\item[\underline{(i).}] Let $\mathcal{N}_{\varepsilon}$ be an $\varepsilon$-net defined on $(\sphere^{d-1}, \norm{\cdot}_2)$ for some $\varepsilon\in(0,1)$ to be determined later.
		We claim that
		\begin{equation}\label{eq:our2}
		\norm{\Matrix{S}}_2 \leq \frac{1}{1-\varepsilon} \sup_{\vect{x}\in\mathcal{N}_{\varepsilon}}\norm{\Matrix{S}\vect{x}}_2.
		\end{equation}
		Let $\vect{w}^*\in\argmax_{\vect{x}\in\sphere^{d-1}}\norm{\Matrix{S}\vect{x}}_2$, and since there exists $\vect{x}^*\in\mathcal{N}_{\varepsilon}$ satisfying $\norm{\vect{w}^*-\vect{x}^*}_2\leq \varepsilon$, by submultiplicativity and triangle inequality, we get
		\[
		\varepsilon\norm{\Matrix{S}}_2 \geq \norm{\Matrix{S}(\vect{w}^*-\vect{x}^*)}_2 \geq \norm{\Matrix{S}}_2 - \norm{\Matrix{S}\vect{x}^*}_2 \ge \norm{\Matrix{S}}_2 -\sup_{\vect{x}\in\mathcal{N}_{\varepsilon}}\norm{\Matrix{S}\vect{x}}_2,
		\]
		and rearranging the terms yields \eqref{eq:our2}.
		
		\item[\underline{(ii).}] Show that
		\begin{equation}\label{eq:ours3}
		\prob{\sup_{\vect{x}\in\mathcal{N}_{\varepsilon}}\norm{\Matrix{S}\vect{x}}_2 \leq \left(\frac{3}{2}\, np(1-p + pd) \right)^{\frac{1}{2}} } \geq 1-  3^d\varepsilon^{-d}e^{-\bigOmega (n)} \ge 1-e^{-\Omega(n+d\ln\frac{\varepsilon}{3})}.
		\end{equation}
		
		For each $\vect{x}\in\mathcal{N}_{\varepsilon}$, the first inequality in Corollary~\ref{cor:length-bernoulli} in Appendix~\ref{app:C1} with $\vect{x}=\vect{x}$ and $\delta=\frac{1}{2}$ (here $n$ and $d$ are reversed) implies that we have probability at least $1-e^{-\bigOmega(n)}$
		\[
		\norm{\Matrix{S}\vect{x}}_2
		\le \left( \frac{3}{2}\, np(1-p + p\langle \vect{x}, \vecones{d}  \rangle^2)\right)^{\frac{1}{2}}
		\leq \left(\frac{3}{2}\, np(1-p + pd)\right)^{\frac{1}{2}},
		\]
		where the last inequality is due to $\langle \vect{x},    \vecones{d}   \rangle^2\leq d$. As the size of $\mathcal{N}_{\varepsilon}$ is upper bounded by $  3^d\varepsilon^{-d} $ (see Lemma~\ref{lem:size-of-eps-net} on the top of this subsection), the union bound over all $\vect{x}\in \mathcal{N}_{\varepsilon}$ yields \eqref{eq:ours3}.
		
	\end{enumerate}
	Finally, setting $\varepsilon=1/e$ in \eqref{eq:our2}-\eqref{eq:ours3} and assumption $n-2d>d$ lead to $\norm{\Matrix{S}\vect{x}}_2\le \bigO\left(\sqrt{nd}\right)$ holds with probability at least $1-e^{-\Omega(n-2d)} >1-e^{-\Omega(d)} $. The union bound completes our proof as desired.	
\end{pf}

\subsection{\rsum with positive semidefinite matrices}
\begin{reptheorem}{thm:rsum-main}\label{proof:rsum-main}
	Let \Matrix{A} be a positive semi-definite matrix with $\eigvalA{1}>0$ and $\hu=\rsum(\Matrix{A},q,d,p)$ for any constant $p\in(0,1)$, any $q\in\ints$, and $d\geq 2$.
	Then, 
	\[R(\hu) = \left(\bigOmega\left(\frac{\max\{d, \langle \eigvecA{1}, \vecones{n}\rangle^2\}}{n}\right)\right)^{\frac{1}{2q+1}}\]
	with probability at least $1-e^{-\bigOmega(d)}$.
\end{reptheorem}
\begin{pf}
	Define
	$\mathcal{A}_1 =\Big\{\left[\begin{matrix}\vect{a_1}\\\veczeros{\lfloor\frac{d}{2}\rfloor}\end{matrix}\right]: \vect{a_1}\in\sphere^{\lceil\frac{d}{2}\rceil -1}\Big\}$ and $\mathcal{A}_2 =\Big\{\left[\begin{matrix}\veczeros{\lceil\frac{d}{2}\rceil}\\ \vect{a_2}\end{matrix}\right]:\vect{a_2}\in\sphere^{\lfloor\frac{d}{2}\rfloor-1}\Big\}$.
	Since $R(\hu)\geq\max_{\vect{a}\in\sphere^{d-1}}R_{\vect{a}}$, where $R_{\vect{a}}$ is introduce on Section~\ref{sec:4.1} and has an expression~\eqref{eq:rsvd-Ra2}, we can conclude that
	\[
		R(\hu)  \geq \max\left\{	\max_{\vect{a}\in\mathcal{A}_1}R_{\vect{a}}, \max_{\vect{a}\in\mathcal{A}_2}R_{\vect{a}}
		\right\} \geq \max\left\{\cos^2\theta(\eigvecA{1}, \Matrix{S_1}), \cos^2\theta(\eigvecA{1}, \Matrix{S_2})\right\}^{\frac{1}{2q+1}},
	\]
	where the last inequality is an application of \eqref{eq:rsvd-key} in Section~\ref{sec:4.1}. The proof is completed by Lemma~\ref{lem:cos-angle-normal} and Lemma~\ref{lem:cos-angle-bernoulli-lb}.
\end{pf}

\medskip
\noindent

\subsection{\rsum with indefinite matrices}
Assumption \ref{assumption:rsum-indefinite} is restated here for convenience.

\medskip

\begin{repassumption}{assumption:rsum-indefinite}		
		Assume that (i) $\langle\vect{u}_1,\vecones{n}\rangle^2=\bigOmega(1)$ and
		(ii) there exists a constant $\kappa'\in(0,1]$ such that
		\[
		\sum_{i=2}^n\eigvalA{i}^{2q+1}\xi_i \geq \kappa' \sum_{i=2}^n\abs{\eigvalA{i}}^{2q+1}\xi_i,
		\]
		where $\xi_i=\expectation{\langle\Matrix{S}^T\vect{u}_i,\frac{\vecones{d}}{\sqrt{d}}\rangle^2}= p(1-p+pd\langle\vect{u}_i,\vecones{n}\rangle^2)$, $\forall i\in[n]$. 
\end{repassumption}

\medskip

\begin{replemma}{lem:rsum-indefinite}\label{proof:rsum-indefinite}
	Assume that $\Matrix{A}$ satisfies Assumption~\ref{assumption:rsum-indefinite}.
	Let $\Matrix{S}\sim\bernoulli{p}^{n\times d}$ for a constant $p\in(0,1)$.
	There exists a constant $c_{\kappa'}\in(0,1]$ such that
	\[
	\prob{\sum_{i=1}^n\eigvalA{i}^{2q+1}\langle \Matrix{S}^T\eigvecA{i}, \Matrix{S}^T\eigvecA{1}\rangle^2 \geq c_{\kappa'} \sum_{i=1}^n\abs{\eigvalA{i}}^{2q+1}\langle \Matrix{S}^T\eigvecA{i}, \Matrix{S}^T\eigvecA{1}\rangle^2} \geq 1-e^{-\bigOmega(\sqrt{d}{\kappa'}^2)}.
	\]
\end{replemma}
\begin{pf}
	Here we introduce
		\[
			\mu_i =\expectation{\norm{\Matrix{S}_{:,1}^T\vect{u}_i}_2^2}=p(1-p+p\langle\vect{u}_i,\vecones{n}\rangle^2), \, \forall i\in[n].
		\]
	Recall that $\neigval{i}=\eigvalA{i}/\eigvalA{1}$ for all $i\in[n]$. From Assumption~\ref{assumption:rsum-indefinite}, we have:
	\begin{itemize}
		\item 
			By (i), there exists a constant $\nu\in(0,1]$ such that $\langle\vect{u}_1,\vecones{n}\rangle^2\geq \nu$.
			It follows that
			\begin{equation}\label{eq:xi1-mu1}
			\xi_1 \geq p^2d\langle\vect{u}_1,\vecones{n}\rangle^2 = pd\cdot p((1-p)\langle\vect{u}_1,\vecones{n}\rangle^2 + p\langle\vect{u}_1,\vecones{n}\rangle^2) \geq p\nu  d\mu_1.
			\end{equation}
		\item By (ii), there exists $\kappa'\in (0,1]$ such that $\sum_{i=2}^n\neigval{i}^{2q+1}\xi_i=\kappa'\sum_{i=2}^n\abs{\neigval{i}}^{2q+1}\xi_i$.
	\end{itemize}
		We then partition $[n]$ into three subsets, $[n]=\{1\}\cup \mathcal{I}_+\cup \mathcal{I}_-$, where $\mathcal{I}_+=\{i\in[n]\backslash\{1\}:\neigval{i}>0\}$ and $\mathcal{I}_-=\{i\in[n]:\neigval{i}<0\}$.
It is natural to assume $\mathcal{I}_- \neq \emptyset$, as otherwise this lemma trivially holds. As similar to what we proceed in the proof of Lemma~\ref{lem:rsvd-indefinite} in Appendix~\ref{app:A3}, two important quantities follows from this partition: $s_+=\sum_{i\in\mathcal{I}_+}\neigval{i}^{2q+1}\xi_i$ and $s_-=\sum_{i\in\mathcal{I}_-}\abs{\neigval{i}}^{2q+1}\xi_i$. \\

	Firstly, for the term $\langle \Matrix{S}^T\eigvecA{1}, \Matrix{S}^T\eigvecA{1}\rangle^2$, applying Corollary~\ref{cor:length-bernoulli} in Appendix~\ref{app:C1} with $\vect{x}=\vect{u}_1$ and $\delta=1-\sqrt{1-\kappa'/4}$ (resp. $\delta=\sqrt{1+\kappa'/4}-1$) for the lower-tail (resp. upper-tail) yields that
	\begin{equation}\label{eq:B3_0}
		\prob{d^2\mu_1^2\left(1-\frac{\kappa'}{4}\right)
			\leq \langle \Matrix{S}^T\eigvecA{1}, \Matrix{S}^T\eigvecA{1}\rangle^2 
			\leq d^2\mu_1^2\left(1+\frac{\kappa'}{4}\right)}
		\geq 1-e^{-\bigOmega(d{\kappa'}^2)}.
	\end{equation}
	As for the remaining terms, we carefully apply concentration inequalities under two scenarios: either (i). $s_- = \bigOmega(\sqrt{d})$ or (ii). $s_- = o(\sqrt{d})$.\\
	
	\underline{(i). $s_- = \bigOmega(\sqrt{d})$.} Invoking the lower-tail (resp. upper-tail) of Lemma~\ref{lem:bernoulli-sum-of-inner-product} in \ref{app:tech} with $\delta=\epsilon=\sqrt{1+\kappa'/2}-1$ (resp. $\delta=\epsilon=\sqrt{1-\kappa'/2}-1$), $\beta_i=\alpha_i$ for $i\in \mathcal{I}_+$, and $\beta_i=0$ otherwise, we get
	{\highlight
	\begin{equation}\label{eq:B3_11}
		\prob{
			\xi_1\left(1-\frac{\kappa'}{2}\right)s_+
			\leq \sum_{i\in\mathcal{I}_+}\neigval{i}^{2q+1}\langle \Matrix{S}^T\eigvecA{i}, \Matrix{S}^T\eigvecA{1}\rangle^2
			\leq d\mu_1\left(1+\frac{\kappa'}{2}\right)s_+
		}
		\geq 1-e^{-\bigOmega(\sqrt{d}{\kappa'}^2)}.
	\end{equation}
	}
	Again, using Lemma~\ref{lem:bernoulli-sum-of-inner-product} with $\delta=\epsilon=\sqrt{1+\kappa'/2}-1$, $\beta_i=\left|\alpha_i\right|$ for $i\in \mathcal{I}_-$, and $\beta_i=0$ otherwise leads to
	\begin{equation}\label{eq:B3_12}
		\prob{
			\sum_{i\in\mathcal{I}_-}\abs{\neigval{i}}^{2q+1}\langle \Matrix{S}^T\eigvecA{i}, \Matrix{S}^T\eigvecA{1}\rangle^2
			\leq d\mu_1\left(1+\frac{\kappa'}{2}\right)s_-
		}
		\geq 1-e^{-\bigOmega(\sqrt{d}{\kappa'}^2)}.
	\end{equation}	
	Now, we prove our assertion. The lower-tails in \eqref{eq:B3_0}\eqref{eq:B3_11} and upper-tail in \eqref{eq:B3_12} imply that
	{\highlight
	\begin{align*}
		\sum_{i=1}^n\neigval{i}^{2q+1}\langle\Matrix{S}\vect{u}_i, \Matrix{S}\vect{u}_1\rangle^2  
		& \geq d\mu_1\left(d\mu_1(1-\frac{\kappa'}{4}) + \frac{\xi_1}{d\mu_1}(1-\frac{\kappa'}{2})s_+ - (1+\frac{\kappa'}{2})s_-\right)\\
		& \overset{(a)}{\geq} p\nu\cdot d\mu_1\left(d\mu_1(1-\frac{\kappa'}{4}) + (1-\frac{\kappa'}{2})s_+ - (1+\frac{\kappa'}{2})s_-\right)\\
		& \overset{(b)}{=} p\nu\cdot d\mu_1\left(\frac{d(4-\kappa')\mu_1}{4} + \frac{\kappa'(s_+ + s_-)}{2}\right)\\
		& \overset{(c)}{\geq} \frac{p\nu\kappa'}{3}\left(d\mu_1\left(\frac{d(4+\kappa')\mu_1}{4} + (1+\frac{\kappa'}{2})(s_+ + s_-)\right)\right)
		\overset{(d)}{\geq} \frac{p\nu\kappa'}{3}\sum_{i=1}^n\abs{\neigval{i}}^{2q+1}\langle\Matrix{S}\vect{u}_i, \Matrix{S}\vect{u}_1\rangle^2,
	\end{align*}}
	where 
	$(a)$ uses \eqref{eq:xi1-mu1}, 
	$(b)$ is due to $(1-\kappa')s_+ = (1+\kappa')s_-$ (rearranged from $\sum_{i\neq 1}\neigval{i}^{2q+1}\xi_i=\kappa'\sum_{i\neq 1}\abs{\neigval{i}}^{2q+1}\xi_i$),
	$(c)$ is easily checked by comparing the coefficients,
	and $(d)$ follows from the upper-tails in \eqref{eq:B3_0}\eqref{eq:B3_11}\eqref{eq:B3_12}. Therefore, a union bound completes the proof with 
	{\highlight $c_{\kappa'}=\frac{p\nu\kappa'}{3}$} 
	in this case.

	\smallskip
	\noindent
	
	\underline{(ii). $s_- = o(\sqrt{d})$.} There exists a constant $c>0$ such that $s_- \leq c\sqrt{d}$.
	Observe that for two summations of interest, $\sum_{i=1}^n \alpha_{i}^{2q+1} \langle \Matrix{S}^T\eigvecA{i}, \Matrix{S}^T\eigvecA{1}\rangle^2$ and $\sum_{i=1}^n\abs{\alpha_{i}}^{2q+1}\langle \Matrix{S}^T\eigvecA{i}, \Matrix{S}^T\eigvecA{1}\rangle^2$, only the terms in $\mathcal{I}_-$ change their signs. Our assertion follows in the sense that the terms with indices in $\mathcal{I}_-$  do not affect too much with high probability. \\
	
	Invoking Lemma~\ref{lem:bernoulli-sum-of-inner-product} in ~\ref{app:tech} with
	{\highlight $\delta=\frac{\kappa'(\sqrt{d}-1)\mu_1}{4c}$, $\epsilon=\frac{\kappa'\mu_1}{4c(1+\delta)}$, }
	$\beta_i=\left|\alpha_i\right|$ for $i\in \mathcal{I}_-$, and $\beta_i=0$ otherwise , we get
	\begin{equation}\label{eq:B3_22}
		\prob{
			\sum_{i\in\mathcal{I}_-}\abs{\neigval{i}}^{2q+1}\langle \Matrix{S}^T\eigvecA{i}, \Matrix{S}^T\eigvecA{1}\rangle^2
			\leq d\mu_1\left(1+\frac{\kappa'\sqrt{d}\mu_1}{4c}\right)s_-
		}
		\geq 1-e^{-\bigOmega(\sqrt{d}{\kappa'}^2)}.
	\end{equation}
	On the one hand, the lower-tail in \eqref{eq:B3_0} and the upper-tail in \eqref{eq:B3_22} yield that
	\begin{equation}\label{eq:B3_23}
		\sum_{i=1}^n\neigval{i}^{2q+1}\langle\Matrix{S}\vect{u}_i, \Matrix{S}\vect{u}_1\rangle^2
		\geq d^2\mu_1^2\left(1-\frac{\kappa'}{4}\right) - d\mu_1\left(1+\frac{\kappa'\sqrt{d}\mu_1}{4c}\right)s_- + \sum_{i\in\mathcal{I}_+}^n\neigval{i}^{2q+1}\langle\Matrix{S}\vect{u}_i, \Matrix{S}\vect{u}_1\rangle^2
	\end{equation}
	On the other hand, the upper-tails in \eqref{eq:B3_0}\eqref{eq:B3_22} imply that
	\begin{equation}\label{eq:B3_24}
		\sum_{i=1}^n\abs{\neigval{i}}^{2q+1}\langle\Matrix{S}\vect{u}_i, \Matrix{S}\vect{u}_1\rangle^2
		\leq  d^2\mu_1^2(1+\frac{\kappa'}{4}) + d\mu_1(1+\frac{\kappa'\sqrt{d}\mu_1}{4c})s_-
		+ \sum_{i\in\mathcal{I}_+}^n\neigval{i}^{2q+1}\langle\Matrix{S}\vect{u}_i, \Matrix{S}\vect{u}_1\rangle^2
	\end{equation}
	As a consequence of a union bound on \eqref{eq:B3_23}\eqref{eq:B3_24}, we have with probability at least $1-e^{-\bigOmega(\sqrt{d}{\kappa'}^2)}$, 
	\begin{align*}
		& \sum_{i=1}^n\neigval{i}^{2q+1}\langle\Matrix{S}\vect{u}_i, \Matrix{S}\vect{u}_1\rangle^2 - \frac{1}{6}\cdot \sum_{i=1}^n\abs{\neigval{i}}^{2q+1}\langle\Matrix{S}\vect{u}_i, \Matrix{S}\vect{u}_1\rangle^2 \\
		& \qquad \geq d^2\mu_1^2\left(1-\frac{\kappa'}{4}\right) - d\mu_1\left(1+\frac{\kappa'\sqrt{d}\mu_1}{4c}\right)s_-
		- \frac{d^2\mu_1^2(1+\frac{\kappa'}{4}) + d\mu_1(1+\frac{\kappa'\sqrt{d}\mu_1}{4c})s_-}{6}\\
		& \qquad\geq 
		\frac{1}{12}\left((10-7\kappa')d^2\mu_1^2 - 14cd\sqrt{d}\mu_1\right) \geq 0,
	\end{align*}
	where the second inequality is due to $s_-\leq c\sqrt{d}$, for any $d\geq \left(\frac{14c}{(10-7\kappa')\mu_1}\right)^2 = \bigTheta(1)$. Hence, the proof is completed with $c_{\kappa'}=\frac{1}{6}$ in this case.
\end{pf}

\medskip
\noindent

\begin{reptheorem}{thm:rsum-indefinite-stronger}\label{proof:rsum-indefinite-stronger}
	Assume that $\Matrix{A}$ satisfies Assumptions~\ref{assumption:rsvd-indefinite} and ~\ref{assumption:rsum-indefinite}.
	Let $\hu=\rsum(\Matrix{A},q,d,p)$ for any constant $p\in(0,1)$ and any $q\in\ints$, 
	and $i_0$ be defined as in Definition~\ref{def:power-law} in Section~\ref{sec:rsvd-psd}.
	Then,
	{\highlight
	\[R(\hu) = \bigOmega\left(\max\left\{
		c_{\kappa}\left(\frac{d}{d+i_0}\right)^{\frac{1}{2q+1}},
		c_{\kappa'}\left(\frac{\max\{d, \langle \eigvecA{1}, \vecones{n}\rangle^2\}}{n}\right)^{\frac{1}{2q+1}}
	\right\}\right)\]
	with probability at least $1-e^{-\bigOmega(\sqrt{d}\min(\kappa,\kappa')^2)}$.
	}
\end{reptheorem}

\begin{pf}
	Let 
	\[
	\vect{a}_1=\left[
	\begin{matrix}
		\frac{\Matrix{S}_1^T\eigvecA{1}}{\norm{\Matrix{S}_1^T\eigvecA{1}}_2}\\
		\veczeros{\lfloor\frac{d}{2}\rfloor}
	\end{matrix}\right]
	\quad\text{ and }\quad
	\vect{a}_2=\left[
	\begin{matrix}
		\veczeros{\lceil\frac{d}{2}\rceil}\\
		\frac{\Matrix{S}_2^T\eigvecA{1}}{\norm{\Matrix{S}_2^T\eigvecA{1}}_2}
	\end{matrix}\right].
	\]
	{\highlight
	A union bound of Lemma~\ref{lem:rsvd-indefinite} and Lemma~\ref{lem:rsum-indefinite} implies that there exist constants $c_{\kappa}$ and $c_{\kappa'}$ such that
	\[
	R_{\vect{a}_1} \geq c_{\kappa}\bar{R}_{\vect{a}_1}
	\quad\text{ and }\quad
	R_{\vect{a}_2} \geq c_{\kappa'}\bar{R}_{\vect{a}_2},
	\]
	with probability at least $1-e^{-\bigOmega(\sqrt{d}\kappa^2)}-e^{-\bigOmega(\sqrt{d}{\kappa'}^2)}$},	where $R_{\vect{a}}$ and $\bar{R}_{\vect{a}}$ are defined in \eqref{eq:rsvd-Ra2} in Section~\ref{sec:4.1} and \eqref{eq:rsvd-Ra-bar} in Section~\ref{sec:rsvd-indefinite}, respectively.
	Hence, 
	\[
	\prob{R(\hu) 
	\geq \max\left\{c_{\kappa}\bar{R}_{\vect{a}_1}, c_{\kappa'}\bar{R}_{\vect{a}_2}\right\}
	} \geq 1-e^{-\bigOmega(\sqrt{d}\min(\kappa,\kappa')^2)}.
	\]
	Finally, applying similar argument in the proof of Theorem~\ref{thm:rsvd-power-law} (see Appendix \ref{proof:rsvd-power-law}) to lower bound $\bar{R}_{\vect{a}_1}$ and Theorem~\ref{thm:rsum-main} to lower bound $\bar{R}_{\vect{a}_2}$ completes the proof.
\end{pf}

\section{Concentration inequalities}\label{sec:tools}

{\highlight
Before showing our lemmas on  both Gaussian and Bernoulli random variables, there are some necessary definition and standard concentration inequalities to be introduced. For the random variables considered in this work, sub-gaussian and sub-exponential norms are useful to quantify the probabilities of rare events.
In~\ref{app:sub}, we introduce them for completeness and list the concentration inequalities (Hoeffding, Bernstein, and Hanson-Wright inequalities) used in the following proofs. In~\ref{app:C1}, we provide two corollaries yielded by Bernstein inequality for Gaussian and Bernoulli distributions respectively. Finally, our technical lemmas for these two random variables will be shown in~\ref{app:tech}. 
}

\medskip
\noindent

\subsection{Sub-gaussian norm and sub-exponential norm}\label{app:sub}

\begin{definition}[Definition 2.5.6 \citep{vershynin2018high}]\label{def:subgaussian-norm}
	The sub-gaussian norm $\norm{\cdot}_{\psi_2}$ is a norm on the space of sub-gaussian random variables.
	For any sub-gaussian random variable $X$, 
	\[\norm{X}_{\psi_2}=\inf\{t>0:\expectation{\exp\left(X^2/t^2\right)} \leq 2\}.\]
\end{definition}

\smallskip
\noindent

The sum of sub-gaussian random variables is still a sub-gaussian random variable, and its norm can be characterized by the following Proposition.\\

\begin{proposition}[Proposition 2.6.1 \citep{vershynin2018high}]\label{prop:subgaus-rot-inv}
	Let $X_1,\cdots,X_m$ be a zero-mean sub-gaussian random variables.
	Then, 
	\[
	\norm{\sum_{i\in[m]}X_i}_{\psi_2}^2 = \bigO\left(\sum_{i\in[m]}\norm{X_i}_{\psi_2}^2\right),
	\]
	where $\bigO$ hides an absolute constant.
\end{proposition}

\smallskip
\noindent

\begin{definition}[Definition 2.7.5 \citep{vershynin2018high}]\label{def:subexp-norm}
	The sub-exponential norm $\norm{\cdot}_{\psi_1}$ is a norm on the space of sub-exponential random variables.
	For any sub-exponential random variable $X$, 
	\[\norm{X}_{\psi_1}=\inf\{t>0:\expectation{\exp\left(\abs{X}/t\right)} \leq 2\}.\]
\end{definition}

\smallskip
\noindent

If $X$ is sub-gaussian random variable, then $X$ is also a sub-exponential random variable. Besides, there is one well-known property for these two norms.\\

\begin{proposition}[Lemma 2.7.6 \citep{vershynin2018high}]\label{prop:subgaus-exp}
	Let $X$ be a zero-mean sub-gaussian random variable.
	Then, \[\norm{X}_{\psi_2}^2=\norm{X^2}_{\psi_1}.\]
\end{proposition}

\smallskip
\noindent

For concreteness, we compute sub-gaussian norms for two basic variables.\\

\begin{example}\label{exp:subgaussian-norm}
	Here we evaluate the values of $\norm{\cdot}_{\psi_2}$ and $\norm{\cdot}_{\psi_1}$ for the sub-gaussian random variables which will be used later
	\begin{itemize}
		\item If $X\sim\gaussian{0}{\sigma^2}$, for some $\sigma\in\reals_+$, then $\norm{X}_{\psi_2} = 2\sigma$.
		\item If $Y\sim\bernoulli{p}$, for some $p\in(0,1)$, then $\norm{Y}_{\psi_2} =  \frac{1}{\sqrt{\ln(1+p^{-1})}}$ and $\norm{Y}_{\psi_1}=\frac{1}{\ln(1+p^{-1})}$.
	\end{itemize}
\end{example}
\begin{pf}
	For any $t>\sqrt{2}\sigma$, we observe that
	\[
	\expectation{\exp\left(X^2/t^2\right)}=\frac{1}{\sigma\sqrt{2\pi}} \int_{x\in\reals}\exp\left(-\frac{x^2}{2\sigma^2} + \frac{x^2}{t^2}\right)dx =\frac{1}{\sigma\sqrt{\frac{1}{2\sigma^2}-\frac{1}{t^2}}},
	\]
	which is $2$ when $t=2\sigma$, hence $\norm{X}_{\psi_2} = 2\sigma$. 
	As for $Y$, elementary calculus shows that
	\[
	\norm{Y}_{\psi_2}
	= \inf\left\{t>0: p\exp(t^{-2}) + (1-p) \leq 2\right\}
	= \inf\left\{t>0: \exp(t^{-2}) \leq \frac{1+p}{p}\right\}
	= \frac{1}{\sqrt{\ln(1+p^{-1})}},
	\]	
	and that 
	\[
	\norm{Y}_{\psi_1}
	= \inf\left\{t>0: p\exp(t^{-1}) + (1-p) \leq 2\right\}
	= \inf\left\{t>0: \exp(t^{-1}) \leq \frac{1+p}{p}\right\}
	= \frac{1}{\ln(1+p^{-1})}.
	\]	
\end{pf}

\medskip
\noindent

Here is the list of concentration inequalities we will use later. The first proposition is an immediate result from Definition~\ref{def:subgaussian-norm} and \ref{def:subexp-norm}, the others are standard concentration inequalities characterized by these two norms.\\   

\begin{proposition}[Proposition~2.5.2 and Proposition~2.7.1 in \citep{vershynin2018high}]\label{prop:sub-gaussian-and-exp}
	Let $X$ and $Y$ be a sub-gaussian and a sub-exponential random variables, respectively.
	Then for any $t\geq 0$, we have
	\[
	\prob{\abs{X-\expectation{X}} \geq t} \leq \exp\left(-\bigOmega\left(\frac{t^2}{\norm{X}_{\psi_2}^2}\right)\right)
	\quad\text{and}\quad
	\prob{\abs{Y-\expectation{Y}} \geq t} \leq \exp\left(-\bigOmega\left(\frac{t}{\norm{Y}_{\psi_1}}\right)\right).
	\]
\end{proposition}

\medskip
\noindent

\begin{lemma}[Hoeffding's inequality (Theorem 2.6.3 in \citep{vershynin2018high})]\label{lem:hoeffding}
	Let $m\in \mathbb{N}$, $X_1,\cdots,X_m$ be i.i.d. zero-mean sub-gaussian random variables, and $\vect{a}\in\reals^m$ be a nonzero vector. Then,
	\[
	\forall t\geq 0,\quad
	\prob{\left|\sum_{i=1}^m\vect{a}_iX_i \right| > t} \leq \exp\left(-\bigOmega\left(\frac{t^2}{K\norm{\vect{a}}_2^2}\right)\right),
	\]
	where $K=\norm{\Matrix{X}_1}_{\psi_2}^2$.
\end{lemma}

\medskip
\noindent

\begin{lemma}[Bernstein's inequality (Theorem 2.8.2 in \citep{vershynin2018high})]\label{lem:bernstein-weighted-ineq}
	Let $m\in \mathbb{N}$ and $\vect{a}=(\vect{a}_1,\cdots,\vect{a}_m)\in\reals^m\setminus \{\veczeros{n}\}$. Let $X_1,\cdots,X_m$ be independent sub-gaussian r.v.'s.
	Then there exists a universal constant $c>0$ such that for any $t>0$, 
	\[
	\prob{\left| \sum_{i=1}^m \vect{a}_i\left(X_i^2-\expectation{X_i^2}\right) \right| \geq t} \leq 2 \exp\left(-c\min\left\{\frac{t^2}{K^2\norm{\vect{a}}_2^2}, \frac{t}{K\norm{\vect{a}}_{\infty}}\right\}\right),
	\]
	where $K=\max_{i\in[m]}\norm{X_i^2 - \expectation{X_i^2}}_{\psi_1}$.
\end{lemma}

\medskip
\noindent

\begin{lemma}[Hanson-Wright inequality (Theorem 6.2.1 in \citep{vershynin2018high})]\label{lem:hanson-wright}
	Let $m\in \mathbb{N}$ and $\vect{X}=(\vect{X}_1,\ldots,\vect{X}_m)$ be a random vector with i.i.d zero-mean sub-gaussian entries and $\Matrix{M}\in\reals^{m\times m}\setminus\{ \veczeros{m\times m}\}$. Then,
	\[
	\forall t>0,\quad
	\prob{
		\left|
			\sum_{i,j\in [m]}\Matrix{M}_{i,j}\vect{X}_i\vect{X}_j
			-\expectation{\sum_{i,j\in [m]}\Matrix{M}_{i,j}\vect{X}_i\vect{X}_j}
		\right| > t}
	\leq \exp\left(-\bigOmega\left(\min\left\{
	\frac{t^2}{K^2\norm{\Matrix{M}}_F^2}, 
	\frac{t}{K\norm{\Matrix{M}}_2}
	\right\}\right)\right),
	\]
	where $K=\norm{\vect{X}_1}_{\psi_2}^2$. 
\end{lemma}

\subsection{Useful lemmas derived from Bernstein's inequality}\label{app:C1}
In this subsection, we will use Lemma~\ref{lem:bernstein-weighted-ineq} in \ref{app:sub} to derive two Bernstein-type concentration inequalities. Corollary~\ref{cor:length-gaussian} (resp. Corollary~\ref{cor:length-bernoulli}) provides tail bounds on the length $\norm{\Matrix{S}\vect{x}}_2$ of Gaussian (resp. Bernoulli) random matrix $S$ with linear combination weights $\vect{x}$ of its columns.

\medskip
\noindent

\begin{corollary}\label{cor:length-gaussian}
	Let $\vect{x}\in\reals^n\setminus \{\veczeros{n}\}$ and $\Matrix{S}\sim\gaussian{0}{1}^{n\times d}$. Then, $\forall \delta>0$,
	\[
	\prob{\norm{\Matrix{S}^T\vect{x}}_2^2 \geq d(1+\delta)\norm{\vect{x}}_2^2} \leq e^{-\bigOmega\left(d\min\left\{\delta, \delta^2\right\}\right)},
	\quad\text{ and }\quad
	\prob{\norm{\Matrix{S}^T\vect{x}}_2^2 \leq d(1-\delta)\norm{\vect{x}}_2^2} \leq e^{-\bigOmega\left(d\min\left\{\delta, \delta^2\right\}\right)}.
	\]
\end{corollary}
\begin{pf}
	For each $i=1,\ldots, d$, the $i$-th column of $\Matrix{S}$ is denoted as $\Matrix{S}_{:,i}$. Because $\langle \Matrix{S}_{:,1} ,  \frac{\vect{x}}{\left\| \vect{x}\right\|_2}\rangle ,\ldots,\langle \Matrix{S}_{:,d}, \frac{\vect{x}}{\left\| \vect{x}\right\|_2}\rangle $ are i.i.d. random variable drawn from $\gaussian{0}{1}$, the application of Lemma \ref{lem:bernstein-weighted-ineq} with $m=d$, $\vect{a}=\vecones{d}$, $t=\delta d$, and $X_i=\langle \Matrix{S}_{:,i} ,  \vect{x}/\left\| \vect{x}\right\|_2\rangle$ for $i=1,\ldots , d$, implies that there is a universal constant $c >0$ such that
	\[
	\prob{\left| \sum_{i=1}^d \langle \Matrix{S}_{:,i} ,\frac{\vect{x}}{\left\| \vect{x}\right\|_2} \rangle^2 - d \right| \ge \delta\cdot d} \leq 2 \exp\left(-c\min\left\{\frac{\delta^2d}{K^2}, \frac{\delta d}{K}\right\}\right) 
	=\exp\left(-\bigOmega\left(d\min\left\{\delta, \delta^2\right\}\right)\right),
	\]
	where $K=\norm{\Matrix{X}_1^2-\expectation{\Matrix{X}_1^2}}_{\psi_1}$. A triangle inequality on $\psi_1$ norm gives the of $K$as:
	\[
	K
	\leq \norm{\Matrix{X}_1^2}_{\psi_1} + \norm{\expectation{\Matrix{X}_1^2}}_{\psi_1} 
	\leq \norm{\Matrix{X}_1}_{\psi_2}^2 + \frac{1}{\ln 2} \leq 2+\frac{1}{\ln 2},
	\]
	where the second inequality is a consequence of Proposition~\ref{prop:subgaus-exp} and the last one is shown in Example \ref{exp:subgaussian-norm}. As $\norm{\Matrix{S}^T\vect{x}}_2^2=\sum_{i=1}^d \langle \Matrix{S}_{:,i} ,\vect{x} \rangle^2 $, the two claimed inequalities hold by rearranging the above inequality.
\end{pf}

\medskip
\noindent

\begin{corollary}\label{cor:length-bernoulli}
	Let $\vect{x}\in\reals^n$ and $\Matrix{S}\sim\bernoulli{p}^{n\times d}$ for a constant $p\in(0,1)$.
	Then, $\forall \delta>0$,
	\[
	\prob{\norm{\Matrix{S}^T\vect{x}}_2^2 \geq d(1+\delta)\mu} \leq e^{-\bigOmega\left(d\min\left\{\delta, \delta^2\right\}\right)}
	\quad\text{ and }\quad
	\prob{\norm{\Matrix{S}^T\vect{x}}_2^2 \leq d(1-\delta)\mu} \leq e^{-\bigOmega\left(d\min\left\{\delta, \delta^2\right\}\right)},
	\]
	where $\mu=p(1-p) \norm{\vect{x}}_2^2 + p^2\langle \vect{x}, \vecones{n}\rangle^2$.
\end{corollary}
\begin{pf}
	For each $i=1,\ldots, d$, we denote the $i$-th column of $\Matrix{S}$ as $\Matrix{S}_{:,i}$. 
	Since $\langle \Matrix{S}_{:,1} ,  \vect{x}\rangle ,\ldots,\langle \Matrix{S}_{:,d}, \vect{x}\rangle $ are i.i.d., Lemma \ref{lem:bernstein-weighted-ineq} with $m=d$, $\vect{a}=\vecones{d}$, $t=\delta d\mu$, and $X_i=\langle \Matrix{S}_{: , i} ,  \vect{x}\rangle$ for $i=1,\ldots, d$, implies that there exists a universal constant $c >0$ such that
	\[
	\prob{\left| \sum_{i=1}^d \langle \Matrix{S}_{:,i} ,\vect{x} \rangle^2 - d \expectation{\langle \Matrix{S}_{:,1}, \vect{x}\rangle^2} \right| \ge \delta\cdot d\mu} 
	\leq 2 \exp\left(-c\min\left\{\frac{d\mu^2\delta^2}{K^2}, \frac{d\mu\delta}{K}\right\}\right),
	\]
	where $K=\norm{\langle \Matrix{S}_{:,1} ,\vect{x}\rangle^2 - \expectation{\langle \Matrix{S}_{:,1} ,\vect{x}\rangle^2}}_{\psi_1}$.
	The proof is done by showing (i). $\expectation{\langle \Matrix{S}_{:,1}, \vect{x}\rangle^2} =\mu$, and (ii). $K=\bigTheta(\mu)$.\\
	
	\underline{(i). Show $\expectation{\langle \Matrix{S}_{:,1}, \vect{x}\rangle^2} =\mu$:}
	By using linearity of expectation repeatedly, we obtain that
	\begin{align*}
	\expectation{\langle \Matrix{S}_{:,1}, \vect{x}\rangle^2} 
	& = \expectation{\left(\sum_{i=1}^n \Matrix{S}_{i,1}\vect{x}_i\right)^2}
	= \sum_{i=1}^n\expectation{(\Matrix{S}_{i,1}\vect{x}_i)^2} + \sum_{ i\neq j}\expectation{(\Matrix{S}_{i,1}\vect{x}_i)(\Matrix{S}_{j,1}\vect{x}_j)}\nonumber\\
	& = p\norm{\vect{x}}_2^2 + p^2(\langle \vect{x}, \vecones{n}\rangle^2 - \norm{\vect{x}}_2^2)
	= p(1-p) \norm{\vect{x}}_2^2 + p^2\langle \vect{x}, \vecones{n}\rangle^2 =\mu.
	\end{align*}
	
	\underline{(ii). Show $K=\bigTheta(\mu)$:}
	Let $\vect{Z}=\Matrix{S}_{:,1}-p\vecones{n}$. As verified in (i), $\expectation{\langle \Matrix{S}_{:,1}, \vect{x}\rangle^2} =\mu=p(1-p) \norm{\vect{x}}_2^2 + p^2\langle \vect{x}, \vecones{n}\rangle^2$, we get
	\begin{align*}
	K & =\norm{\langle \Matrix{S}_{:,1} ,\vect{x}\rangle^2 - \expectation{\langle \Matrix{S}_{:,1} ,\vect{x}\rangle^2}}_{\psi_1}
	= \norm{\langle \vect{Z},\vect{x}\rangle^2 + 2p\langle\vect{Z}, \vect{x}\rangle\langle\vect{x},\vecones{n}\rangle - p(1-p)\norm{\vect{x}}_2^2}_{\psi_1}\\
	& \leq \norm{\langle \vect{Z},\vect{x}\rangle^2}_{\psi_1}
	+ 2p\abs{\langle\vect{x},\vecones{n}\rangle}\norm{\langle\vect{Z}, \vect{x}\rangle}_{\psi_1}
	+ \frac{p(1-p)\norm{\vect{x}}_2^2}{\ln 2}.
	\end{align*}
	Since $\vect{Z}$ has i.i.d. entries and $p=\bigTheta(1)$, we evaluate 
	\begin{align*}
	\norm{\langle \vect{Z},\vect{x}\rangle^2}_{\psi_1} 
	& = \norm{\langle \vect{Z},\vect{x}\rangle}_{\psi_2}^2 
	= \sum_{i=1}^n\vect{x}_i^2\norm{\vect{Z}_i}_{\psi_2}^2 
	= \norm{x}_2^2\norm{\vect{Z}_1}_{\psi_2}^2
	\leq \norm{x}_2^2\left(\norm{\Matrix{S}_{1,1}}_{\psi_2} + \norm{p}_{\psi_2}\right)^2,\\
	\text{ and }\norm{\langle\vect{Z}, \vect{x}\rangle}_{\psi_1} 
	& = \abs{\langle\vect{x}, \vecones{n}\rangle}\norm{\vect{Z}_1}_{\psi_1}
	\leq \abs{\langle\vect{x}, \vecones{n}\rangle}\left(\norm{\Matrix{S}_{1,1}}_{\psi_1} + \norm{p}_{\psi_1}\right).
	\end{align*}
	Because $\norm{\Matrix{S}_{1,1}}_{\psi_2}=\frac{1}{\sqrt{\ln(1+p^{-1})}}=\bigTheta(1)$, $\norm{\Matrix{S}_{1,1}}_{\psi_1}=\frac{1}{\ln(1+p^{-1})}=\bigTheta(1)$ (see Example \ref{exp:subgaussian-norm} for $\psi_1$ and $\psi_2$ norm),
	$\norm{p}_{\psi_2}=\frac{p}{\sqrt{\ln 2}}=\bigTheta(1)$,
	and	$\norm{p}_{\psi_1}=\frac{p}{\ln 2}=\bigTheta(1)$, 
	combining all yields $K=\bigTheta(\mu)$.
\end{pf}

\subsection{Techinical Lemmas}\label{app:tech}

\medskip
\noindent

\begin{lemma}\label{lem:gaussian-sum-of-inner-product}
	Let $\boldsymbol{\beta}=(\beta_1,\cdots,\beta_n)\in [0,1]^n$ s.t $(\beta_2,\ldots,\beta_n)\neq\veczeros{n-1}$, $\Matrix{U}=[\vect{u}_1,\cdots, \vect{u}_n]\in\reals^{n\times n}$ be an orthonormal matrix, and $\Matrix{S}\sim\gaussian{0}{1}^{n\times d}$.
	Then, for any $\delta >0$ and $\epsilon\in (0,1)$,
	\begin{align*}
		\prob{
			\sum_{i=2}^n
			\beta_i\langle \Matrix{S}^T\vect{u}_i, \Matrix{S}^T\vect{u}_1\rangle^2
			\geq d(1+\epsilon)(1+\delta)\sum_{i=2}^n\beta_i
		} 
		& \leq \exp\left(-\bigOmega\left(
		\max\left\{1,\sum_{i=2}^n\beta_i\right\} \min \left\{\delta ,\delta^2\right\}
		\right)\right) + e^{-\bigOmega(d\epsilon^2)},\\
		\text{ and }\quad
		\prob{
			\sum_{i=2}^n
			\beta_i\langle\Matrix{S}^T\vect{u}_i, \Matrix{S}^T\vect{u}_1\rangle^2
			\leq d(1-\epsilon)(1-\delta)\sum_{i=2}^n\beta_i
		} 
		& \leq \exp\left(-\bigOmega\left(
		\max\left\{1,\sum_{i=2}^n\beta_i\right\}
		\min \left\{\delta ,\delta^2\right\}
		\right)\right) + e^{-\bigOmega(d\epsilon^2)}.
	\end{align*}
\end{lemma}
\begin{pf}\label{proof:gaussian-sum-of-inner-product}
	In the following, we only focus on the upper-tail bound as the others will hold by symmetry.\\
	
	For the simplicity of presentation, we introduce a set $\mathcal{V}_\epsilon=\{\vect{v}\in \reals^d : 0<\left\|\vect{v}\right\|^2_2\le d(1+\epsilon)\}$ and the events
	\[
 	E = \mathbb{I}\left\{  \sum_{i=2}^n
 	\beta_i\langle \Matrix{S}^T\vect{u}_i, \Matrix{S}^T\vect{u}_1\rangle^2
 	\geq d(1+\epsilon)(1+\delta)\sum_{i=2}^n\beta_i   \right\}  \hbox{ and }  G(\vect{v}) =\mathbb{I} \left\{ \Matrix{S}^T\vect{u}_1 =\vect{v}\right\}, \forall \vect{v}\in \mathcal{V}_\epsilon.
	\]
	Using Corollary \ref{cor:length-gaussian} in \ref{app:C1} with $\vect{x} = \vect{u}_1, \delta=\epsilon<1$ and the fact that $\left\| \Matrix{S}^T\vect{u}_1\right\|^2_2>0$ a.e. yield that 
	$	\mathbb{P} \left[ \neg \left( \cup_{\vect{v}\in \mathcal{V}_\epsilon}   G(\vect{v})  \right)\right]= \mathbb{P}\left[   \left\|\Matrix{S}^T\vect{u}_1\right\|^2_2 > d(1+\epsilon) \right]  \le e^{-\bigOmega ( d\epsilon^2)}$, which explicitly says that $\cup_{\vect{v}\in \mathcal{V}_\epsilon}   G(\vect{v}) $ happens with high probability. As a consequence, we have
	\begin{equation}\label{eq:gaus0}
		\prob{   E } \le \prob{E\cap \neg\left( \cup_{\vect{v}\in \mathcal{V}_\epsilon}   G(\vect{v})  \right) }+\int_{\vect{v} \in \mathcal{V}_\epsilon } \prob{  E\cap G(\vect{v}) } d \prob{ G(\vect{v})}  \le e^{-\bigOmega ( d\epsilon^2)}+  \sup_{\vect{v} \in \mathcal{V}_\epsilon }  \prob{  E\cap G(\vect{v}) },
	\end{equation}
	where the last inequality follows from $\prob{\cup_{\vect{v}\in \mathcal{V}_\epsilon}   G(\vect{v}) }\le 1$ and the upper bound of $\prob{ \neg\left( \cup_{\vect{v}\in \mathcal{V}_\epsilon}   G(\vect{v})  \right)  }$ proved above. By \eqref{eq:gaus0}, it is sufficient to show that for any $\vect{v}\in \mathcal{V}_\epsilon$,  
	\begin{equation}\label{eq:gaus1}
	\prob{  E\cap G(\vect{v}) } \le  \exp\left(-\bigOmega\left(\max\left\{1,\sum_{i=2}^n\beta_i\right\}\min \left\{\delta ,\delta^2\right\}\right)\right).
	\end{equation}
	\underline{ \bf Show \eqref{eq:gaus1} for any $\vect{v}\in \mathcal{V}_\epsilon$}\\
		
	Since $\left\|\vect{v} \right\|^2_2\le d(1+\epsilon)$ for each $\vect{v}\in \mathcal{V}_\epsilon$, 
	\begin{align}\label{eq:gauss2}
		\nonumber	\prob{  E\cap G(\vect{v}) } &=\prob{    \sum_{i=2}^n \beta_i \langle \Matrix{S}^T \vect{u}_i, \vect{v}\rangle^2 \ge \left(\sum_{i=2}^n  \beta_i \right)d(1+\delta)(1+\epsilon)  }\\
			&\le \prob{    \sum_{i=2}^n \beta_i \langle \Matrix{S}^T \vect{u}_i, \vect{v} \rangle^2 \ge \left(\sum_{i=2}^n  \beta_i \right)  (1+\delta)\left\|\vect{v}\right\|_2^2  }.
	\end{align}
	Because for each $i=2,\ldots, n$, $\langle \Matrix{S}^T \vect{u}_i, \vect{v} \rangle = \sum_{r=1}^n\sum_{s=1}^d \Matrix{S}_{r,s}(\vect{u}_i)_{r}\vect{v}_s$ is a random variable from $\mathcal{N} (0, \left\|\vect{v}\right\|_2^2)$ (a linear combination of normal distributions is a normal distribution again), 
	Example \ref{exp:subgaussian-norm} in \ref{app:sub} shows that $\left\|\langle \Matrix{S}^T \vect{u}_i, \vect{v} \rangle \right\|_{\psi_2}=2\left\|\vect{v}\right\|_2$. 
	Moreover, the assumption $\Matrix{U}$ is an orthonormal matrix implies that $\{\langle \Matrix{S}^T \vect{u}_i, \vect{v} \rangle \}_{i=2}^n$ are independent (see Theorem 8.1, Chap 5\citep{gut2009multivariate}). 
	By applying Lemma \ref{lem:bernstein-weighted-ineq} in \ref{app:sub} with $m=n-1$, $X_i=\langle \Matrix{S}^T \vect{u}_{i+1}, \vect{v} \rangle$, $\forall i=1,\cdots,n-1$, $\vect{a}=(\beta_2,\ldots, \beta_n)$, 
	and $t= \delta\cdot \left(\sum_{i=2}^n  \beta_i \right) \left\|\vect{v}\right\|_2^2 $, we give an upper bound of right-hand side of \eqref{eq:gauss2} as below (it is already shown that $\mathbb{E}[  \langle \Matrix{S}^T \vect{u}_i, \vect{v} \rangle^2 ]=\left\|\vect{v}\right\|_2^2$ before):  
	\begin{align}\label{eq:gauss3}
\nonumber		\prob{    \sum_{i=2}^n \beta_i \left(\langle \Matrix{S}^T \vect{u}_i, \vect{v} \rangle^2 - \left\|\vect{v}\right\|_2^2  \right)\ge \delta\left(\sum_{i=2}^n  \beta_i \right) \left\|\vect{v}\right\|_2^2   } &\le \prob{    \left|\sum_{i=2}^n \beta_i \left(\langle \Matrix{S}^T \vect{u}_i, \vect{v} \rangle^2 - \left\|\vect{v}\right\|_2^2\right)\right|\ge \delta\left(\sum_{i=2}^n  \beta_i \right) \left\|\vect{v}\right\|_2^2 }  \\
\nonumber & \le 2\exp\left(-c\min \left\{ \frac{   \delta^2 \left(\sum_{i=2}^n  \beta_i \right)^2 \left\|\vect{v}\right\|_2^4  }{  \left\|\vect{v}\right\|_2^4  \sum_{i=2}^n\beta_i^2   },  \frac{  \delta\cdot \left(\sum_{i=2}^n  \beta_i \right) \left\|\vect{v}\right\|_2^2 }{  \left\|\vect{v}\right\|_2^2 \max_{i\neq 1}\beta_i   }  \right\}\right) \\
		&=  2\exp \left(   -c \min \left\{   \frac{ (\sum_{i=2}^n \beta_i )^2\delta^2 }{\sum_{i=2}^n \beta_i^2}  ,        \frac{ \sum_{i=2}^n \beta_i \delta }{\max_{i\neq 1} \beta_i}    \right\} \right).
	\end{align}
 Combining \eqref{eq:gauss2}\eqref{eq:gauss3}, it remains to show that
 \[
  \hbox{(i)} \frac{(\sum_{i=2}^n \beta_i)^2}{\sum_{i=2}^n\beta_i^2} \ge \max\left\{1,\sum_{i=2}^n\beta_i\right\}\hbox{, and (ii).} \frac{\sum_{i=2}^n\beta_i }{\max_{i\neq 1}\beta_i}\ge \max\left\{1,\sum_{i=2}^n\beta_i\right\}.
  \]
	For (i). As $(\sum_{i=2}^n \beta_i)^2=\sum_{i=2}^n \beta_i^2+\sum_{i\neq j}\beta_i\beta_j\ge \sum_{i=2}^n\beta_i^2$, and $\sum_{i=2}^n\beta_i^2\le \sum_{i=2}^n\beta_i$, (i) holds by using these two inequalities in numerator and denominator respectively. \\
	For (ii).  As $\sum_{i=2}^n \beta_i\ge \max_{i\neq 1}\beta_i$, and $\max_{i\neq 1}\beta_i\le 1$, (ii) follows by using these two inequalities in numerator and denominator respectively. 
\end{pf}

\medskip
\noindent

\begin{lemma}\label{lem:bernoulli-sum-of-inner-product}
	{\highlight
	Let $\boldsymbol{\beta} =(\beta_1,\cdots,\beta_n)\in[0,1]^n$ s.t. $(\beta_2,\ldots,\beta_n)\neq\veczeros{n-1}$,
	$\Matrix{U}=[\vect{u}_1,\cdots,\vect{u}_n]\in\reals^{n\times n}$ be an orthonormal matrix with $\langle\eigvecA{1},\vecones{n}\rangle^2=\bigOmega(1)$,
	and $\Matrix{S}\sim\bernoulli{p}^{n\times d}$ with some constant $p\in(0,1)$.
	Then, for any $\delta>0$ and $\epsilon\in(0,1)$, we have probability at least $1-e^{-\bigOmega\left(	\max\left\{1,\sum_{i=2}^n\beta_i\xi_i\right\}\frac{(1-\epsilon)^2}{(1+\epsilon)^2}\min\left\{\delta,\delta^2\right\}\right)} - e^{-\bigOmega(\min(d,\xi_1)\epsilon^2)}$ that
	\[
	(1-\delta)(1-\epsilon)\sum_{i =2}^n\beta_i\xi_i\xi_1
	\leq \sum_{i=2}^n\beta_i\langle \Matrix{S}^T\vect{u}_i, \Matrix{S}^T\vect{u}_1\rangle^2
	\leq d(1+\delta)(1+\epsilon)\sum_{i =2}^n\beta_i\xi_i\mu_1,
	\]
	where 
	\[\mu_i=p(1-p+p\langle\vect{u}_i,\vecones{n}\rangle^2),\ \hbox{ and }\ \xi_i=p(1-p+pd\langle\vect{u}_i,\vecones{n}\rangle^2), \ \ \forall i\in[n].\]
	}
\end{lemma}
\begin{pf}
	Similar to the proof \ref{proof:gaussian-sum-of-inner-product} of Lemma~\ref{lem:gaussian-sum-of-inner-product}, we introduce the set 
	\[
	\mathcal{V}_\epsilon=\{\vect{v}\in \reals^d : 
	d\mu_1(1-\epsilon)\leq\left\|\vect{v}\right\|^2_2\leq d\mu_1(1+\epsilon)
	\text { and } 
	d\xi_1(1-\epsilon)\leq \langle\vect{v},\vecones{d}\rangle^2_2\leq d\xi_1(1+\epsilon)
	\}
	\]
	and the events
\begin{align*}
E = \mathbb{I}&\left(\left\{ \sum_{i=2}^n\beta_i\langle \Matrix{S}^T\vect{u}_i, \Matrix{S}^T\vect{u}_1\rangle^2 \le  (1-\delta)(1-\epsilon)\sum_{i =2}^n\beta_i\xi_i\xi_1   	\right\}\right. \cup\\
& \left. \left\{ \sum_{i=2}^n\beta_i\langle \Matrix{S}^T\vect{u}_i, \Matrix{S}^T\vect{u}_1\rangle^2 \ge d(1+\delta)(1+\epsilon)\sum_{i =2}^n\beta_i\xi_i\mu_1 \right\}\right)
\end{align*}
	and $G(\vect{v}) =\mathbb{I} \left\{ \Matrix{S}^T\vect{u}_1 =\vect{v}\right\}, \forall \vect{v}\in \mathcal{V}_\epsilon$.
	The sum rule of probability implies that 
	\begin{equation}\label{eq:bernoulli0}
		\prob{ E } 
		= \prob{E\cap \neg\left( \cup_{\vect{v}\in \mathcal{V}_\epsilon}   G(\vect{v})  \right) } + \int_{\vect{v} \in \mathcal{V}_\epsilon } \prob{  E\cap G(\vect{v}) } d \prob{ G(\vect{v})}  
		\leq \prob{\neg\left( \cup_{\vect{v}\in \mathcal{V}_\epsilon}   G(\vect{v})  \right)} +  \sup_{\vect{v} \in \mathcal{V}_\epsilon }  \prob{  E\cap G(\vect{v}) }.
	\end{equation}
	
	To bound the first term in \eqref{eq:bernoulli0}, we claim that 
	\begin{align*}
		\hbox{(i). } & \prob{\left|\norm{\Matrix{S}^T\vect{u}_1}_2^2  - d\mu_1\right|
			\geq \epsilon\cdot d\mu_1}
		\leq e^{-\bigOmega(d\epsilon^2)},\\
		\hbox{(ii). } &
		\prob{\left|\langle\Matrix{S}^T\vect{u}_1,\vecones{d}\rangle^2 - d\xi_1\right|\geq 
			\epsilon\cdot d\xi_1
		} 
		\leq e^{-\bigOmega(\xi_1\epsilon^2)},
	\end{align*}
	and then an application of a union bound of (i)(ii) yields $\prob{\neg\left( \cup_{\vect{v}\in \mathcal{V}_\epsilon}   G(\vect{v})  \right)} \leq e^{-\bigOmega(d\epsilon^2)} + e^{-\bigOmega(\xi_1\epsilon^2)}$.
	As for the second term in \eqref{eq:bernoulli0}, one consequence of Lemma~\ref{lem:bernoulli-sum-of-inner-product-iii} at the end of this section is that 
	\[
	\sup_{\vect{v} \in \mathcal{V}_\epsilon }\prob{E \cap G(\vect{v})} \leq e^{-\bigOmega\left(
		\max\left\{1,\sum_{i=2}^n\beta_i\xi_i\right\}
		\min\left\{
		1,
		\frac{\langle\vect{v},\vecones{d}\rangle^4}{d^2\norm{\vect{v}}_2^4},
		\frac{\langle\vect{v},\vecones{d}\rangle^2}{d\norm{\vect{v}}_2^2}\right\}
		\min\left\{\delta,\delta^2\right\}
		\right)}
	\]
	in which 
	$\min\left\{
	1,
	\frac{\langle\vect{v},\vecones{d}\rangle^4}{d^2\norm{\vect{v}}_2^4},
	\frac{\langle\vect{v},\vecones{d}\rangle^2}{d\norm{\vect{v}}_2^2}\right\}
	= \min\left\{
	1,
	\frac{(1-\epsilon)^2}{(1+\epsilon)^2}\frac{\xi_1^2}{d^2\mu_1^2}
	\right\}
	= \bigOmega\left(\frac{(1-\epsilon)^2}{(1+\epsilon)^2}\right)
	$.
	Hence, combining all by union bound gives the desired.
	
	It remains to show (i) and (ii).
	For convenience, let $\Matrix{Z}=\Matrix{S} - p\vecones{n}\vecones{d}^T$, a zero-mean matrix.
	
	\underline{(i).} This is a direct result of the first inequality in Corollary \ref{cor:length-bernoulli} in \ref{app:C1} with $\vect{x} = \vect{u}_1$ and $\delta=\epsilon$.
	
	\underline{(ii). To show $\prob{\left|\langle\Matrix{S}^T\vect{u}_1,\vecones{d}\rangle^2 - d\xi_1\right|\geq \epsilon\cdot d\xi_1		} 		\leq e^{-\bigOmega(\xi_1\epsilon^2)}$, where $\xi=p(1-p+pd\langle\vect{u}_1,\vecones{n}\rangle^2)$.} 
	As the lower tail is proved in a similar to the upper tail, in what follows, we will pay attention on the upper tail only. Firstly, it is easy to verified that
	\begin{equation}\label{eq:14-2}
		\langle\Matrix{S}^T\vect{u}_1,\vecones{d}\rangle^2
		=\langle\Matrix{Z}^T\vect{u}_1,\vecones{d}\rangle^2 + 2pd\langle\Matrix{Z}^T\vect{u}_1,\vecones{d}\rangle\langle\vect{u}_1,\vecones{n}\rangle + p^2d^2\langle\vect{u}_1,\vecones{n}\rangle^2.
	\end{equation}
	To bound the first (resp. the second) term in \eqref{eq:14-2}, we will use Proposition~\ref{prop:sub-gaussian-and-exp} in \ref{app:sub} for sub-exponential (resp. sub-gaussian) r.v., which is quantified the sub-gaussian norm, denoted by $K=\norm{\langle\Matrix{Z}^T\vect{u}_1,\vecones{d}\rangle}_{\psi_2}$ (recall that the sub-exponential norm can be obtained by sub-gaussian norm, and vice versa, see Propsition~\ref{prop:subgaus-exp} in \ref{app:sub}). By Proposition~\ref{prop:subgaus-rot-inv} and Example \ref{exp:subgaussian-norm} in \ref{app:sub}, we have
	\[
		K^2 
		=  \bigO\left(d\sum_{i\in[n]}(\vect{u}_1)_i^2\norm{\Matrix{Z}_{1,1}}_{\psi_2}^2\right)
		= \bigO\left(d\norm{\Matrix{Z}_{1,1}}_{\psi_2}^2\right)
		= \bigO\left(d\left(\norm{\Matrix{S}_{i,j}}_{\psi_2}+\norm{p}_{\psi_2}\right)^2\right)
		= \bigO(d).
	\]
	Additionally, one can evaluate $\langle\Matrix{Z}^T\vect{u}_1,\vecones{d}\rangle^2=dp(1-p)$ by repeatedly use the linearity of expectation and the fact that the entries of $\Matrix{S}$ are i.i.d. drawn from $\bernoulli{p}$. Hence, invoking the concentration inequality for sub-exponential (resp. sub-gaussian) in Proposition~\ref{prop:sub-gaussian-and-exp} in \ref{app:sub} with $t=\frac{d\xi_1\epsilon}{3}$ (resp. $t=\frac{\epsilon\sqrt{d\xi_1}}{3}$) on $\langle\Matrix{Z}^T\vect{u}_1,\vecones{d}\rangle^2$ (resp. $\langle\Matrix{Z}^T\vect{u}_1,\vecones{d}\rangle$) yields that
	\begin{align}
\label{eq:14-3}		\prob{\left|  \langle\Matrix{Z}^T\vect{u}_1,\vecones{d}\rangle^2-dp(1-p)\right|
			\geq \frac{d\xi_1\epsilon}{3}
		} & \leq e^{-\bigOmega(\frac{d\xi_1\epsilon}{K^2})}
		\le e^{-\bigOmega(\xi_1\epsilon^2)},\\
\label{eq:14-4}		\prob{\left|\langle\Matrix{Z}^T\vect{u}_1,\vecones{d}\rangle\right| \geq \frac{\epsilon\sqrt{d\xi_1}}{3}} 
		& \leq e^{-\bigOmega(\frac{d\xi_1\epsilon^2}{K^2})}
		= e^{-\bigOmega(\xi_1\epsilon^2)}.
	\end{align}
	Plugging these \eqref{eq:14-3} and \eqref{eq:14-4} into \eqref{eq:14-2}, a union bound gives us that
	\begin{align*}
		\langle\Matrix{Z}^T\vect{u}_1,\vecones{d}\rangle^2 + 2pd\langle\Matrix{Z}^T\vect{u}_1,\vecones{d}\rangle \langle\vect{u}_1,\vecones{n}\rangle + p^2d^2\langle\vect{u}_1,\vecones{n}\rangle^2 &\le dp(1-p)+p^2d^2\langle\vect{u}_1,\vecones{n}\rangle^2+\frac{d\xi_1\epsilon}{3}+2pd \langle\vect{u}_1,\vecones{n}\rangle \frac{\epsilon\sqrt{d\xi_1}}{3}\\
		&\le d\xi_1+d\xi_1\epsilon,
	\end{align*}
	where the second inequality is yielded by $\xi_1=p(1-p+pd\langle\vect{u}_1,\vecones{n}\rangle^2)$ and $pd\langle \vect{u}_1,\vecones{n}\rangle\le \sqrt{d\xi_1}$. Then we conclude this lemma with (ii) as desired.
\end{pf}

\medskip
\noindent

\begin{lemma}\label{lem:bernoulli-sum-of-inner-product-iii}
	{\highlight
	Let $\vect{v}\in\mathbb{R}^d\setminus \{\veczeros{d}\},\,(\beta_2,\ldots ,\beta_n)\in [0,1]^{n-1}\setminus  \{\veczeros{n-1}\}$, $[\vect{u}_1,\ldots,\vect{u}_n]\in\reals^{n\times n}$ be an orthonormal matrix, $\Matrix{S}\sim\bernoulli{p}^{n\times d}$ with some constant $p\in(0,1)$, and $\xi_i=p(1-p+pd\langle\vect{u}_i,\vecones{n}\rangle^2), \forall i\in[n]$.
	Then, 
	\[
	\prob{(1-\delta)\eta_1 
		\leq \sum_{i=2}^n \beta_i \langle \Matrix{S}^T \vect{u}_i, \vect{v}\rangle^2
		\leq (1+\delta)\eta_2	} \ge 1- e^{-\bigOmega\left(
		\max\left\{1,\sum_{i=2}^n\beta_i\xi_i\right\}
		\min\left\{
		1,
		\frac{\langle\vect{v},\vecones{n}\rangle^4}{d^2\norm{\vect{v}}_2^4},
		\frac{\langle\vect{v},\vecones{n}\rangle^2}{d\norm{\vect{v}}_2^2}\right\}
		\min\left\{\delta,\delta^2\right\}
		\right)},
	\]
	where $\eta_1=\sum_{i =2}^n\beta_i\xi_i\frac{\langle\vect{v},\vecones{d}\rangle^2}{d}$ and $\eta_2=\sum_{i =2}^n\beta_i\xi_i\norm{\vect{v}}_2^2$.
	}
\end{lemma}
\begin{pf}
An elementary calculation of evaluating the expectation of $\sum_{i=2}^n \beta_i \langle \Matrix{S}^T \vect{u}_i, \vect{v} \rangle^2$ leads to
\[
\expectation{\sum_{i=2}^n \beta_i \langle \Matrix{S}^T \vect{u}_i, \vect{v} \rangle^2}  
= p(1-p)\sum_{i =2}^n\beta_i\norm{\vect{v}}_2^2 + p^2\sum_{i=2}^n\beta_i\langle\vect{u}_i, \vecones{n}\rangle^2\langle\vect{v}, \vecones{d}\rangle^2.
\]
After applying Cauchy inequality, $\langle\vect{v}, \vecones{d}\rangle^2\leq d\norm{\vect{v}}_2^2$, twice, we get that 
$\eta_1 
\leq \expectation{\sum_{i=2}^n \beta_i \langle \Matrix{S}^T \vect{u}_i, \vect{v}\rangle^2}
\leq \eta_2$. This observation inspires us to give high probability lower bound in term of $\eta_1$ and upper bound in term of $\eta_2$ respectively.\\

  
Define $\Matrix{Z}=\Matrix{S}-p\vecones{n}\vecones{d}^T,\, \Matrix{M}=\sum_{i=2}^n \beta_i\vect{u}_i\vect{u}_i^T$ and $\Matrix{B}=\Matrix{M}\otimes \vect{v}\vect{v}^T$ where $\otimes$ is the Kronecker product. With these definition, we can express the weighted sum as:
\begin{align}
\label{eq:B=1+2}\sum_{i=2}^n \beta_i \langle \Matrix{S}^T \vect{u}_i, \vect{v}\rangle^2 & = \sum_{(i_1,j_1),(i_2,j_2)\in[n]\times [d]}\Matrix{B}_{(i_1,j_1),(i_2,j_2)}\Matrix{S}_{i_1,j_1}\Matrix{S}_{i_2,j_2}= \RC{I} + \RC{II},\\
\nonumber\text{ where } & \quad \RC{I}=\sum_{(i_1,j_1),(i_2,j_2)\in[n]\times [d]}\Matrix{B}_{(i_1,j_1),(i_2,j_2)}\Matrix{Z}_{i_1,j_1}\Matrix{Z}_{i_2,j_2} \\
\nonumber\text{ and } & \quad \RC{II}=\sum_{(i_1,j_1),(i_2,j_2)\in[n]\times [d]}\Matrix{B}_{(i_1,j_1),(i_2,j_2)}(\Matrix{Z}_{i_1,j_1}+\Matrix{Z}_{i_2,j_2}+p)p.
\end{align}
Such decomposition allows us to bound $\RC{I}$ by Lemma \ref{lem:hanson-wright} and bound $\RC{II}$ by Lemma~\ref{lem:hoeffding} in ~\ref{app:sub}, which require us to evaluate the necessary quantities.
\begin{itemize}
\item $\norm{\Matrix{B}}_F^2 = \sum_{i_1,i_2\in[n],j_1,j_2\in[d]}\left(\Matrix{M}_{i_1,i_2}\vect{v}_{j_1}\vect{v}_{j_2}\right)^2
= \norm{\vect{v}}_2^4\left\|\Matrix{M}\right\|_F
= \norm{\vect{v}}_2^4\sum_{i=2}^n\beta_i^2$, \\
where the last equation is due to $\Matrix{M}=\sum_{i=2}^n \beta_i\vect{u}_i\vect{u}_i^T$ is an eigenvalue decomposition of $\Matrix{M}$.
\item $\norm{\Matrix{B}}_2 = \norm{\Matrix{M}}_2\norm{\vect{v}\vect{v}^T}_2= \max_{i\neq 1}\beta_i \norm{\vect{v}}_2^2$,\\
 where the first equation is a property of Kronecker product (see e.g. Theorem 4.2.15 in \citep{horn1994topics}).
\end{itemize}

To bound $\RC{I}$, invoking Lemma~\ref{lem:hanson-wright} with $m=nd$, $\Matrix{M}=\Matrix{B}$, $\vect{X}_{(i-1)d+j}=\Matrix{Z}_{i,j},\,\forall i\in [n],j\in [d]$ and $t=\delta\eta_1/2$ (resp. $t=\delta\eta_2/2$) for the lower- (resp. upper-) tail bounds yields that
\begin{equation}\label{eq:14-31}
	\prob{\neg\left\{-\frac{\delta\eta_1}{2} < \RC{I} -\expectation{\RC{II}} < \frac{\delta\eta_2}{2}\right\}}
	\leq \exp\left(-\bigOmega\left(\min\left\{
	\frac{\eta_2^2\delta^2}{\norm{\Matrix{B}}_F^2},
	\frac{\eta_2\delta}{\norm{\Matrix{B}}_2}
	\right\}
	\right)\right)
	+ \exp\left(-\bigOmega\left(
	\min\left\{
	\frac{\eta_1^2\delta^2}{\norm{\Matrix{B}}_F^2},
	\frac{\eta_1\delta}{\norm{\Matrix{B}}_2}
	\right\}
	\right)\right).
\end{equation}
To bound $\RC{II}$, applying Lemma~\ref{lem:hoeffding} with $t=\delta\eta_1/4$ (resp. $t=\delta\eta_2/4$) for the lower- (resp. upper-) tail bounds yields
\begin{equation}\label{eq:14-32}
	\prob{\neg\left\{-\frac{\delta\eta_1}{2} < \RC{II} - \expectation{\RC{II}} < \frac{\delta\eta_2}{2}\right\}}
	\leq \exp\left(-\bigOmega\left(
	\frac{\eta_2^2\delta^2}{\norm{\Matrix{B}}_F^2}
	\right)\right)
	+ \exp\left(-\bigOmega\left(
	\frac{\eta_1^2\delta^2}{\norm{\Matrix{B}}_F^2}
	\right)\right).
\end{equation}	
In what follows, we will show
\begin{enumerate}
	\item[(i).]  $\min\left\{\frac{\eta_2^2}{\norm{\Matrix{B}}_F^2},\frac{\eta_2}{\norm{\Matrix{B}}_2}\right\}
	= \bigOmega\left(\max\left\{1, \sum_{i=2}^n\beta_i\xi_i\right\}\right)$, and 
	\item[(ii).] {\highlight $\min\left\{\frac{\eta_1^2}{\norm{\Matrix{B}}_F^2},\frac{\eta_1}{\norm{\Matrix{B}}_2}\right\}
	=\bigOmega\left(\max\left\{1, \sum_{i=2}^n\beta_i\xi_i\right\}\min\left\{\frac{\langle\vect{v},\vecones{d}\rangle^4}{d^2\norm{\vect{v}}_2^4},\frac{\langle\vect{v},\vecones{d}\rangle^2}{d\norm{\vect{v}}_2^2}\right\}\right)$.}
\end{enumerate}
Then this proof is done by using a union bound of \eqref{eq:14-31} and \eqref{eq:14-32} into \eqref{eq:B=1+2}.

\underline{(i).} From the definition of $\eta_2$ and our above computations, we get $\frac{\eta_2^2}{\norm{\Matrix{B}}_F^2}=\bigOmega\left(\frac{(\sum_{i=2}^n\beta_i\xi_i)^2}{\sum_{i=2}^n\beta_i^2}\right)$ and $\frac{\eta_2}{\norm{\Matrix{B}}_2}=\bigOmega\left(\frac{\sum_{i=2}^n\beta_i\xi_i}{\max_{i\neq 1}\beta_i}\right)$. It is done by the following claims:
\[
\hbox{(a).} \frac{(\sum_{i=2}^n\beta_i\xi_i)^2}{\sum_{i=2}^n\beta_i^2}=\bigOmega\left(\max\left\{1, \sum_{i=2}^n\beta_i\xi_i\right\}\right)
\hbox{, and (b).} \frac{\sum_{i=2}^n\beta_i\xi_i}{\max_{i\neq 1}\beta_i}=\bigOmega\left(\max\left\{1, \sum_{i=2}^n\beta_i\xi_i\right\}\right).
\]
\underline{(a).} stems from $(\sum_{i=2}^n \beta_i\xi_i)^2 \geq p^2(1-p)^2\sum_{i=2}^n \beta_i^2$, and $ p(1-p)\sum_{i=2}^n\beta_i^2\le p(1-p)\sum_{i=2}^n\beta_i  \le \sum_{i=2}^n\beta_i\xi_i$.\\
\underline{(b).} holds since $\sum_{i=2}^n \beta_i\xi \geq p(1-p)\max_{i\neq 1}\beta_i$, and $\max_{i\neq 1}\beta_i\leq 1$.
\\

{\highlight
\underline{(ii).} From the definition of $\eta_2$ and our above computations, we get $$
\min\left\{\frac{\eta_1^2}{\norm{\Matrix{B}}_F^2},\frac{\eta_1}{\norm{\Matrix{B}}_2}\right\}=\bigOmega\left(\min\left\{
\frac{(\sum_{i=2}^n\beta_i\xi_i)^2}{\sum_{i=2}^n\beta_i^2},
\frac{\sum_{i=2}^n\beta_i\xi_i}{\max_{i\neq 1}\beta_i}
\right\}
\min\left\{
\frac{\langle\vect{v},\vecones{d}\rangle^4}{d^2\norm{\vect{v}}_2^4},
\frac{\langle\vect{v},\vecones{d}\rangle^2}{d\norm{\vect{v}}_2^2}
\right\}\right).$$
}
We then deduce (ii). by (a). and (b). and conclude this proof.
\end{pf}
\section{Conflicting group detection: approximation ratio}\label{sec:conflicting-group-approx}
\begin{algorithm}[h!]
	\For{$i=1\to n$}{
		$\vect{r}_i = \sign{\vect{v}_i}\cdot \bernoulli{\abs{\vect{v}_i}}$\;
	}
	return $\vect{r}$\;
	\caption{\randround($\vect{v}$) by \cite{bonchi2019discovering}}
	\label{alg:rand-round}
\end{algorithm}

\begin{theorem}\label{thm:cg-approx}
	For any $\hu\in\sphere^{n-1}$, $\randround(\hu)$ is an $\bigO(n^{1/2}/R(\hu))$-approx algorithm to $2$-conflicting group detection.
\end{theorem}
\begin{pf}
	The proof strategy is similar to the analysis in \citep{bonchi2019discovering}.
	
	Let $\vect{r} = \randround(\hu)$ and $\vect{s}=\sign{\hu}$ where $\vect{s}_i=1$ if $\hu_i>0$ otherwise $\vect{s}_i=0$, $\forall i\in[n]$.
	We have
	\begin{align*}
		\expectation{\frac{\vect{r}^T\Matrix{A}\vect{r}}{\vect{r}^T\vect{r}}} & = \sum_k \expectation{\frac{\vect{r}^T\Matrix{A}\vect{r}}{\vect{r}^T\vect{r}} \Big| \vect{r}^T\vect{r}=k}\prob{\vect{r}^T\vect{r}=k} = \sum_k \frac{1}{k} \sum_{i,j\in[n]} \Matrix{A}_{i,j}\vect{s}_i\vect{s}_j \prob{\vect{r}_i\vect{r}_j=\vect{s}_i\vect{s}_j \Big| \vect{r}^T\vect{r}=k}\prob{\vect{r}^T\vect{r}=k}\\
		& \overset{(a)}{=} \sum_k \frac{1}{k} \sum_{i,j\in[n]} \Matrix{A}_{i,j}\vect{s}_i\vect{s}_j \prob{\vect{r}^T\vect{r}=k \Big| \vect{r}_i\vect{r}_j=\vect{s}_i\vect{s}_j }\prob{\vect{r}_i\vect{r}_j=\vect{s}_i\vect{s}_j} = \sum_{i,j\in[n]} \Matrix{A}_{i,j}\hu_i\hu_j \expectation{\frac{1}{r^Tr} \Big| \vect{r}_i\vect{r}_j=\vect{s}_i\vect{s}_j }\\
		&\overset{(b)}{\geq} \sum_{i,j\in[n]} \Matrix{A}_{i,j}\hu_i\hu_j \frac{1}{\expectation{r^Tr \Big| \vect{r}_i\vect{r}_j=\vect{s}_i\vect{s}_j }},
	\end{align*}
	where $(a)$ results from applying Bayes' rule, and $(b)$ uses conditional Jensen's inequality.
	By
	\[
	\expectation{\vect{r}^T\vect{r} \Big| \vect{r}_i\vect{r}_j=\vect{s}_i\vect{s}_j } = 2 + \sum_{\ell\in[n]\backslash\{i,j\}}\prob{\vect{r}_{\ell}=\vect{s}_{\ell}} \leq 2 + \sqrt{n-2},
	\]
	$(b)$ and $\hu^T\Matrix{A}\hu=R(\hu)$, we get that
	$\displaystyle \expectation{\frac{\vect{r}^T\Matrix{A}\vect{r}}{\vect{r}^T\vect{r}}}  
	\geq \frac{\hu^T\Matrix{A}\hu}{2+\sqrt{n-2}} 
	= \frac{R(\hu)\eigvalA{1}}{2+\sqrt{n-2}}$.
\end{pf}

\end{document}